\documentclass[twoside,11pt]{article}

\usepackage{jmlr2e}
\usepackage[utf8]{inputenc}
\usepackage{amsmath}
\usepackage{bbm}
\usepackage{graphicx}
\usepackage{amsfonts}
\usepackage{float}
\usepackage{subcaption}
\usepackage{booktabs}
\usepackage{multirow}
\usepackage{thmtools, thm-restate}
\usepackage[linesnumbered,ruled,vlined]{algorithm2e}
\SetKwInput{KwInput}{Input} 
\SetKwInput{KwOutput}{Output}

\DeclareMathOperator*{\argmax}{arg\,max}
\DeclareMathOperator*{\argmin}{arg\,min}
\DeclareMathOperator{\Sh}{Sh}
\DeclareMathOperator{\wt}{wt}
\newtheorem{prop}{Proposition}

\usepackage{lastpage}
\jmlrheading{22}{2021}{1-\pageref{LastPage}}{4/21; Revised
10/21}{11/21}{21-0439}{Rory Mitchell, Joshua Cooper, Eibe Frank and Geoffrey Holmes}
\ShortHeadings{title}{Mitchell, Cooper, Frank and Holmes}
\ShortHeadings{Sampling Permutations for Shapley Value Estimation}{Rory Mitchell, Joshua Cooper, Eibe Frank and Geoffrey Holmes}
\firstpageno{1}

\begin{document}
\title{Sampling Permutations for Shapley Value Estimation}

\author{\name Rory Mitchell \email ramitchellnz@gmail.com \\
       \addr Nvidia Corporation\\
       Santa Clara\\
       CA 95051, USA
       \AND
       \name Joshua Cooper \email cooper@math.sc.edu \\
       \addr Department of Mathematics\\
       University of South Carolina\\
       1523 Greene St.\\
       Columbia, SC 29223, USA
       \AND
       \name Eibe Frank \email eibe@cs.waikato.ac.nz  \\
       \addr Department of Computer Science\\
       University of Waikato\\
       Hamilton, New Zealand
       \AND
       \name Geoffrey Holmes \email geoff@cs.waikato.ac.nz  \\
        \addr Department of Computer Science\\
       University of Waikato\\
       Hamilton, New Zealand
       }

\editor{Jean-Philippe Vert}

\maketitle
\begin{abstract}%   <- trailing '%' for backward compatibility of .sty file
Game-theoretic attribution techniques based on Shapley values are used to interpret black-box machine learning models, but their exact calculation is generally NP-hard, requiring approximation methods for non-trivial models. As the computation of Shapley values can be expressed as a summation over a set of permutations, a common approach is to sample a subset of these permutations for approximation. Unfortunately, standard Monte Carlo sampling methods can exhibit slow convergence, and more sophisticated quasi-Monte Carlo methods have not yet been applied to the space of permutations. To address this, we investigate new approaches based on two classes of approximation methods and compare them empirically. First, we demonstrate quadrature techniques in a RKHS containing functions of permutations, using the Mallows kernel in combination with kernel herding and sequential Bayesian quadrature. The RKHS perspective also leads to quasi-Monte Carlo type error bounds, with a tractable discrepancy measure defined on permutations. Second, we exploit connections between the hypersphere $\mathbb{S}^{d-2}$ and permutations to create practical algorithms for generating permutation samples with good properties. Experiments show the above techniques provide significant improvements for Shapley value estimates over existing methods, converging to a smaller RMSE in the same number of model evaluations.
\end{abstract}

\begin{keywords}
  Interpretability, quasi-Monte Carlo, Shapley values
\end{keywords}

\section{Introduction} 
The seminal work of \cite{shapley1953value} introduces an axiomatic attribution of collaborative game outcomes among coalitions of participating players. Aside from their original applications in economics, Shapley values are popular in machine learning  \citep{cohen2007feature,efficient_explanation,explaining_prediction,unified} because the assignment of feature relevance to model outputs is structured according to axioms consistent with human notions of attribution. In the machine learning context, each feature is treated as a player participating in the prediction provided by a machine learning model and the prediction is considered the outcome of the game. Feature attributions via Shapley values provide valuable insight into the output of complex models that are otherwise difficult to interpret.

Exact computation of Shapley values is known to be NP-hard in general \citep{deng1994complexity} and approximations based on sampling have been proposed by several authors: \cite{mann1960values,owen_multilinear,CASTRO20091726,maleki2015addressing,CASTRO2017180}. In particular, a simple Monte Carlo estimate for the Shapley value is obtained by sampling from a uniform distribution of permutations. The extensively developed quasi-Monte Carlo theory for integration on the unit cube shows that careful selection of samples can improve convergence significantly over random sampling, but these results do not extend to the space of permutations. Here, our goal is to better characterise `good' sample sets for this unique approximation problem, and to develop tractable methods of obtaining these samples, reducing computation time for high-quality approximations of Shapley values. Crucially, we observe that sample evaluations, in this context corresponding to evaluations of machine learning models, dominate the execution time of approximations. Due to the high cost of each sample evaluation, considerable computational effort can be justified in finding such sample sets.   

In Section \ref{sec:kernel_methods}, we define a reproducing kernel Hilbert space (RKHS) with several possible kernels over permutations by exploiting the direct connection between Shapley values and permutations. Using these kernels, we apply kernel herding, and sequential Bayesian quadrature algorithms to estimate Shapley values. In particular, we observe that kernel herding, in conjunction with the universal Mallows kernel, leads to an explicit convergence rate of $O(\frac{1}{n})$ as compared to $O(\frac{1}{\sqrt{n}})$ for ordinary Monte Carlo. An outcome of our investigation into kernels is a quasi-Monte Carlo type error bound, with a tractable discrepancy formula. 

In Section \ref{sec:sphere}, we describe another family of methods for efficiently sampling Shapley values, utilising a convenient isomorphism between the symmetric group $\mathfrak{S}_d$ and points on the hypersphere $\mathbb{S}^{d-2}$. These methods are motivated by the relative ease of selecting well-spaced points on the sphere, as compared to the discrete space of permutations. We develop two new sampling methods, termed orthogonal spherical codes and Sobol permutations, that select high-quality samples by choosing points well-distributed on $\mathbb{S}^{d-2}$.

Our empirical evaluation in Section~\ref{sec:evaluation} examines the performance of the above methods compared to existing methods on a range of practical machine learning models, tracking the reduction in mean squared error against exactly calculated Shapley values for boosted decision trees and considering empirical estimates of variance in the case of convolutional neural networks. Additionally, we evaluate explicit measures of discrepancy (in the quasi-Monte Carlo sense) for the sample sets generated by our algorithms. This evaluation of discrepancy for the generated samples of permutations may be of broader interest, as quasi-Monte Carlo error bounds based on discrepancy apply to any statistics of functions of permutations and not just Shapley values.

In summary, the contributions of this work are:
\begin{itemize}
  \item The characterisation of the Shapley value approximation problem in terms of reproducing kernel Hilbert spaces.
  \item Connecting the Shapley value approximation problem to existing quasi-Monte Carlo approaches, using kernels and connections between the hypersphere and symmetric group.
  \item Experimental evaluation of these methods in terms of discrepancy, and the error of Shapley value approximations on tabular and image datasets.
\end{itemize}
\section{Background and Related Work}
\label{sec:background}
We first introduce some common notation for permutations and provide the formal definition of Shapley values. Then, we briefly review the literature for existing techniques for approximating Shapley values.

\subsection{Notation}
We refer to the symmetric group of permutations of $d$ elements as $\mathfrak{S}_d$. We reserve the use of $n$ to refer to the number of samples. The permutation $\sigma \in \mathfrak{S}_d$ assigns rank $j$ to element $i$ by $\sigma(i)=j$. For example, given the permutation written in one-line notation
$$\sigma = \begin{pmatrix}
  1 & 4 & 2 & 3
\end{pmatrix},$$
and the list of items 
$$(x_1,x_2,x_3,x_4),$$
the items are reordered such that $x_i$ occupies the $\sigma(i)$ coordinate
$$(x_1, x_3, x_4, x_2),$$
and the inverse $\sigma^{-1}(j)=i$ is
$$\sigma^{-1} = \begin{pmatrix}
  1 & 3 & 4 & 2
\end{pmatrix}.$$

An \textit{inversion} is a pair of elements in the permutation $(\sigma_i,\sigma_j)$ such that $i < j$ and $\sigma(i) > \sigma(j)$. The identity permutation, 
$$I=\begin{pmatrix}
  1 & 2 & 3 & \cdots
\end{pmatrix},$$
contains 0 inversions, and its reverse 
$$\text{Rev}(I)=\begin{pmatrix}
  \cdots & 3 & 2 & 1
\end{pmatrix},$$
contains the maximum number of inversions, $\binom{d}{2}$.
\subsection{Shapley Values}
\label{sec:shapley_values}
Shapley values \citep{shapley1953value} provide a mechanism to distribute the proceeds of a cooperative game among the members of the winning coalition by measuring marginal contribution to the final outcome. The Shapley value $\Sh_i$ for coalition member $i$ is defined as
\begin{equation}
\label{eq:shapley}
\Sh_i(v)=\sum_{S \subseteq N \setminus
\{i\}}\frac{|S|!\; (|N|-|S|-1)!}{|N|!}(v(S\cup\{i\})-v(S)),    
\end{equation}
where $S$ is a partial coalition, $N$ is the grand coalition (consisting of all members), and $v$ is the so-called ``characteristic function" that is assumed to return the proceeds (i.e., value) obtained by any coalition. 

The Shapley value function may also be conveniently expressed in terms of permutations
\begin{equation}
\label{eq:shapley_permutations}
\Sh_i(v)=\frac{1}{|N|!} \sum_{\sigma \in \mathfrak{S}_d} \big[ v([\sigma]_{i-1} \cup\{i\}) - v([\sigma]_{i-1}) \big] ,
\end{equation}
where $[\sigma]_{i-1}$ represents the set of players ranked lower than $i$ in the ordering $\sigma$. To see the equivalence between \eqref{eq:shapley} and \eqref{eq:shapley_permutations}, consider that $|S|!$ is the number of unique orderings the members of $S$ can join the coalition before $i$, and $(|N|-|S|-1)!$ is the number of unique orderings the remaining members $N \setminus S \cup \{i\}$ can join the coalition after $i$. The Shapley value is unique and has the following desirable properties:

\begin{enumerate}
  \item \textit{Efficiency}: $\sum_{i=1}^{n}\Sh_i(v)=v(N).$ The sum of Shapley values for each coalition member is the value of the grand coalition $N$. 
  \item \textit{Symmetry}: If, $\forall S\subseteq N \setminus\{i,j\}, v(S\cup\{i\})=v(S\cup\{j\})$, then $\Sh_i=\Sh_j.$ If two players have the same marginal effect on each coalition, their Shapley values are the same.
  \item \textit{Linearity}: $\Sh_i(v + w) = \Sh_i(v) + \Sh_i(w)$. The Shapley values of sums of games are the sum of the Shapley values of the respective games.
  \item \textit{Dummy}: If, $\forall S\subseteq N \setminus\{i\}, v(S\cup\{i\})=v(S)$, then $\Sh_i=0$. The coalition member whose marginal impact is always zero has a Shapley value of zero.
\end{enumerate}

 Evaluation of the Shapley value is known to be NP-hard in general \citep{deng1994complexity} but may be approximated by sampling terms from the sum of either Equation \ref{eq:shapley} or Equation \ref{eq:shapley_permutations}. This paper focuses on techniques for approximating Equation \ref{eq:shapley_permutations} via carefully chosen samples of permutations. We discuss characteristic functions $v$ that arise in the context of machine learning models, with the goal of attributing predictions to input features. 
 
 Shapley values have been used as a feature attribution method for machine learning in many prior works \citep{cohen2007feature,efficient_explanation,explaining_prediction,unified}. In the terminology of supervised learning, we have some learned model $f(x)=y$ that maps a vector of features $x$ to a prediction $y$. In this context, the Shapley values will be used to evaluate the weighted marginal contribution of features to the output of the predictive model. The value of the characteristic function is assumed to be given by $y$, and the grand coalition is given by the full set of features. In a partial coalition, only some of the features are considered ``active" and their values made available to the model to obtain a prediction. Applying the characteristic function for partial coalitions requires the definition of $f(x_S)$, where the input features $x$ are perturbed in some way according to the active subset $S$. A taxonomy of possible approaches is given in \cite{covert2020explaining}.
 
\subsection{Monte Carlo}
\label{sec:monte_carlo}
An obvious Shapley value approximation is the simple Monte Carlo estimator,
\begin{equation}
\label{eq:mc}
\bar{\Sh}_i(v)= \frac{1}{n} \sum_{\sigma \in \Pi} \big[ v([\sigma]_{i-1} \cup\{i\}) - v([\sigma]_{i-1}) \big],
\end{equation}
for a uniform sample of permutations $\Pi \subset \mathfrak{S}_d$ of size $n$. Monte Carlo techniques were used to solve electoral college voting games in \cite{mann1960values}, and a more general analysis is given in \cite{CASTRO20091726}. Equation \ref{eq:mc} is an unbiased estimator that converges asymptotically at a rate of $O(1/\sqrt{n})$ according to the Central Limit Theorem.

From a practical implementation perspective, note that a single sample of permutations $\Pi$ can be used to evaluate $\Sh_i$ for all features $i$. For each permutation $\sigma \in \Pi$ of length $d$, first evaluate the empty set $v(\{\})$, then walk through the permutation, incrementing $i$ and evaluating $v([\sigma]_i)$, yielding $d+1$ evaluations of $v$ that are used to construct marginal contributions for each feature. $v([\sigma]_{i-1})$ is not evaluated, but reused from the previous function evaluation, providing a factor of two improvement over the naive approach.

\subsection{Antithetic Sampling}
\label{sec:antithetic}
Antithetic sampling is a variance reduction technique for Monte Carlo integration where samples are taken as correlated pairs instead of standard i.i.d. samples. The  antithetic Monte Carlo estimate (see \cite{rubinstein2016simulation}) is
$$\hat{\mu}_{anti}=\frac{1}{n} \sum_{i=1}^{n/2} f(X_{i})+f(Y_{i}),$$
with variance given by
\begin{equation}
\label{eq:antithetic_variance}
\text{Var}(\hat{\mu}_{anti})= \frac{\sigma}{n}(1 + \text{Corr}(f(X),f(Y)),
\end{equation}
such that if $f(X)$ and $f(Y)$ are negatively correlated, the variance is reduced. A common choice for sampling on the unit cube is $X \sim U(0,1)^d$ with $Y_i = 1-X_i$. Antithetic sampling for functions of permutations is discussed in \cite{lomeli2019antithetic}, with a simple strategy being to take permutations and their reverse. We implement this sampling strategy in our experiments with antithetic sampling.

\subsection{Multilinear Extension}
\label{sec:owen}
Another Shapley value approximation method is the multilinear extension of \cite{owen_multilinear}. The sum over feature subsets from \eqref{eq:shapley} can be represented equivalently as an integral by introducing a random variable for feature subsets. The Shapley value is calculated as
\begin{equation}
\label{eq:multilinear}
\Sh_i(v)= \int_0^1 e_i(q) dq,
\end{equation}
where 
$$e_i(q)= \mathbb{E}[v(E_q \cup {i}) - v(E_q)],$$
and $E_q$ is a random subset of features, excluding $i$, where each feature has probability $q$ of being selected. $e_i(q)$ is estimated with samples. In our experiments, we implement a version of the multilinear extension algorithm using the trapezoid rule to sample $q$ at fixed intervals. A form of this algorithm incorporating antithetic sampling is also presented in \cite{okhrati2020multilinear}, by rewriting Equation \ref{eq:multilinear} as
$$
\Sh_i(v)= \int_0^\frac{1}{2} e_i(q) + e_i(1-q) dq 
$$
where the sample set $E_i$ is used to estimate $e_i(q)$ and the `inverse set', $\{N \setminus \{E_i, i\}\}$, is used to estimate $e_i(1-q)$. In Section \ref{sec:evaluation}, we include experiments for the multilinear extension method both with and without antithetic sampling.

\subsection{Stratified Sampling}
\label{sec:stratified}
Another common variance reduction technique is stratified sampling, where the domain of interest is divided into mutually exclusive subregions, an estimate is obtained for each subregion independently, and the estimates are combined to obtain the final estimate. For integral $\mu= \int_\mathcal{D} f(x)p(x)dx$ in domain $\mathcal{D}$, separable into $J$ non-overlapping regions $\mathcal{D}_1,\mathcal{D}_2,\cdots,\mathcal{D}_J$ where $w_j=P(X \in \mathcal{D}_j)$ and $p_j(x) = w^{-1}_j p(x)\mathbbm{1}_{x \in \mathcal{D}_j}$, the basic stratified sampling estimator is
$$\hat{\mu}_{strat} =\sum_{j=1}^{J} \frac{w_j}{n_j} \sum_{i=1}^{n_j} f(X_{ij}),$$
where $X_{ij} \sim p_j$ for $i=1,\cdots,n_j$ and $j=1,\cdots,J$ (see \cite{owen2003quasi}). The stratum size $n_j$ can be chosen with the Neyman allocation \citep{neyman} if estimates of the variance in each region are known. The stratified sampling method was first applied to Shapley value estimation by \cite{maleki2015addressing}, then improved by \cite{CASTRO2017180}. We implement the version in \cite{CASTRO2017180}, where strata $\mathcal{D}^{\ell}_{i}$ are considered for all $i=1,\cdots, d$ and $\ell=1,\cdots, d$, where $\mathcal{D}^{\ell}_{i}$ is the subset of marginal contributions with feature $i$ at position $\ell$.

This concludes discussion of existing work; the next sections introduce the primary contributions of this paper.

\section{Kernel Methods}
\label{sec:kernel_methods}
A majority of Monte Carlo integration work deals with continuous functions on $\mathbb{R}^d$, where the distribution of samples is well defined. In the space of permutations, distances between samples are not implicitly defined, so we impose a similarity metric via a kernel and select samples with good distributions relative to these kernels.

Given a positive definite kernel $K: \mathcal{X} \times \mathcal{X} \rightarrow \mathbb{R}$ over some input space $\mathcal{X}$, there is an embedding $\phi : \mathcal{X} \rightarrow \mathcal{F}$ of elements of $\mathcal{X}$ into a Hilbert space $\mathcal{F}$, where the kernel computes an inner product $K(x, y) = \langle\phi(x),\phi(y)\rangle_{\mathcal{K}}$ given $x,y \in \mathcal{X}$. Hilbert spaces associated with a kernel are known as reproducing kernel Hilbert spaces (RKHS). Kernels are used extensively in machine learning for learning relations between arbitrary structured data. In this paper, we use kernels over permutations to develop a notion of the quality of finite point sets for the Shapley value estimation problem, and for the optimisation of such point sets. For this task, we investigate three established kernels over permutations: the Kendall, Mallows, and Spearman kernels.

The Kendall and Mallows kernels are defined in \cite{permutation_kernels}. Given two permutations $\sigma$ and $\sigma'$ of the same length, both kernels are based on the number of concordant and discordant pairs between the permutations:
$$n_{\textrm{con}}(\sigma, \sigma') = \sum_{i<j} [\mathbbm{1}_{\sigma(i) < \sigma(j)}\mathbbm{1}_{\sigma'(i) < \sigma'(j)} + \mathbbm{1}_{\sigma(i) > \sigma(j)}\mathbbm{1}_{\sigma'(i) > \sigma'(j)}],$$
$$n_{\textrm{dis}}(\sigma, \sigma') = \sum_{i<j} [\mathbbm{1}_{\sigma(i) < \sigma(j)}\mathbbm{1}_{\sigma'(i) > \sigma'(j)} + \mathbbm{1}_{\sigma(i) > \sigma(j)}\mathbbm{1}_{\sigma'(i) < \sigma'(j)}].$$

Assuming the length of the permutation is $d$, the Kendall kernel, corresponding to the well-known Kendall tau correlation coefficient \citep{kendall1938new}, is
$$K_{\tau}(\sigma, \sigma')=\frac{n_{\textrm{con}}(\sigma, \sigma') - n_{\textrm{dis}}(\sigma, \sigma')}{\binom{d}{2}}.$$

The Mallows kernel, for $\lambda \geq 0$, is defined as
$$K^{\lambda}_M(\sigma, \sigma') = e^{-\lambda n_{\textrm{dis}}(\sigma, \sigma')/\binom{d}{2}}.$$
Here, the Mallows kernel differs slightly from that of \cite{permutation_kernels}. We normalise the $n_{dis(\sigma,\sigma')}$ term relative to $d$, allowing a consistent selection of the  $\lambda$ parameter across permutations of different length. 

While the straightforward implementation of Kendall and Mallows kernels is of order $O(d^2)$, a $O(d \log d)$ variant based on merge-sort is given by \cite{kt_complexity}. 
 
Note that $K_{\tau}$ can also be expressed in terms of a feature map of $\binom{d}{2}$ elements,
$$\Phi_{\tau}(\sigma) = \left( \frac{1}{\sqrt{\binom{d}{2}}}(\mathbbm{1}_{\sigma(i) > \sigma(j)} - \mathbbm{1}_{\sigma(i) < \sigma(j)}) \right)_{1 \leq i < j \leq d},$$
so that
$$K_{\tau}(\sigma, \sigma')= \Phi(\sigma)^T\Phi(\sigma').$$

The Mallows kernel corresponds to a more complicated feature map, although still finite dimensional, given in \cite{mania2018kernel}.

We also define a third kernel based on Spearman's $\rho$. The (unnormalised) Spearman rank distance,
$$d_{\rho}(\sigma, \sigma') = \sum_{i=1}^{d} (\sigma(i)-\sigma'(i))^2 = || \sigma - \sigma' ||^2_2,$$
is a semimetric of negative type \citep{diaconis1988group}, therefore we can exploit the relationship between semimetrics of negative type and kernels from \cite{sejdinovic2013equivalence} to obtain a valid kernel. Writing $\sum_{i=0}^d \sigma(i) \sigma(i)'$ using vector notation as $\sigma^T\sigma'$, we have
\begin{align*}
d(\sigma,\sigma') &= K(\sigma,\sigma) + K(\sigma',\sigma') -2K(\sigma,\sigma') \\
d_{\rho}(\sigma,\sigma') &= \sigma^T\sigma + \sigma'^T\sigma' - 2\sigma^T\sigma' \\
\implies K_{\rho}(\sigma,\sigma') &= \sigma^T\sigma'.
\end{align*}
and the kernel's feature map is trivially
$$\Phi_{\rho}(\sigma)=\sigma.$$

Before introducing sampling algorithms, we derive an additional property for the above kernels: analytic formulas for their expected values at some fixed point $\sigma$ and values drawn from a given probability distribution $\sigma' \sim p$. The distribution of interest for approximating \eqref{eq:shapley_permutations} is the uniform distribution $U$. The expected value is straightforward to obtain for the Spearman and Kendall kernels: 
$$\forall \sigma \in \Pi, \quad \mathbb{E}_{\sigma' \sim U}[K_{\rho}(\sigma, \sigma')] = \frac{d(d+1)^2}{4},$$
$$\forall \sigma \in \Pi, \quad \mathbb{E}_{\sigma' \sim U}[K_{\tau}(\sigma, \sigma')] = 0.$$
The Mallows kernel is more difficult. Let $X$ be a random variable representing the number of inversions over all permutations of length $d$. Its distribution is studied in \cite{muir1898simple}, with probability generating function given as
$$\phi_d(x)= \prod_{j=1}^{d}\frac{1-x^j}{j(1-x)}.$$
There is no convenient form in terms of standard functions for its associated density function. From the probability generating function of $X$, we obtain the moment generating function:
\begin{align*}
M_d(t)&=\phi_d(e^t) \\
&=\prod_{j=1}^{d}\frac{1-e^{tj}}{j(1-e^t)} \\
&=\mathbb{E}[e^{tX}].
\end{align*}
The quantity $n_{\textrm{\textrm{dis}}}(I, \sigma)$, where $I$ is the identity permutation, returns exactly the number of inversions in $\sigma$. Therefore, we have
\begin{align*}
M_d(-\lambda / \textstyle\binom{d}{2}) &=\mathbb{E}[e^{-\lambda X / \binom{d}{2}}]\\
&=\mathbb{E}_{\sigma' \sim U}[K_{M}(I, \sigma')]. 
\end{align*}
The quantity $n_{\textrm{dis}}$ is right-invariant in the sense that $n_{\textrm{dis}}(\sigma, \sigma')= n_{\textrm{dis}}(\tau\sigma, \tau\sigma')$ for $\tau\in  \mathfrak{S}_d$ \citep{diaconis1988group}, so
\begin{align*}
\forall \tau \in \mathfrak{S}_d, \quad \mathbb{E}_{\sigma' \sim U}[K_{M}(I, \sigma')]&= \mathbb{E}_{\sigma' \sim U}[K_{M}(\tau I, \tau\sigma')] \\
&= \mathbb{E}_{\sigma' \sim U}[K_{M}(\tau I, \sigma')]\\
\forall \sigma \in \mathfrak{S}_d, \quad \mathbb{E}_{\sigma' \sim U}[K_{M}(I, \sigma')] &= \mathbb{E}_{\sigma' \sim U}[K_{M}(\sigma, \sigma')]\\
&=\prod_{j=1}^{d}\frac{1-e^{-\lambda j / \binom{d}{2}}}{j(1-e^{-\lambda / \binom{d}{2}})}.
\end{align*}

We now describe two greedy algorithms for generating point sets improving on simple Monte Carlo---kernel herding and sequential Bayesian quadrature.

\subsection{Kernel Herding}
\label{sec:kernel_herding}
A greedy process called ``kernel herding" for selecting (unweighted) quadrature samples in a reproducing kernel Hilbert space is proposed in \cite{kernel_herding}. The sample $n+1$ in kernel herding is given by
\begin{equation}
\label{eq:kernel_herding}
x_{n+1}= \argmax_x \Big[ \mathbb{E}_{x' \sim p}[K(x, x')] - \frac{1}{n + 1}\sum_{i=1}^{n} K(x, x_i) \Big],
\end{equation}
which can be interpreted as a greedy optimisation process selecting points for maximum separation, while also converging on the expected distribution $p$. In the case of Shapley value estimation, the samples are permutations $\sigma \in \mathfrak{S}_d$ and $p$ is a uniform distribution with $p(\sigma)=\frac{1}{\sigma!}, \forall \sigma \in \mathfrak{S}_d$. 

Kernel herding has time complexity $O(n^2)$ for $n$ samples, assuming the argmax can be computed in $O(1)$ time and $\mathbb{E}_{x' \sim p}[K(x, x')]$ is available. We have analytic formulas for $\mathbb{E}_{x' \sim p}[K(x, x')]$ from the previous section for the Spearman, Kendall, and Mallows kernels, and they give constant values depending only on the size of the permutation $d$. We compute an approximation to the argmax in constant time by taking a fixed number of random samples at each iteration and retaining the one yielding the maximum.

If certain conditions are met, kernel herding converges at the rate $O(\frac{1}{n})$, an improvement over $O(\frac{1}{\sqrt{n}})$ for standard Monte Carlo sampling. According to \cite{kernel_herding}, this improved convergence rate is achieved if the RKHS is universal, and mild assumptions are satisfied by the argmax (it need not be exact). Of the Spearman, Kendall and Mallows kernels, only the Mallows kernel has the universal property \citep{mania2018kernel}.

Next, we describe a more sophisticated kernel-based algorithm generating weighted samples.

\subsection{Sequential Bayesian Quadrature}
\label{sec:sbq}
Bayesian Quadrature \citep{OHAGAN1991245,rasmussen2003bayesian} (BQ) formulates the integration problem
$$Z_{f,p}= \int f(x)p(x) dx$$
as a Bayesian inference problem. Standard BQ imposes a Gaussian process prior on $f$ with zero mean and kernel function $K$. A posterior distribution is inferred over $f$ conditioned on a set of points $(x_0,x_1,\cdots,x_n)$. This implies a distribution on $Z_{f,p}$ with expected value
$$
\mathbb{E}_{GP}[Z]=z^TK^{-1}f(X),
$$
where $f(X)$ is the vector of function evaluations at points $(x_0,x_1,\cdots,x_n)$, $K^{-1}$ is the inverse of the kernel covariance matrix, and $z_i = \mathbb{E}_{x' \sim p}[K(x_i, x')]$. Effectively, for an arbitrary set of points, Bayesian quadrature solves the linear system $Kw=z$ to obtain a reweighting of the sample evaluations, yielding the estimate
$$Z \simeq w^T f(X).$$

An advantage of the Bayesian approach is that uncertainty is propagated through to the final estimate. Its variance is given by
\begin{equation}
\label{eq:sbq_variance}
\mathbb{V}[Z_{f,p}|f(X)]=\mathbb{E}_{x,x' \sim p}[K(x, x')] - z^TK^{-1}z.
\end{equation}
This variance estimate is used in \cite{sbq} to develop sequential Bayesian quadrature (SBQ), a greedy algorithm selecting samples to minimise Equation \ref{eq:sbq_variance}. This procedure, summarised in Algorithm \ref{alg:sbq}, is shown by \cite{sbq} to be related to optimally weighted kernel herding. Note that the expectation term in \eqref{eq:sbq_variance} and Algorithm \ref{alg:sbq} is constant and closed-form for all kernels considered here.

\begin{algorithm}[t]
\small
\DontPrintSemicolon
 \KwInput{$n$, kernel $K$, sampling distribution $p$, integrand $f$}
 $X_0 \leftarrow RandomSample(p)$\;
 $K^{-1} = I$\tcp*{Inverse of covariance matrix}
 $z_0 \leftarrow \mathbb{E}_{x' \sim p}[K(X_0, x')]$\;
  \For{$i\leftarrow 2$ \KwTo $n$}{
    $X_i \leftarrow \argmin\limits_{x} \mathbb{E}_{x,x' \sim p}[K(x, x')] - z^TK^{-1}z$\;
    $y \leftarrow \vec{0}$\;
    \For{$j\leftarrow 1$ \KwTo $i$}{
        $y_j = K(X_i, X_j)$\;
    }
    $K^{-1} \leftarrow CholeskyUpdate(K^{-1}, y)$\;
    $z_i \leftarrow \mathbb{E}_{x' \sim p}[K(X_i, x')]$\;
  }
  $w = z^TK^{-1}$\label{lst:line:bq_minimisation}\;
  \KwRet $w^Tf(X)$\;
 \caption{Sequential Bayesian Quadrature}
\label{alg:sbq}
\end{algorithm}

SBQ has time complexity $O(n^3)$ for $n$ samples if the argmin takes constant time, and an $O(n^2)$ Cholesky update algorithm is used to form $K^{-1}$, adding one sample at a time. In general, exact minimisation of Equation \ref{eq:sbq_variance} is not tractable, so as with kernel herding, we approximate the argmin by drawing a fixed number of random samples and choosing the one yielding the minimum variance.

\subsection{Error Analysis in RKHS}
Canonical error analysis of quasi Monte-Carlo quadrature is performed using the Koksma-Hlawka inequality \citep{hlawka1961funktionen,niederreiter}, decomposing error into a product of function variation and discrepancy of the sample set. We derive a version of this inequality for Shapley value approximation in terms of reproducing kernel Hilbert spaces. Our derivation mostly follows \cite{hickernell}, with modification of standard integrals to weighted sums of functions on $\mathfrak{S}_d$, allowing us to calculate discrepancies for point sets generated by kernel herding and SBQ with permutation kernels. The analysis is performed for the Mallows kernel, which is known to be a universal kernel \citep{mania2018kernel}. 

Given a symmetric, positive definite kernel $K$, we have a unique RKHS $\mathcal{F}$ with inner product $\langle \cdot,\cdot \rangle_K$ and norm $||\cdot||_K$, where the kernel reproduces functions $f \in \mathcal{F}$ by
\begin{align*}
f(\sigma) = \langle f, K(\cdot,\sigma) \rangle_K.
\end{align*}
Define error functional
$$\text{Err}(f,\Pi,w)=\frac{1}{d!}\sum_{\sigma \in \mathfrak{S}_d}f(\sigma) - \sum_{\tau \in \Pi}w_{\tau}f(\tau),$$
where $\Pi$ is a sample set of permutations and $w_{\tau}$ is the associated weight of sample $\tau$. Because the Mallows kernel is a universal kernel, the bounded Shapley value component functions $f(\sigma)$ belong to $\mathcal{F}$. Given that $\text{Err}(f,\Pi,w)$ is a continuous linear functional on $\mathcal{F}$ and assuming that it is bounded, by the Riesz  Representation Theorem, there is a function $\xi \in \mathcal{F}$ that is its representer: $\text{Err}(f,\Pi,w) = \langle \xi,f \rangle_K$. Using the Cauchy-Schwarz inequality, the quadrature error is bounded by
$$|\text{Err}(f,\Pi,w)|=|\langle \xi,f \rangle_K| \leq ||\xi||_K ||f||_K = D(\Pi,w) V(f),$$
where $D(\Pi,w)=||\xi||_K$ is the discrepancy of point set $\Pi$ with weights $w$ and $V(f)= ||f||_K$ is the function variation. The quantity $D(\Pi,w)$ has an explicit formula. As the function $\xi$ is reproduced by the kernel, we have:
\begin{align*}
\xi(\sigma') = \langle \xi, K(\cdot,\sigma') \rangle_K &= \text{Err}(K(\cdot,\sigma'),\Pi,w)\\
&=\frac{1}{d!}\sum_{\sigma \in \mathfrak{S}_d}K(\sigma,\sigma') - \sum_{\tau \in \Pi}w_{\tau}K(\tau,\sigma').
\end{align*}
Then the discrepancy can be obtained, using the fact that $\text{Err}(f,\Pi,w) = \langle \xi,f \rangle_K$, by
\begin{align}
\label{eq:discrepancy}
D(\Pi,w) {}&= ||\xi||_k = \sqrt{\langle \xi,\xi \rangle_K}=\sqrt{\text{Err}(\xi,\Pi,w)} \nonumber\\
{}& = \left( \frac{1}{d!}\sum_{\sigma \in \mathfrak{S}_d}\xi(\sigma) - \sum_{\tau \in \Pi}w_{\tau}\xi(\tau) \right)^{\frac{1}{2}} \nonumber\\
\begin{split}
{}&=\Bigg( \frac{1}{d!}\sum_{\sigma \in \mathfrak{S}_d} \left[ \frac{1}{d!}\sum_{\sigma' \in \mathfrak{S}_d}K(\sigma,\sigma') - \sum_{\tau \in \Pi}w_{\tau}K(\tau,\sigma) \right] \\ & \qquad - \sum_{\tau \in \Pi}w_{\tau}\left[ \frac{1}{d!}\sum_{\sigma \in \mathfrak{S}_d}K(\sigma,\tau) - \sum_{\tau' \in \Pi}w_{\tau'}K(\tau,\tau') \right] \Bigg)^{\frac{1}{2}}    
\end{split} \nonumber\\
{}&= \Bigg( \frac{1}{(d!)^2}\sum_{\sigma,\sigma' \in \mathfrak{S}_d} K(\sigma,\sigma') - \frac{2}{d!}\sum_{\sigma \in \mathfrak{S}_d}\sum_{\tau \in \Pi}w_{\tau}K(\tau,\sigma) + \sum_{\tau,\tau' \in \Pi}w_{\tau}w_{\tau'}K(\tau,\tau') \Bigg)^{\frac{1}{2}} \nonumber\\
{}&= \Bigg( \mathbb{E}_{\sigma, \sigma' \sim U}[K(\sigma, \sigma')]  - 2\sum_{\tau \in \Pi}w_{\tau} \mathbb{E}_{\sigma \sim U}[K(\tau,\sigma)] + \sum_{\tau,\tau' \in \Pi}w_{\tau}w_{\tau'}K(\tau,\tau') \Bigg)^{\frac{1}{2}}.
\end{align}

It can be seen that kernel herding (Equation \ref{eq:kernel_herding}) greedily minimises $D(\Pi, w)^2$ with constant weights $\frac{1}{n}$, by examining the reduction in $D(\Pi,\frac{1}{n})^2$ obtained by the addition of a sample to $\Pi$. The kernel herding algorithm for sample $\sigma_{n+1} \in \Pi$ is
$$
\sigma_{n+1} = \argmax_{\sigma} \left [ \mathbb{E}_{\sigma' \sim U} [K(\sigma,\sigma')] - \frac{1}{n+1} \sum_{i=1}^n K(\sigma,\sigma_i) \right ].
$$
Note that, since $K(\cdot,\cdot)$ is right-invariant, the quantity $\mathbb{E}_{\sigma' \sim U} [K(\sigma,\sigma')]$ does not depend on $\sigma$, so the argmax above is simply minimizing $\sum_{i=1}^n K(\sigma,\sigma_i)$.  On the other hand, denoting the identity permutation by $I$, for a newly selected permutation sample $\pi$:
\begin{align*}
    D(\Pi, \textstyle{\frac{1}{n}})^2 - D(\Pi \cup \{\pi\},\textstyle{\frac{1}{n+1}})^2 &= 2 \!\!\! \sum_{\tau \in \Pi \cup \{\pi\}} \!\! \frac{1}{n+1} \mathbb{E}_{\sigma \sim U} [K(\tau,\sigma)] - 2 \sum_{\tau \in \Pi} \frac{1}{n} \mathbb{E}_{\sigma \sim U} [K(\tau,\sigma)] \\
    & \quad + \sum_{\tau,\tau' \in \Pi} \frac{1}{n^2} K(\tau,\tau') - \sum_{\tau,\tau' \in \Pi \cup \{\pi\}} \frac{1}{(n+1)^2} K(\tau,\tau') \\
    &= 2 \frac{n+1}{n+1} \mathbb{E}_{\sigma \sim U} [K(I,\sigma)] - 2 \frac{n}{n} \mathbb{E}_{\sigma \sim U} [K(I,\sigma)] \\
    & \quad + \sum_{\tau,\tau' \in \Pi} \frac{2n+1}{n^2(n+1)^2} K(\tau,\tau') - 2 \sum_{\tau \in \Pi} \frac{1}{(n+1)^2} K(\tau,\pi) \\
    &= \frac{K(I,I)}{(n+1)^2} + \sum_{\tau,\tau' \in \Pi} \frac{2n+1}{n^2(n+1)^2} K(\tau,\tau') \\
    & \qquad - \frac{2}{(n+1)^2} \sum_{\tau \in \Pi}  K(\tau,\pi),
\end{align*}
where both equalities use right-invariance.  Note that the first two summands in the last expression are constants (i.e., do not depend on the choice of $\pi$), so maximizing this quantity is the same as minimizing $\sum_{\tau \in \Pi} K(\tau,\pi)$, i.e., the same as the kernel herding optimization subproblem.

Furthermore, we can show that Bayesian quadrature minimises squared discrepancy via optimisation of weights. Writing $z_i = \mathbb{E}_{\sigma' \sim p}[K(\sigma_i, \sigma')]$ and switching to vector notation we have
$$D(\Pi, w)^2 = c -2w^Tz + w^TKw,$$
where the first term is a constant not depending on $w$. Taking the gradient with respect to $w$, setting it to 0, and solving for $w$, we obtain:
\begin{align}
\label{eq:bq_minimisation}
\nabla D(\Pi, w)^2 &= -2z + 2w^TK=0 \nonumber\\
w^* &= z^TK^{-1},
\end{align}
where \eqref{eq:bq_minimisation} is exactly line \ref{lst:line:bq_minimisation} of Algorithm \ref{alg:sbq}.

We use the discrepancy measure in \eqref{eq:discrepancy} for numerical experiments in Section \ref{sec:discrepancy_eval} to determine the quality of a set of sampled permutations in a way that is independent of the integrand $f$. 

\section{Sampling Permutations on \texorpdfstring{$\mathbb{S}^{d-2}$}{S d-2}}
\label{sec:sphere}
Kernel herding and sequential Bayesian quadrature directly reduce the discrepancy of the sampled permutations via greedy optimisation. We now describe two approaches to sampling permutations of length $d$ based on a relaxation to the Euclidean sphere $\mathbb{S}^{d-2} = \left\{ x \in \mathbf{R}^{d-1} : \left\| x \right\| = 1 \right\}$, where the problem of selecting well-distributed samples is simplified. We describe a simple procedure for mapping points on the surface of this hypersphere to the nearest permutation, where the candidate nearest neighbours form the vertices of a Cayley graph inscribing the sphere. This representation provides a natural connection between distance metrics over permutations, such as Kendall's tau and Spearman's rho, and Euclidean space. We show that samples taken uniformly on the sphere result in a uniform distribution over permutations, and evaluate two unbiased sampling algorithms. Our approach is closely related to that of \cite{permutations_angular}, where an angular view of permutations is used to solve inference problems.

\subsection{Spheres, Permutohedrons, and the Cayley Graph}

Consider the projection of permutations $\sigma \in \mathfrak{S}_d$ as points in $\mathbf{R}^d$, where the $i$-th coordinate is given by $\sigma^{-1}(i)$. These points form the vertices of a polytope known as the permutohedron \citep{permutohedron}. The permutohedron is a $d-1$ dimensional object embedded in $d$ dimensional space, lying on the hyperplane given by
$$\sum_{i=1}^{d}\sigma^{-1}(i) = \frac{d(d+1)}{2},$$
with normal vector
\begin{equation}
\label{eq:normal}
\vec{n}= \begin{bmatrix}
           \frac{1}{\sqrt{d}}\\
           \frac{1}{\sqrt{d}} \\
           \vdots \\
           \frac{1}{\sqrt{d}}
         \end{bmatrix},
\end{equation}
and inscribing the hypersphere $\mathbb{S}^{d-2}$ lying on the hyperplane, defined by
$$\sum_{i=1}^{d}\sigma^{-1}(i)^2=\frac{d(d+1)(2d+1)}{6}.$$

\begin{figure}
\centering
\begin{minipage}{.48\textwidth}
  \centering
  \includegraphics[width=1.0\linewidth]{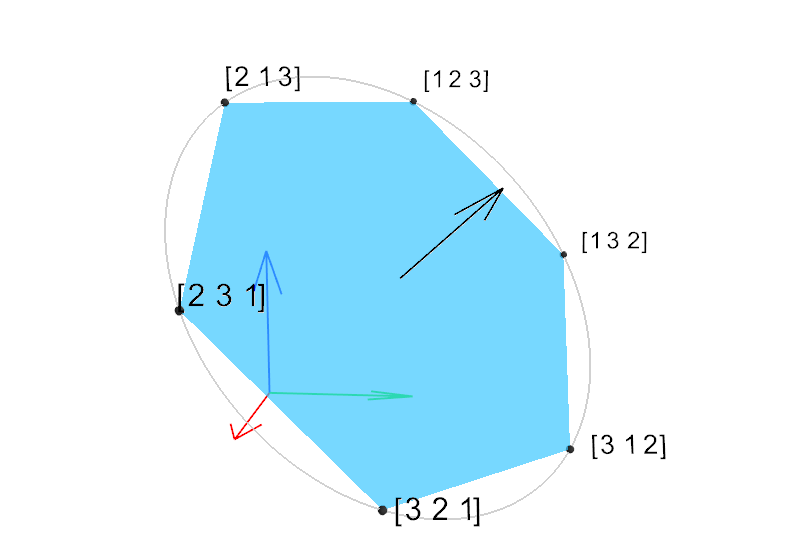}
  \captionof{figure}{Cayley Graph of $d=3$}
  \label{fig:cayley2d}
\end{minipage}%
\begin{minipage}{.48\textwidth}
  \centering
  \includegraphics[width=1.0\textwidth]{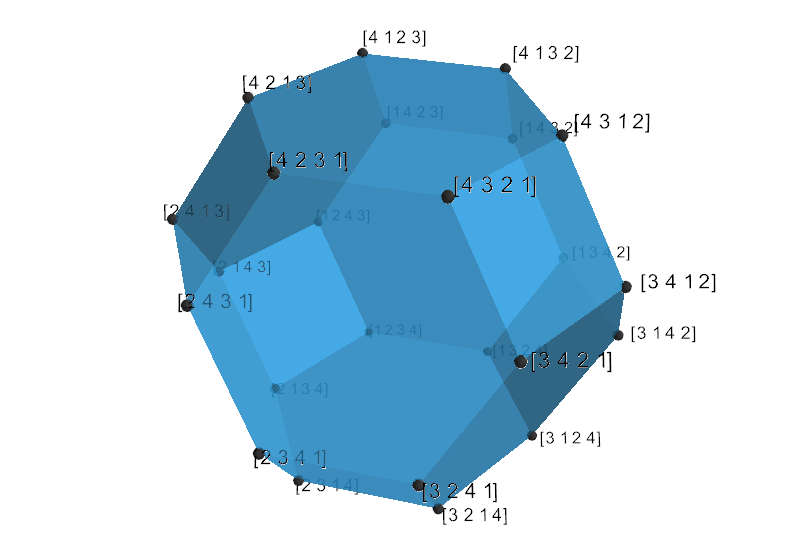}
  \captionof{figure}{Cayley Graph of $d=4$}
  \label{fig:cayley}
\end{minipage}
\end{figure}

Inverting the permutations at the vertices of the permutohedron gives a Cayley graph of the symmetric group with adjacent transpositions as the generating set. Figure \ref{fig:cayley2d} shows the Cayley graph for $\mathfrak{S}_3$, whose vertices form a hexagon inscribing a circle on a hyperplane, and Figure \ref{fig:cayley} shows the Cayley graph of $\mathfrak{S}_4$ projected into three dimensions (its vertices lie on a hyperplane in four dimensions). Each vertex $\sigma^{-1}$ in the Cayley graph has $d-1$ neighbours, where each neighbour differs by exactly one adjacent transposition (one bubble-sort operation). Critically for our application, this graph has an interpretation in terms of distance metrics on permutations. The Kendall-tau distance is  the graph distance in the vertices of this polytope, and Spearman distance is the squared Euclidean distance between two vertices \citep{generalised_permutation}. Additionally, the antipode of a permutation is its reverse permutation. With this intuition, we use the hypersphere as a continuous relaxation of the space of permutations, where selecting samples far apart on the hypersphere corresponds to sampling permutations far apart in the distance metrics of interest.

We now describe a process for sampling from the set of permutations inscribing $\mathbb{S}^{d-2}$. First, shift and scale the permutohedron to lie around the origin with radius $r=1$. The transformation on vertex $\sigma^{-1}$ is given by
\begin{equation}
\label{eq:scale}
\hat{\sigma}^{-1}=\frac{\sigma^{-1}-\mu}{||\sigma^{-1}||},
\end{equation}
where $\mu=(\frac{d+1}{2},\frac{d+1}{2},\cdots)$ is the mean vector of all permutations, and $||\sigma^{-1}||=\sqrt{\sum_{i=1}^{d}\sigma^{-1}(i)^2}$.

Now select some vector $x$ of dimension $d-1$, say, uniformly at random from the surface of $\mathbb{S}^{d-2}$. Project $x$ onto the hyperplane in $\mathbb{R}^d$ using the following $(d-1) \times d$ matrix:

\begin{align*}
    U = \begin{bmatrix}
    1 & -1 & 0 & \dots & 0 \\
    1 & 1 & -2 & \dots & 0  \\
    & \vdots & & \ddots&\\
    1 & 1 & 1 & \dots & -(d-1)\\
    \end{bmatrix}.
\end{align*}
It is easily verifiable that this basis of row vectors is orthogonal to hyperplane normal $\vec{n}$. Normalising the row vectors of $U$ gives a transformation matrix $\hat{U}$ used to project vector $x$ to the hyperplane by 
$$\Tilde{x}=\hat{U}^Tx,$$
so that
$$\Tilde{x}^T\vec{n}=0.$$

Given $\Tilde{x}$, find the closest permutation $\hat{\sigma}^{-1}$ by maximising the inner product

\begin{equation}
\label{eq:argmax}
\hat{y}=\argmax_{\hat{\sigma}^{-1}} \Tilde{x}^T\hat{\sigma}^{-1}.
\end{equation}

This maximisation is simplified by noting that $\hat{\sigma}^{-1}$ is always a reordering of the same constants ($\hat{\sigma}^{-1}$ is a scaled and shifted permutation). The inner product is therefore maximised by matching the largest element in $\hat{\sigma}^{-1}$ against the largest element in $\Tilde{x}$, then proceeding to the second-largest, and so on. Thus the argmax is performed by finding the permutation corresponding to the order type of $\Tilde{x}$, which is order-isomorphic to the coordinates of $\Tilde{x}$. The output $\hat{y}$ is a vertex on a scaled permutohedron --- to get the corresponding point on the Cayley graph, undo the scale/shift of Eq. \ref{eq:scale} to get a true permutation, then invert that permutation:
\begin{equation}
\label{eq:inverse}
y=\text{inverse}(\hat{y}||\sigma^{-1}|| + \mu).
\end{equation}
In fact, both Eq. \ref{eq:argmax} and \ref{eq:inverse} can be simplified via a routine $\textit{argsort}$, defined by
$$\text{argsort}(a)=b,$$
such that
$$a_{b_0} \leq a_{b_1} \leq \cdots \leq a_{b_n}.$$
In other words, $b$ contains the indices of the elements of $a$ in sorted position. 

Algorithm \ref{alg:sample} describes the end-to-end process of sampling. We use the algorithm of \cite{knuth} for generating points uniformly at random on $\mathbb{S}^{d-2}$: sample from $d-1$ independent Gaussian random variables and normalise the resulting vector to have unit length. We now make the claim that Algorithm \ref{alg:sample} is unbiased.

\begin{algorithm}[t]
\small
\DontPrintSemicolon
 \KwOutput{$\sigma$, a permutation of length $d$}
 $x \leftarrow N(0, 1)$ \label{alg:sample1}  \tcp*{x is a vector of $d-1$ i.i.d. normal samples} 
 $x \leftarrow \frac{x}{||x||}$ \label{alg:sample2} \tcp*{x lies uniformly on $\mathbb{S}^{d-2}$}
 $\Tilde{x}=\hat{U}^Tx$ \label{alg:sample3} \;
 $\sigma \leftarrow \text{argsort}(\Tilde{x})$\tcp*{$\sigma$ is a uniform random permutation}
 \caption{Sample permutation from $\mathbb{S}^{d-2}$}
 \label{alg:sample}
\end{algorithm}

\begin{theorem}
\label{th:uniform}
Algorithm \ref{alg:sample} generates permutations uniformly at random, i.e., $Pr(\sigma) = \frac{1}{d!},  \forall \sigma \in \mathfrak{S}_d$, from a uniform random sample on $\mathbb{S}^{d-2}$.
\end{theorem}

\begin{proof}
The point $x \in \mathbb{S}^{d-2}$ from Algorithm \ref{alg:sample}, line \ref{alg:sample2}, has multivariate normal distribution with mean 0 and covariance $\Sigma=aI$ for some scalar $a$ and $I$ as the identity matrix. $\Tilde{x}=\hat{U}^Tx$ is an affine transformation of a multivariate normal and so has covariance
\begin{align*}
\mathrm{Cov(\Tilde{x})} &= \hat{U}^T\Sigma\hat{U} \\
&= a\hat{U}^TI\hat{U} \\
&= a\hat{U}^T\hat{U}.
\end{align*}
The $d \times d$ matrix $\hat{U}^T\hat{U}$ has the form
\begin{align*}
\hat{U}^T\hat{U} = \begin{bmatrix}
    \frac{d-1}{d} & \frac{-1}{d} &  \dots & \frac{-1}{d} \\
    \frac{-1}{d} & \frac{d-1}{d} & \dots & \frac{-1}{d}  \\
    & \vdots & \ddots&\\
    \frac{-1}{d} & \frac{-1}{d} &\dots & \frac{d-1}{d}\\
    \end{bmatrix},
\end{align*}
with all diagonal elements $\frac{d-1}{d}$ and off diagonal elements $\frac{-1}{d}$, and so $\Tilde{x}$ is equicorrelated. Due to equicorrelation, $\Tilde{x}$ has order type such that $\forall \Tilde{x}_i,\Tilde{x}_j \in x, i \neq j: Pr(\Tilde{x}_i < \Tilde{x}_j) = \frac{1}{2}$. In other words, all orderings of $\Tilde{x}$ are equally likely. The function \textit{argsort} implies an order-isomorphic bijection, that is, \textit{argsort} returns a unique permutation for every unique ordering over its input. As every ordering of $\Tilde{x}$ is equally likely, Algorithm \ref{alg:sample} outputs permutations $\sigma \in \mathfrak{S}_d$ with $p(\sigma)=\frac{1}{d!},  \forall \sigma \in \mathfrak{S}_d.$
\end{proof}

Furthermore, Equation \ref{eq:argmax} associates a point on the surface of $\mathbb{S}^{d-2}$ to the nearest permutation. This implies that there is a Voronoi cell on the same surface associated with each permutation $\sigma_i$, and a sample $\Tilde{x}$ is associated with $\sigma_i$ if it lands in its cell. Figure \ref{fig:cayley_voronoi} shows the Voronoi cells on the hypersphere surface for $d=4$, where the green points are equidistant from nearby permutations. A corollary of Theorem \ref{th:uniform} is that these Voronoi cells must have equal measure, which is easily verified for $d=4$.

\subsection{Orthogonal Spherical Codes}
\label{sec:orthogonal}
Having established an order isomorphism $\mathbb{S}^{d-2} \rightarrow \mathfrak{S}_d$, we consider selecting well-distributed points on $\mathbb{S}^{d-2}$. Our first approach, described in Algorithm \ref{alg:sample_k}, is to select $2(d-1)$ dependent samples on $\mathbb{S}^{d-2}$ from a basis of orthogonal vectors. Algorithm \ref{alg:sample_k} uses the Gram-Schmidt process to incrementally generate a random basis, then converts each component and its reverse into permutations by the same mechanism as Algorithm \ref{alg:sample}. The cost of each additional sample is proportional to $O(d^2)$. This sampling method is related to orthogonal Monte Carlo techniques discussed in \cite{choromanski2019unifying}. Writing $v([\sigma]_{i-1} \cup\{i\}) - v([\sigma]_{i-1}) = g_i(\sigma)$, the Shapley value estimate for samples given by Algorithm \ref{alg:sample_k} is
\begin{equation}
\label{eq:orth_est}
\bar{\Sh}^{\textrm{orth}}_i(v)=\frac{1}{n} \sum_{\ell=1}^{n/k} \sum_{j=1}^{k} g_i(\sigma_{\ell j}),
\end{equation}
where $(\sigma_{\ell 1},\sigma_{\ell 2},\cdots,\sigma_{\ell k})$ are a set of correlated samples and $n$ is a multiple of $k$.

\begin{algorithm}[t]
\small
\DontPrintSemicolon
 $X \sim N(0,1)_{k/2,d}$\tcp*{iid. normal random Matrix}
 $Y \leftarrow 0_{k,d}$\tcp*{Matrix storing output permutations}
 \For{$i\leftarrow 1$ \KwTo $k/2$}{
   \For{$j\leftarrow 1$ \KwTo $i$}{
        $X_i \leftarrow X_i - X_jX_i^T \cdot X_j$ \tcp*{Gram-Schmidt process}
   }
   
    $X_i \leftarrow \frac{X_i}{||X_i||}$\;
    $Y_{2i} \leftarrow \text{argsort}(\hat{U}^TX_i)$\;
    $Y_{2i+1} \leftarrow \text{argsort}(\hat{U}^T(-X_i))$\;
 }
 \KwRet $Y$
 \caption{Sample $k = 2(d-1)$ permutations from $\mathbb{S}^{d-2}$}
\label{alg:sample_k}
\end{algorithm}

\begin{prop}
$\bar{\Sh}^{\textrm{orth}}_i(v)$ is an unbiased estimator of $\Sh_i(v)$.
\end{prop}
\begin{proof}
The Shapley value $\Sh_i(v)$ is equivalently expressed as an expectation over uniformly distributed permutations:
\begin{align*}
\Sh_i(v) &=\frac{1}{|N|!} \sum_{\sigma \in \mathfrak{S}_d} \big[ v([\sigma]_{i-1} \cup\{i\}) - v([\sigma]_{i-1}) \big] \\
\Sh_i(v) &= \mathbb{E}_{\sigma \sim U}[g_i(\sigma)].
\end{align*}
The distribution of permutations drawn as orthogonal samples is clearly symmetric, so $p(\sigma_{\ell,j})=p(\sigma_{\ell,m})$ for any two indices $j,m$ in a set of $k$ samples, and $\mathbb{E}[g_i(\sigma_{\ell,j})]=\mathbb{E}[g_i(\sigma_{\ell,m}))]=\mathbb{E}[g_i(\sigma^{ortho})]$.
As the estimator \eqref{eq:orth_est} is a sum, by the linearity of expectation
$$\mathbb{E}[\bar{\Sh}^{\textrm{orth}}_i(v)]=\frac{1}{n} \sum_{\ell=1}^{n/k} \sum_{j=1}^{k} \mathbb{E}[g_i(\sigma_{\ell j})]=\mathbb{E}[g_i(\sigma^{ortho})].$$
By Theorem \ref{th:uniform}, the random variable $\sigma^{ortho}$ has a uniform distribution if its associated sample $x \in \mathbb{S}^{d-2}$ is drawn with uniform distribution. Let $x$ be a component of a random orthogonal basis. If the random basis is drawn with equal probability from the set of orthogonal matrices of order $d-1$ (i.e. with Haar distribution for the orthogonal group), then it follows that $\mathbb{E}[g_i(\sigma^{ortho})]=\mathbb{E}_{\sigma \sim U}[g_i(\sigma)]$. The Gram-Schmidt process applied to a square matrix with elements as i.i.d.~standard normal random variables yields a random orthogonal matrix with Haar distribution \citep{mezzadri2006generate}. Therefore
\begin{align*}
\Sh_i(v) &= \mathbb{E}_{\sigma \sim U}[g_i(\sigma)] = \mathbb{E}_{\sigma \sim U}[g_i(\sigma)] \\
&= \mathbb{E}[\bar{\Sh}^{\textrm{orth}}_i(v)].
\end{align*}
\end{proof}

The variance of the estimator \eqref{eq:orth_est} can be analysed similarly to the antithetic sampling of Section \ref{sec:antithetic}, extended to $k$ correlated random variables. By extension of the antithetic variance in Equation \ref{eq:antithetic_variance}, we have
\begin{align*}
\mathrm{Var}(\bar{\Sh}^{\textrm{orth}}_i(v))=\frac{1}{n} \sum_{\ell=1}^{n/k} \sum_{j,m=1}^{k} \mathrm{Cov}(g(\sigma_{\ell j}),g(\sigma_{\ell m})).
\end{align*}
The variance is therefore minimised by selecting $k$ negatively correlated samples. Our experimental evaluation in Section \ref{sec:evaluation} suggests that, for the domain of interest, orthogonal samples on the sphere are indeed strongly negatively correlated, and the resulting estimators are more accurate than standard Monte Carlo and antithetic sampling in all evaluations.

Samples from Algorithm \ref{alg:sample_k} can also be considered as a type of spherical code. Spherical codes describe configurations of points on the unit sphere maximising the angle between any two points (see \cite{conway}). A spherical code $A(n,\phi)$ gives the maximum number of points in dimension $n$ with minimum angle $\phi$. The orthonormal basis and its antipodes trivially yield the optimal code $A(d-1, \frac{\pi}{2})=2(d-1)$.

From their relative positions on the Cayley graph we obtain bounds on the Kendall tau kernel $K_{\tau}(\sigma,\sigma')$ from Section \ref{sec:kernel_methods} for the samples of Algorithm \ref{alg:sample_k}. The angle between vertices of the Cayley graph is related to $K_{\tau}(\sigma,\sigma')$ in that the maximum kernel value of 1 occurs for two permutations at angle 0 and the minimum kernel value of -1 occurs for a permutation and its reverse, separated by angle $\pi$. As the angle between two points $(x,x')$ on $\mathbb{S}_{d-2}$ increases from 0 to $\pi$, the kernel $K_{\tau}(\sigma,\sigma')$ for the nearest permutations $(\sigma,\sigma')$ decreases monotonically and linearly with the angle, aside from quantisation error. If the angle between two distinct points $(x,x')$ in our spherical codes is $\frac{\pi}{2}$, we obtain via the map, $\mathbb{S}^{d-2} \rightarrow \mathfrak{S}_d$, the permutations $(\sigma,\sigma')$ such that
$$|K_{\tau}(\sigma,\sigma')| \leq 1/2 + \epsilon,$$
with some small constant quantisation error $\epsilon$. Figure \ref{fig:ortho} shows $k=6$ samples for the $d=4$ case.  This is made precise in the following result.  Note that the statement and its proof are in terms of $\sigma$ and $\sigma'$ instead of their inverses (which label the vertices of the permutohedron in our convention), for simplicity; without this change, the meaning is the same, since $n_\textrm{dis}(\sigma,\sigma') = n_\textrm{dis}(\sigma^{-1},\sigma'^{-1})$ and $A(\sigma)^T A(\sigma') = A(\sigma^{-1})^T A(\sigma'^{-1})$ for any permutations $\sigma$, $\sigma'$.  First, let $\rho = \sqrt{d(d^2-1)/12}$, so that the map $A(y) = (y - \mu)/\rho$ maps the permutohedron to an isometric copy of $\mathbb{S}^{d-2}$ centered at the origin in $\mathbb{R}^d$, the intersection of the unit sphere $\mathbb{S}^{d-1}$ with the hyperplane orthogonal to $\vec{n}$.

\begin{restatable}{theorem}{ktau}
\label{thm:k_tau}
Suppose $\sigma, \sigma' \in \mathfrak{S}_d$.  Then
$$
-2 + 4\left (\frac{1-K_\tau(\sigma,\sigma')}{2} \right )^{3/2} \leq A(\sigma)^T A(\sigma') - 3K_\tau(\sigma,\sigma') + O(d^{-1}) \leq 2 - 4 \left ( \frac{1+K_\tau(\sigma,\sigma')}{2} \right )^{3/2}
$$
and, if $A(\sigma)^T A(\sigma') = o(1)$, then
$$
| K_\tau(\sigma,\sigma') | \leq 1/2 +o(1).
$$
\end{restatable}

Proof of the above can be found in Appendix \ref{app:proof}. Theorem \ref{thm:k_tau} is a kind of converse to the so-called Rearrangement Inequality, which states that the maximum dot product between a vector and a vector consisting of any permutation of its coordinates is maximized when the permutation is the identity and minimized when it is the reverse identity.  Here, we show what happens in between: as one varies from the identity to its reverse one adjacent transposition at a time, the dot product smoothly transitions from maximal to minimal, with some variability across permutations having the same number of inversions.  Interestingly, we do not know if the above bound is the best possible.  A quick calculation shows that, letting $k \approx d 2^{-1/3}$ be an integer, the permutation
$$
\pi = (k,k-1,\ldots,2,1,k+1,k+2,\ldots,d-1,d)
$$
has $\nu(\pi) = I^T \pi = d^3 (1/4 +o(1))$, i.e, $A(I)^T A(\pi) \approx 0$.  However, $\pi$ admits $d^2(2^{-5/3}+o(1))$ inversions, whence $K_\tau(I,\pi) \approx 1-2^{-2/3} \approx 0.37 < 1/2$.

Figure \ref{fig:ortho_distribution} shows the distribution of pairs of unique samples taken from random vectors, versus unique samples from an orthogonal basis, at $d=10$. Samples corresponding to orthogonal vectors are tightly distributed around $K_{\tau}(\sigma,\sigma')=0$, and pairs corresponding to a vector and its antipodes are clustered at $K_{\tau}(\sigma,\sigma')=-1$. Figure \ref{fig:ortho_bounds} plots the bounds from Theorem \ref{thm:k_tau} relating the dot product of vectors on $\mathbb{S}^{d-2}$ to the Kendall tau kernel at $d=15$.
\begin{figure}
\centering
\begin{minipage}{.5\textwidth}
  \centering
  \includegraphics[width=.99\linewidth]{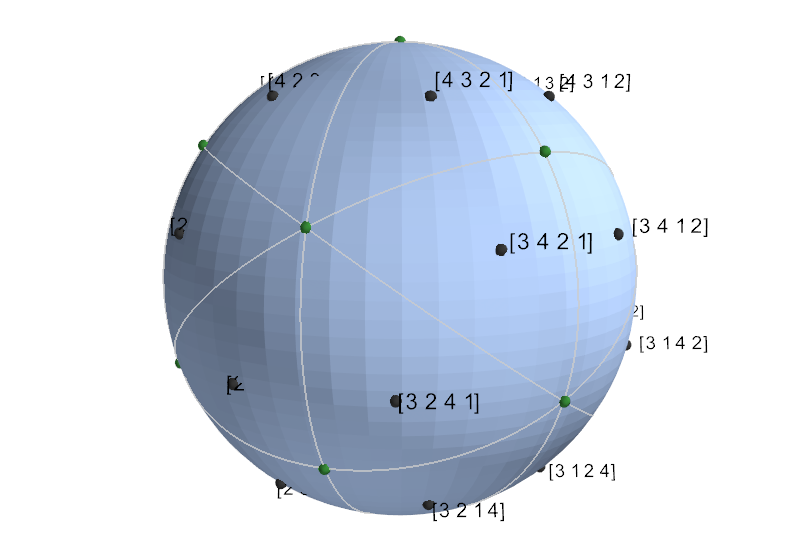}
  \captionsetup{width=0.9\linewidth}
  \captionof{figure}{Voronoi cells for permutations on the n-sphere have equal measure. Uniform samples on the n-sphere mapped to these cells result in uniform samples of permutations.}
  \label{fig:cayley_voronoi}
\end{minipage}%
\begin{minipage}{.5\textwidth}
  \centering
  \includegraphics[width=.99\linewidth]{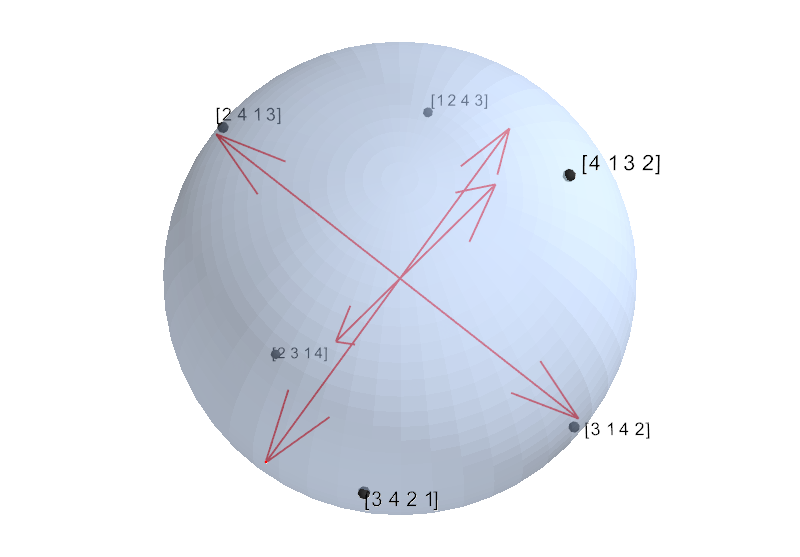}
  \captionsetup{width=0.9\linewidth}
  \captionof{figure}{Orthogonal spherical codes: The permutations associated with each orthogonal vector on the n-sphere must be separated by a certain graph distance.}
  \label{fig:ortho}
\end{minipage}
\\
\vspace{10mm}
\begin{minipage}{.5\textwidth}
  \centering
  \includegraphics[width=.99\linewidth]{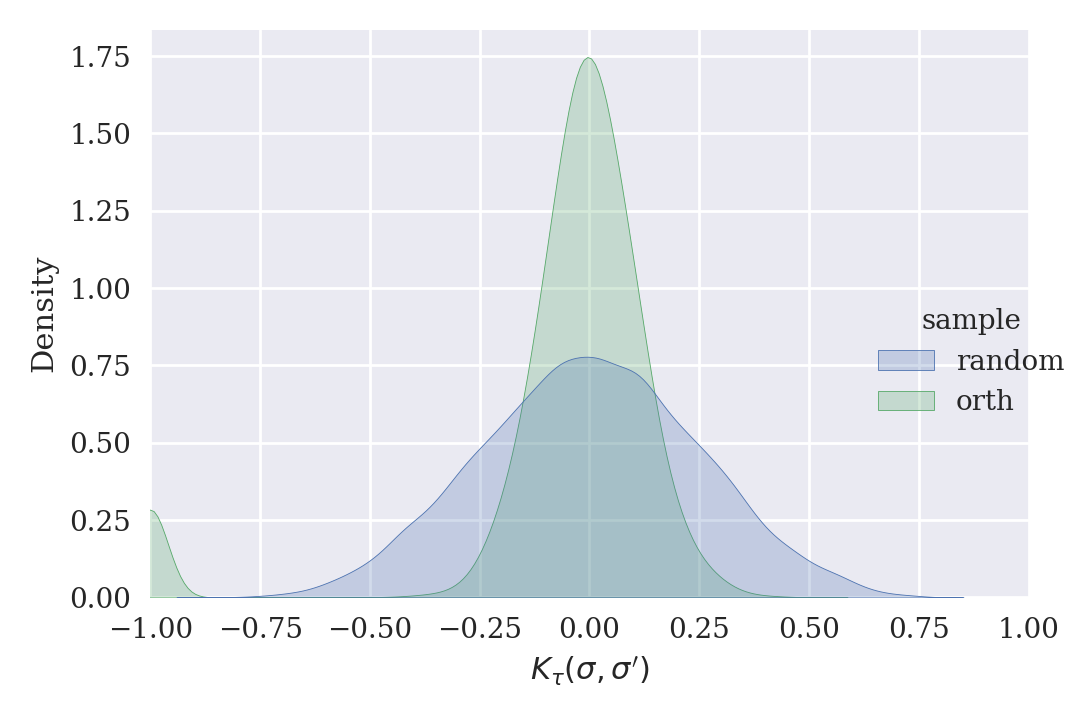}
  \captionsetup{width=0.9\linewidth}
  \captionof{figure}{Kernel density estimate of the $K_{\tau}$ similarity of pairs of unique permutations drawn from orthogonal vectors or random vectors on the n-sphere. The leftmost peak for orth corresponds to the antipode samples. Orthogonal samples do not generate highly similar permutations.}
  \label{fig:ortho_distribution}
\end{minipage}%
\begin{minipage}{.5\textwidth}
  \centering
  \includegraphics[width=.99\linewidth]{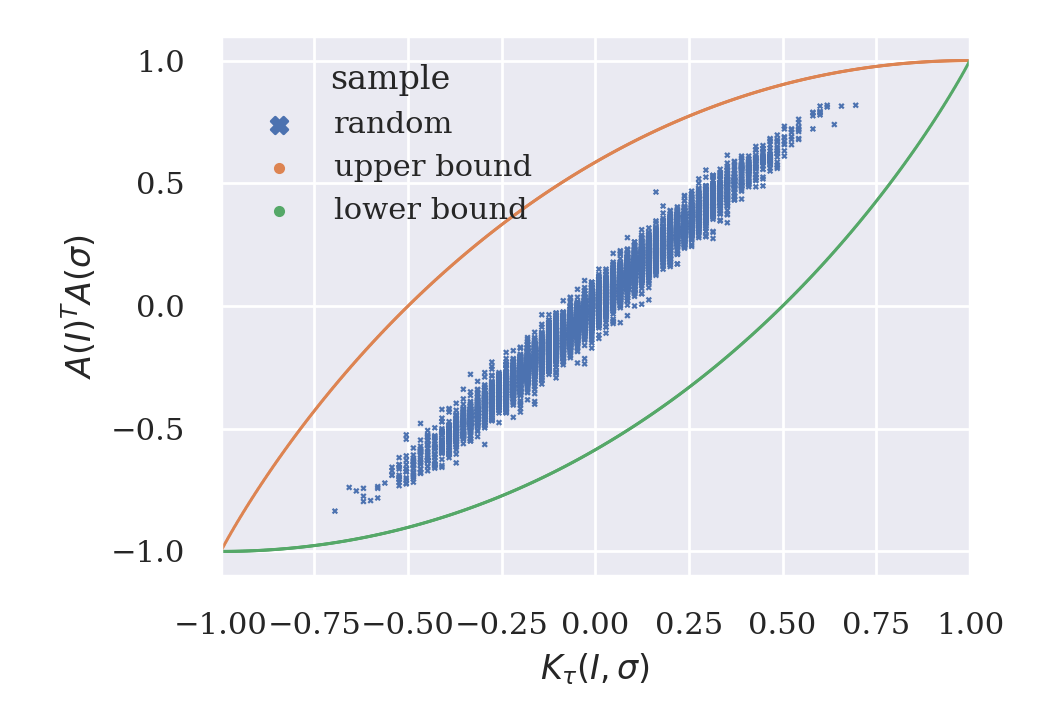}
  \captionsetup{width=0.9\linewidth}
  \captionof{figure}{The dot product of two points on $\mathbb{S}^{d-2}$ is closely related to the graph distance $K_{\tau}(I,\sigma)$ between the associated permutations.}
  \label{fig:ortho_bounds}
\end{minipage}
\end{figure}

\subsection{Sobol Sequences on the Sphere}
\label{sec:sobol}
We now describe another approach to sampling permutations via $\mathbb{S}^{d-2}$, based on standard quasi-Monte Carlo techniques. Low discrepancy point sets on the unit cube $[0,1)^{d-2}$ may be projected to $\mathbb{S}^{d-2}$ via area preserving transformations. Such projections are discussed in depth in \cite{brauchart2012quasi,points}, where they are observed to have good properties for numerical integration. Below we define transformations in terms of the inverse cumulative distribution of the generalised polar coordinate system and use transformed high-dimensional Sobol sequences to obtain well-distributed permutations.

In the generalised polar coordinate system of \cite{blumenson1960derivation}, a point on $\mathbb{S}^{d-2}$ is defined by radius $r$ (here $r=1$) and $d-2$ angular coordinates $(r, \varphi_1,\varphi_2,\cdots,\varphi_{d-2})$, where $(\varphi_1,\cdots,\varphi_{d-3})$ range from $[0,\pi]$ and $\varphi_{d-2}$ ranges from $[0,2\pi]$.

The polar coordinates on the sphere are independent and have probability density functions 
$$f(\varphi_{d-2})=\frac{1}{2\pi},$$
and for $1\leq j < d-2$:
$$f(\varphi_j)=\frac{1}{B(\frac{d-j-1}{2},\frac{1}{2})}\sin^{(d-j-2)}(\varphi_j),$$
where $B$ is the beta function. The above density function is obtained by normalising the formula for the surface area element of a hypersphere to integrate to 1 \citep{blumenson1960derivation}. The cumulative distribution function for the polar coordinates is then
$$F_j(\varphi_j) = \int_0^{\varphi_j}f_j(u)du.$$
As per standard inverse transform sampling, we draw samples $x \in [0,1)^{d-2}$ uniformly from the unit cube and project them to polar coordinates uniformly distributed on the sphere as $\varphi_j=F_j^{-1}(x_j)$. $F_j^{-1}$ can be obtained quickly via a root finding algorithm, such as the bracketing method described in \cite{numerical_recipes}.

The points $x \in [0,1]^{d-2}$ are generated using the Sobol sequence \citep{sobol1967distribution}, also referred to as $(t,s)$-sequences in base 2. Analogously to our discrepancy for functions of permutations in Equation \ref{eq:discrepancy}, derived with the Mallows kernel, Sobol points can be shown to minimise a discrepancy for the kernel
$$K(x,x')=\prod_{i=1}^d\min(1-x_j,1-x'_j),$$
with $x,x' \in [0,1]^d$, where the discrepancy decreases at the rate $O(\frac{(\log n)^d}{n})$ (see \cite{dick2010digital}). Sobol points are relatively inexpensive to generate compared with other algorithms discussed in this paper, although explicit convergence rates for discrepancy on the cube do not translate to $\mathbb{S}^{d-2}$ or $\mathfrak{S}_d$.

Combining Sobol points with inverse transform sampling yields uniformly distributed points on $\mathbb{S}^{d-2}$. To map these points to permutations, we project from $[0,1)^{d-1}$ to the hyperplane in $\mathbb{R}^{d}$ containing the permutohedron (such that points are orthogonal to the normal in Eq. \ref{eq:normal}) using the matrix $\hat{U}$, and apply argsort to obtain permutations.

Combining all of the above, Algorithm \ref{alg:sobol} describes the process of generating permutation samples from a Sobol sequence. Figure \ref{fig:sobol_sphere} shows 200 Sobol points distributed on the surface of the sphere. As our Sobol sequence and inverse CDF sampling generate points uniformly distributed on the n-sphere, Theorem \ref{th:uniform} applies, and Algorithm \ref{alg:sobol} samples permutations from a uniform distribution in an unbiased way. Figure \ref{fig:sobol_histogram} shows the distribution of 1000 permutations sampled with $d=4$, which is clearly uniform.

\begin{algorithm}[t]
\small
\SetKwFunction{FPolarToCartesian}{PolarToCartesian}
\DontPrintSemicolon
 \SetKwProg{Fn}{Function}{:}{}
  \Fn{\FPolarToCartesian{$(r, \varphi_1,\varphi_2,\cdots,\varphi_{d-2})$}}{ 
         \KwOutput{$\vec{x}$}
 
 \For{$i\leftarrow 1$ \KwTo $d-1$}{
    $x_i \leftarrow r$\;
    \For{$j\leftarrow 1$ \KwTo $i-1$}{
        $x_i \leftarrow x_i \sin{\varphi_j}$\;
    }
    \If{$i<d-2$}{
        $x_i \leftarrow x_i \cos{\varphi_i}$\;
    }
   }
 \KwRet $x$\;
  }
  \;
  
\SetKwFunction{FSobolPermutations}{SobolPermutations}
\DontPrintSemicolon
 \SetKwProg{Fn}{Function}{:}{}
  \Fn{\FSobolPermutations{$n,d$}}{ 
         \KwOutput{$\Pi$}
 
 \For{$i\leftarrow 1$ \KwTo $n$}{
    $x \leftarrow \text{SobolPoint}(i,n,d)$\tcp*{$x$ has $d-2$ elements}
    $\varphi \leftarrow \vec{0}$\;
    \For{$j\leftarrow 1$ \KwTo $d-2$}{
          $\varphi_j \leftarrow F_j^{-1}(x_j)$\tcp*{Inverse CDF transformation}
    }
    $y \leftarrow $\FPolarToCartesian{$1,\varphi$}\tcp*{$y$ has $d-1$ elements}
    $z \leftarrow \hat{U}^Ty$\tcp*{$z$ has $d$ elements}
    $\Pi_i \leftarrow \text{argsort}(z)$
   }
 \KwRet $\Pi$\;
  }
  \;
 \caption{Sobol Permutations}
\label{alg:sobol}
\end{algorithm}

\begin{figure}
\centering
\begin{minipage}{.48\textwidth}
  \centering
  \includegraphics[width=1.0\textwidth]{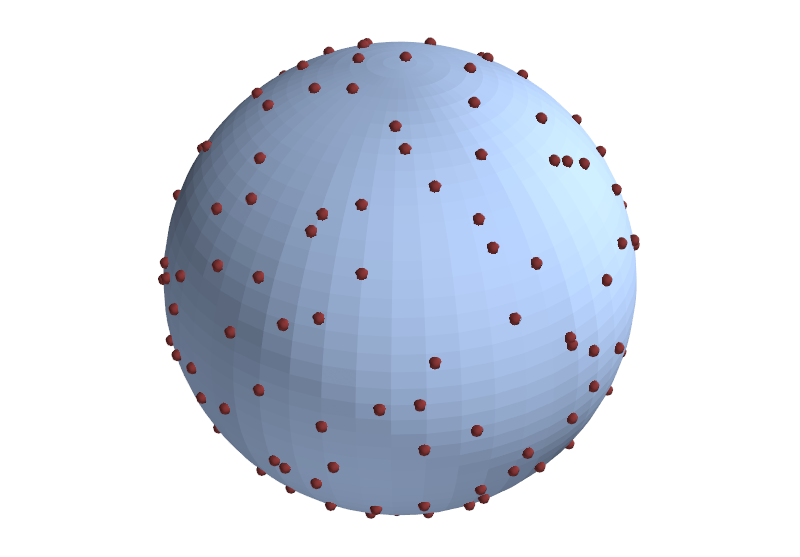}
  \captionof{figure}{Sobol sphere}
  \label{fig:sobol_sphere}
\end{minipage}%
\begin{minipage}{.48\textwidth}
  \centering
  \includegraphics[width=1.0\linewidth]{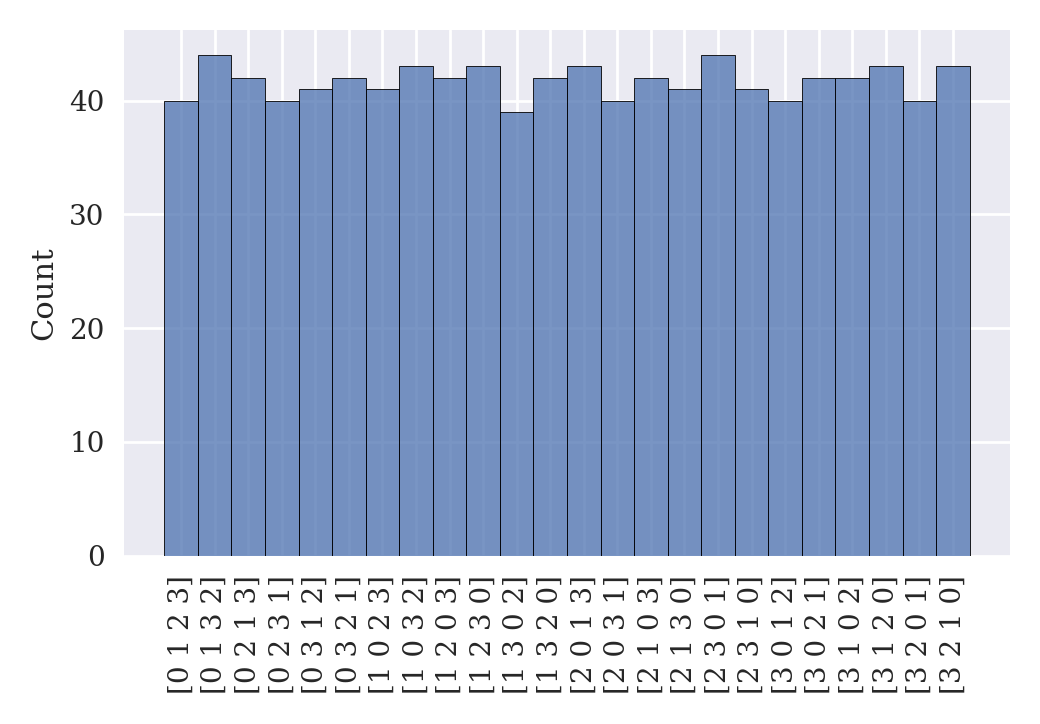}
  \captionof{figure}{Sobol permutations}
  \label{fig:sobol_histogram}
\end{minipage}
\end{figure}

In Section \ref{sec:kernel_methods}, we proposed sampling methods for the Shapley value approximation problem based on directly optimising discrepancy for the symmetric group. While these methods have some more explicit guarantees in terms of quadrature error they also suffer from expensive optimisation processes. The methods discussed in this section, based on the hypersphere, have the advantage of being linear-time in the number of samples $n$. Table \ref{tab:complexity} summarises the complexity of the proposed algorithms. In the next section, we evaluate these algorithms in terms of quadrature error and runtime.

\begin{table}[t]
\small
\caption{Complexity in $n$}
\label{tab:complexity}
%\vskip 0.15in
\centering
%\scriptsize
\begin{sc}
\begin{tabular}{ll}
\toprule
    Algorithm &   Complexity \\
\midrule
      Herding &  $O(n^2)$ \\
          SBQ &  $O(n^3)$ \\
   Orthogonal &  $O(n)$ \\
    Sobol &  $O(n)$ \\

\bottomrule
\end{tabular}
\end{sc}
\end{table}

\section{Evaluation}
\label{sec:evaluation}
We evaluate the performance of permutation sampling strategies on tabular data, image data, and in terms of data-independent discrepancy scores. Table \ref{tab:datasets} describes a set of six tabular datasets. These datasets are chosen to provide a mixture of classification and regression problems, with varying dimensionality, and a mixture of problem domains. For this analysis, we avoid high-dimensional problems, such as natural language processing, due to the difficulty of solving for and interpreting Shapley values in these cases. For the image evaluation we use samples from the ImageNet 2012 dataset of \cite{russakovsky2015imagenet}, grouping pixels into tiles to reduce the dimensionality of the problem to 256. 

\begin{table}[t]
\small
\caption{Tabular datasets}
\label{tab:datasets}
%\vskip 0.15in
\centering
%\scriptsize
\begin{sc}
\begin{tabular}{lrrrrr}
\toprule
               name &  rows &   cols &  task  & ref \\
\midrule
         adult &     48842 &       107 & class&         \cite{adult} \\
         breast\_cancer &    699 &     30 &  class              &\cite{mangasarian1990cancer} \\
         bank & 45211 & 16 & class & \cite{moro2014data}\\
   cal\_housing &     20640 &       8 &     regr & \cite{pace1997sparse} \\
       make\_regression &     1000 &      10 &  regr & \cite{scikit-learn} \\
       year &     515345 &      90 &  regr & \cite{Bertin-Mahieux2011} \\
\bottomrule
\end{tabular}
\end{sc}
\end{table}

Experiments make use of a parameterised Mallows kernel for the kernel herding and SBQ algorithms, as well as the discrepancy scores reported in Section \ref{sec:discrepancy_eval}. To limit the number of experiments, we fix the $\lambda$ parameter for the Mallows kernel at $\lambda=4$ and use 25 samples to approximate the argmax for the kernel herding and SBQ algorithms. These parameters are chosen to give reasonable performance in many different scenarios. Experiments showing the impact of these parameters and justification of this choice can be found in Appendix \ref{app:lambda}. 

To examine different types of machine learning models, we include experiments for gradient boosted decision trees (GBDT), a multilayer perceptron with a single hidden layer, and a deep convolutional neural network. All of these models are capable of representing non-linear relationships between features. We avoid simple models containing only linear relationships because their Shapley value solutions are trivial and can be obtained exactly in a single permutation sample. For the GBDT models, we are able to compute exact Shapley values as a reference, and for the other algorithms we use unbiased estimates of the Shapley values by averaging over many trials. More details are given in the respective subsections.

The sampling algorithms under investigation are listed in Table \ref{tab:sampling_algorithms}. The Monte-Carlo, antithetic Monte-Carlo, stratified sampling, Owen sampling, and Owen-halved methods have been proposed in existing literature for the Shapley value approximation problem. The kernel herding, SBQ, Orthogonal and Sobol methods are the newly proposed methods and form the main line of enquiry in this work.

\begin{table}
\small
\caption{Permutation sampling algorithms under evaluation}
\label{tab:sampling_algorithms}
\centering
\begin{tabular}{lll}
\toprule
Sampling algorithm             & Already proposed for Shapley values & Description and references \\
\midrule
Monte-Carlo                    & Yes                                      & Section \ref{sec:monte_carlo}                   \\
Monte-Carlo Antithetic         & Yes                                      & Section \ref{sec:antithetic}                   \\
Owen                           & Yes                                      & Section \ref{sec:owen}                   \\
Owen-Halved                    & Yes                                      & Section \ref{sec:owen}                   \\
Stratified                     & Yes                                      & Section \ref{sec:stratified}                    \\
Kernel herding                 & No                                       & Section \ref{sec:kernel_herding}                    \\
SBQ& No                                       & Section \ref{sec:sbq}                    \\
Orthogonal Spherical Codes     & No                                       & Section \ref{sec:orthogonal}                    \\
Sobol Sequences                & No                                       & Section \ref{sec:sobol}     \\
\bottomrule
\end{tabular}
\end{table}

The experimental evaluation proceeds as follows:
\begin{itemize}
  \item Section \ref{sec:existing_tabular} first evaluates existing algorithms on tabular data using GBDT models, reporting exact error scores. MC-Antithetic emerges as the clear winner, so we use this as a baseline in subsequent experiments against newly proposed algorithms.
  \item Section \ref{sec:proposed_tabular} examines Shapley values for newly proposed sampling algorithms as well as MC-Antithetic using GBDT models trained on tabular data, and reports exact error scores.
  \item Section \ref{sec:proposed_mlp} examines Shapley values for newly proposed sampling algorithms as well as MC-Antithetic using multilayer perceptron models trained on tabular data, and reports error estimates.
  \item Section \ref{sec:proposed_discrepancy} reports data-independent discrepancy and execution time for newly proposed sampling algorithms and MC-Antithetic.
  \item Section \ref{sec:proposed_images} evaluates Shapley values for newly proposed sampling algorithms and MC-Antithetic using a deep convolutional neural network trained on image data, reporting error estimates.
\end{itemize}

\subsection{Existing algorithms - Tabular data and GBDT models}
\label{sec:existing_tabular}
We train GBDT models on the tabular datasets listed in Table \ref{tab:datasets} using the XGBoost library of \cite{ChenG16}. Models are trained using the entire dataset (no test/train split) using the default parameters of the XGBoost library (100 boosting iterations, maximum depth 6, learning rate 0.3, mean squared error objective for regression, and binary logistic objective for classification). The exact Shapley values are computed for reference using the TreeShap Algorithm (Algorithm 3) of \cite{lundberg2018consistent}, a polynomial-time algorithm specific to decision tree models.

Recall from Section \ref{sec:shapley_values}, to define Shapley values for a machine learning model, features not present in the active subset must be marginalised out. To compare our results to the exact Shapley values, we use the same method as \cite{lundberg2018consistent}. A small fixed set of `background instances' is chosen for each dataset. These form a distribution with which to marginalise out the effect of features. To calculate Shapley values for a given row (a `foreground' instance), features not part of the active subset are replaced with values from a background instance. The characteristic function evaluation $v(S)$ is then the mean of a set of model predictions, where each time, the foreground instance has features not in subset $S$ replaced by a different background instance. For details, see \cite{lundberg2018consistent} or the SHAP software package. For classification models, we examine the log-odds output, as the polynomial-time exact Shapley Value algorithm only works when model outputs are additive, and because additive model outputs are consistent with the efficiency property of Shapley values. 

For each dataset/algorithm combination, Shapley values are evaluated for all features of 10 randomly chosen instances, using a fixed background dataset of 100 instances to marginalise out features. Shapley values are expensive to compute, and are typically evaluated for a small number of test instances, not the entire dataset. The choice of 10 instances is a balance between computation time and representing the variation of Shapley values across the dataset. The approximate Shapley values for the 10 instances form a $10\times d$ matrix, from which we calculate the elementwise mean squared error against the reference Shapley values. For $10\times d$ matrix $Z$, the MSE for our approximation $\hat{Z}$ is defined as
\begin{equation}
\label{eq:mse}
 \text{MSE}(Z,\hat{Z})=\frac{1}{10d}\sum^{10}_i \sum^d_j (Z_{i,j}-\hat{Z}_{i,j})^2. 
\end{equation}
As the sampling algorithms are all randomised, we repeat the experiment 25 times (on the same foreground and background instances) to generate confidence intervals. 

The results are shown in Figure \ref{fig:incumbent_eval}. Algorithms are evaluated according to number of evaluations of $v(S \cup {i}) - v(S)$, written as `marginal\_evals' on the x-axis of figures. If the algorithm samples permutations, the number of marginal evaluations is proportional to $nd$, where $n$ is the number of permutations sampled. The stratified sampling method is missing for the adult and year datasets because it requires at least $2d^2$ samples, which becomes intractable for the higher-dimensional datasets. The shaded areas show a 95\% confidence interval for the mean squared error. Of the existing algorithms, MC-antithetic is the most effective in all experiments. For this reason, in the next sections, we use MC-Antithetic as the baseline when evaluating the kernel herding, SBQ, orthogonal and Sobol methods.

\begin{figure*}[ht]
        \centering
        \begin{subfigure}[b]{0.495\textwidth}
            \centering
            \includegraphics[width=\textwidth]{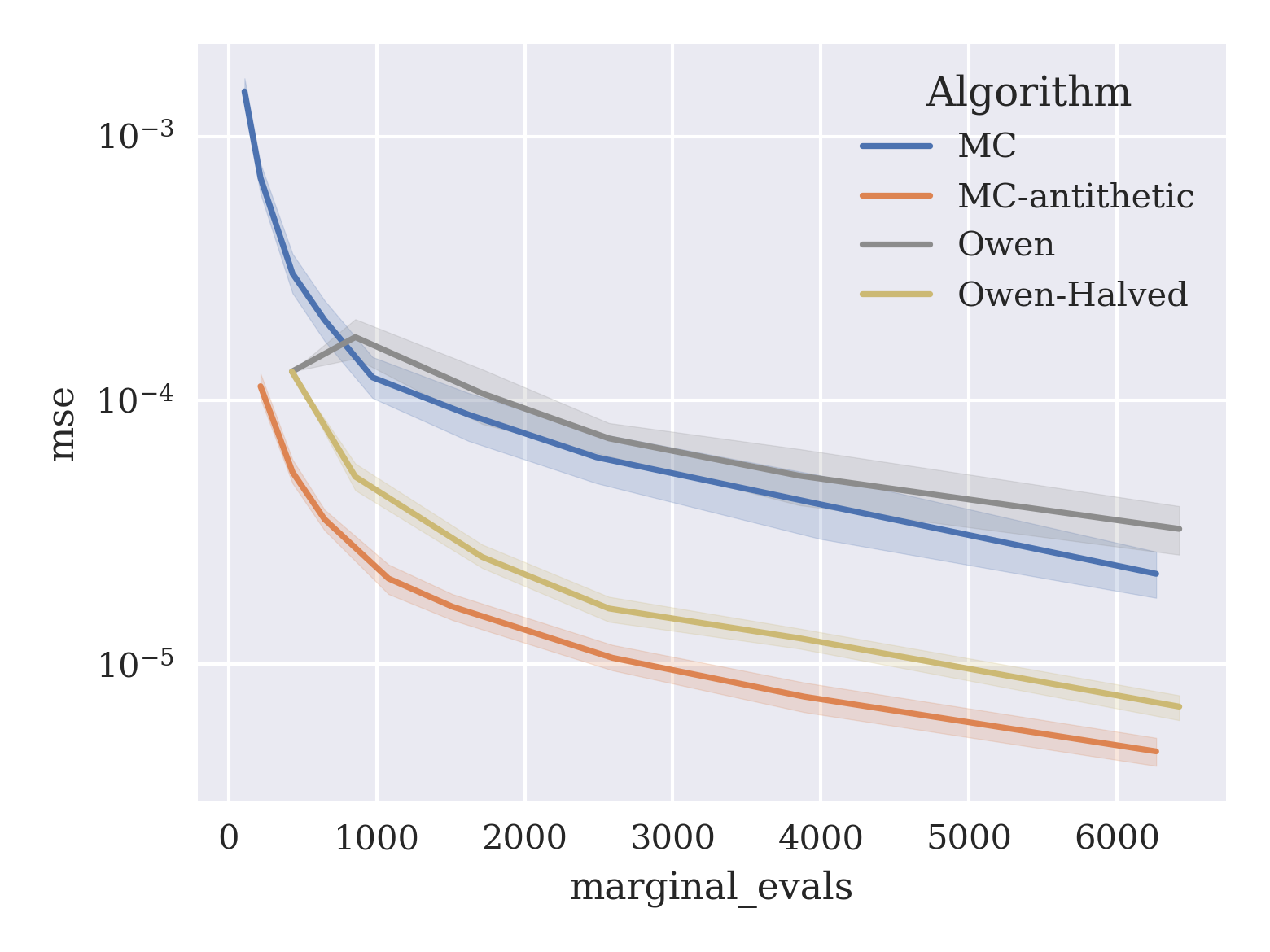}
            \caption[]%
            {{\textit{adult}}}    
        \end{subfigure}
        \hfill
        \begin{subfigure}[b]{0.495\textwidth}  
            \centering 
            \includegraphics[width=\textwidth]{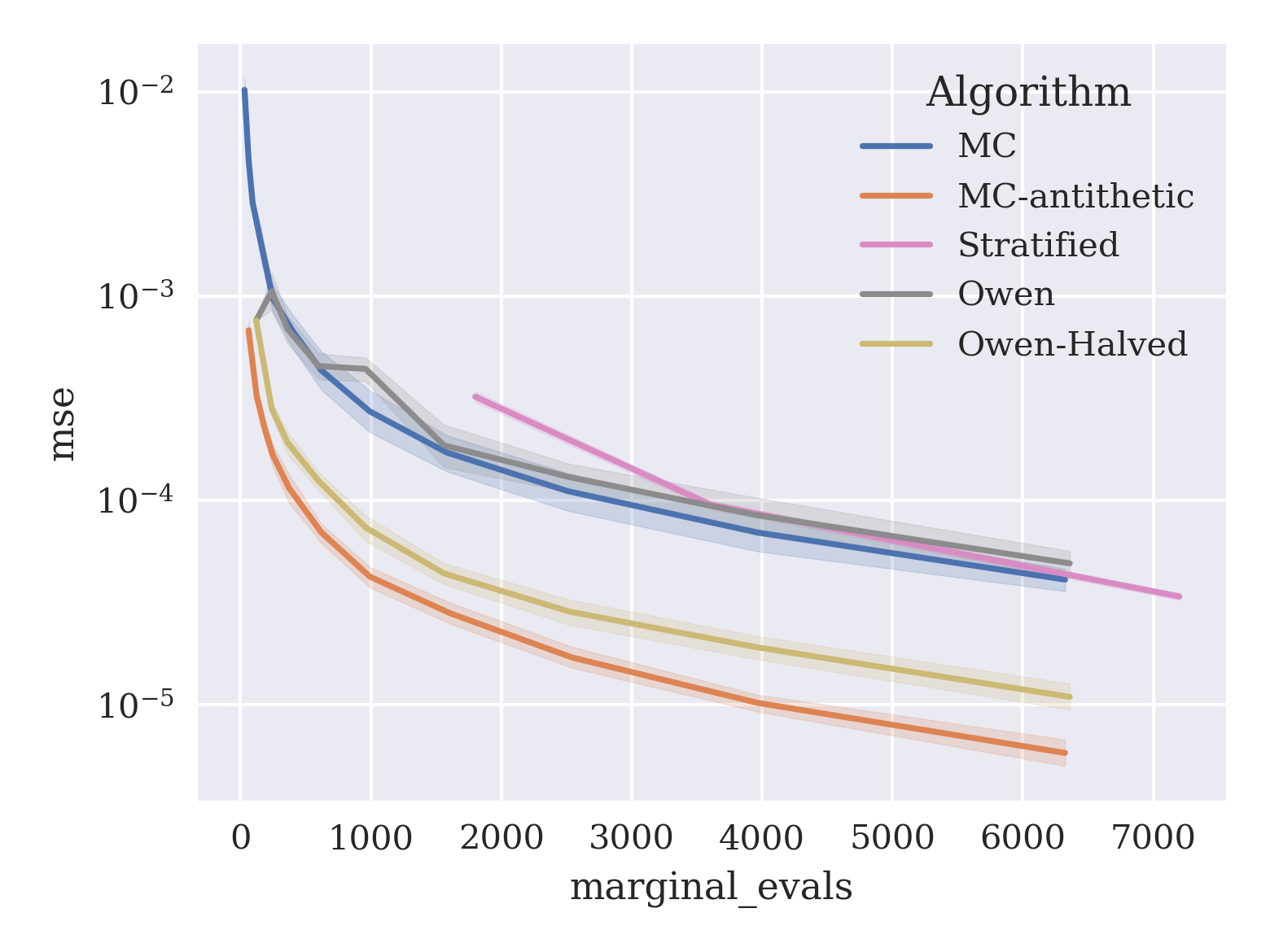}
            \caption[]%
            {{\textit{breast\_cancer}}}    
        \end{subfigure}
%        \vskip\baselineskip
        \begin{subfigure}[b]{0.495\textwidth}   
            \centering 
            \includegraphics[width=\textwidth]{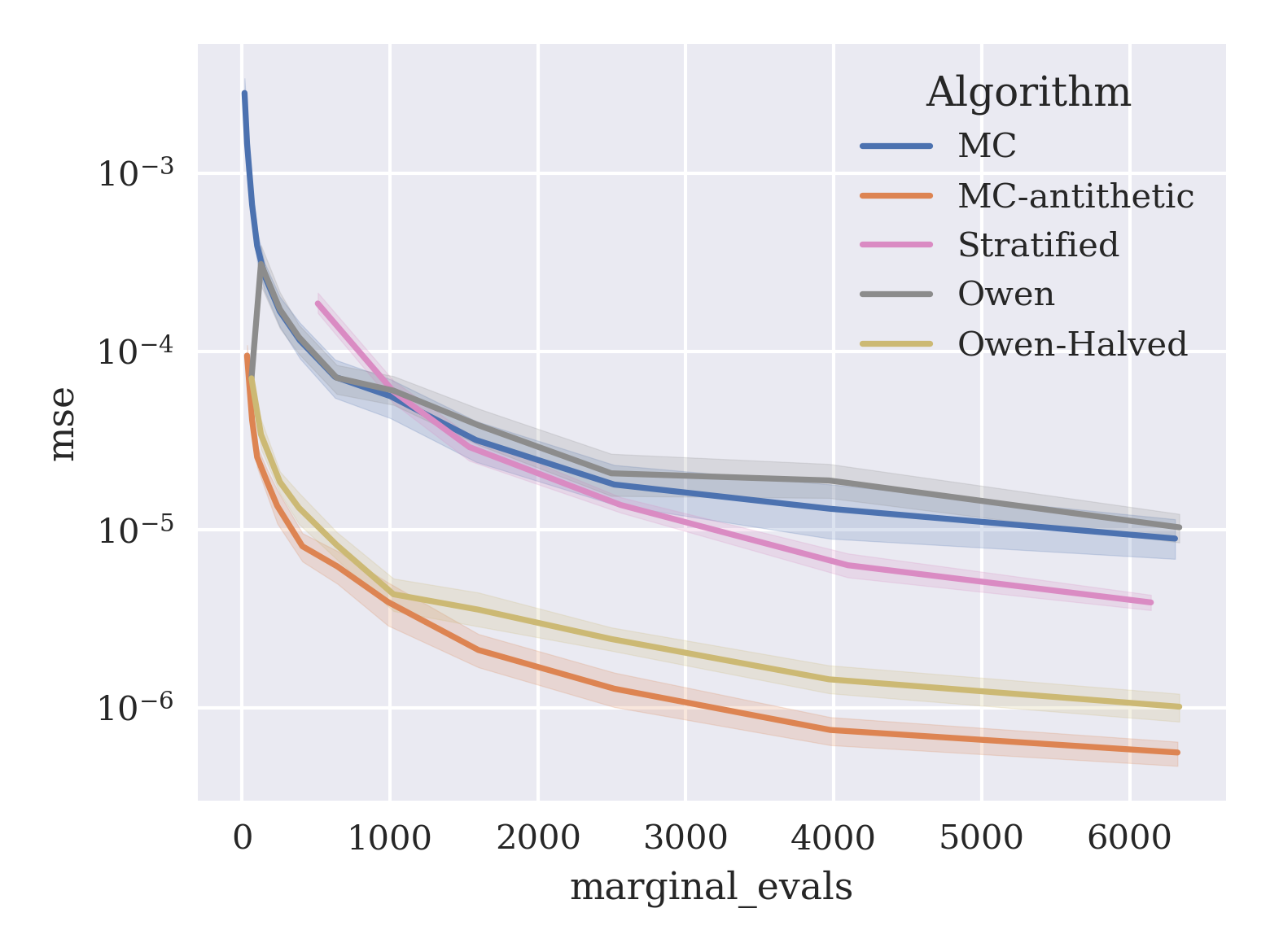}
            \caption[]%
            {{\textit{bank}}}    
        \end{subfigure}
        \hfill
        \begin{subfigure}[b]{0.495\textwidth}   
            \centering 
            \includegraphics[width=\textwidth]{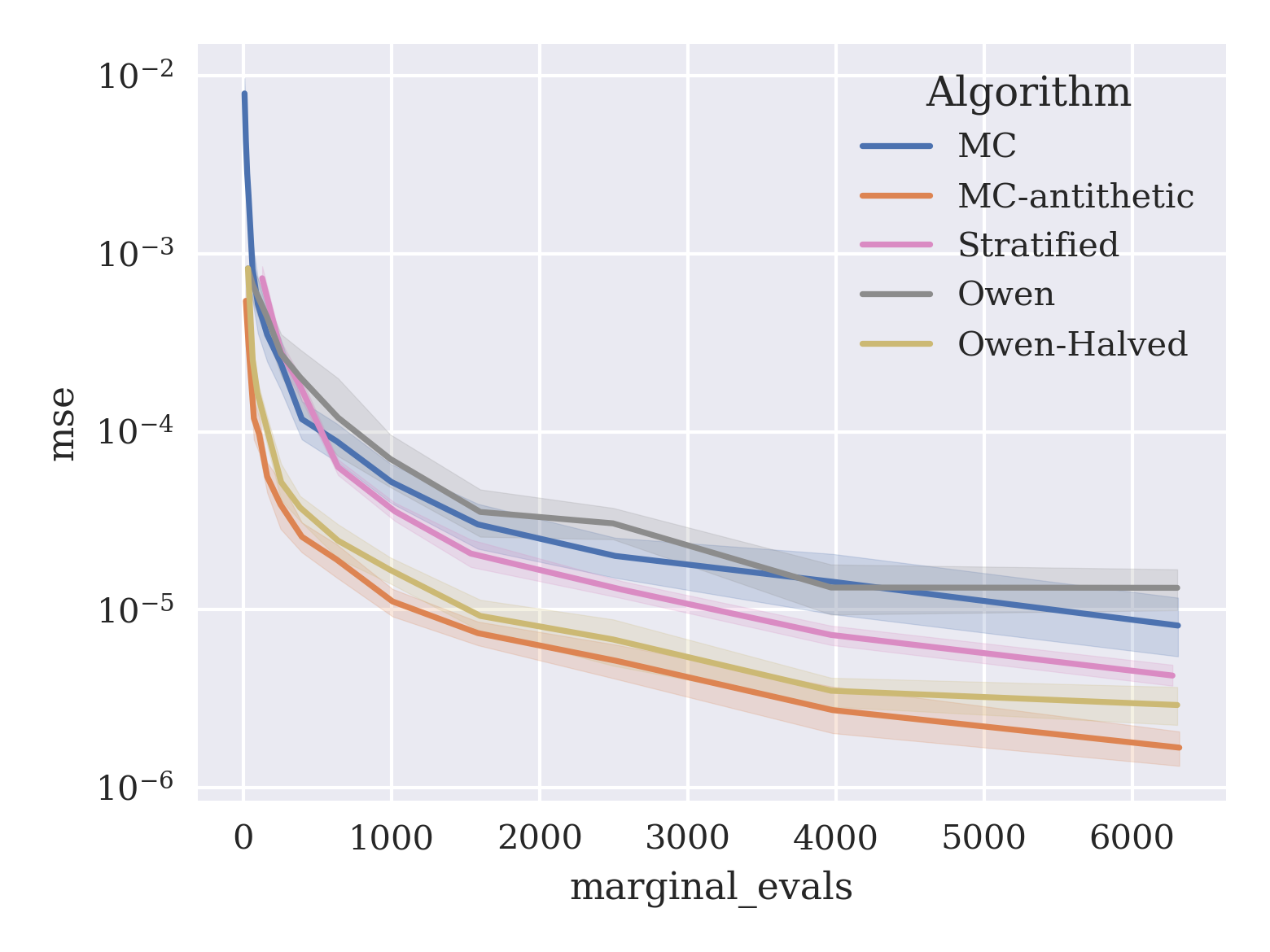}
            \caption[]%
            {{\textit{cal\_housing}}}    
        \end{subfigure}

         \begin{subfigure}[b]{0.495\textwidth}   
            \centering 
            \includegraphics[width=\textwidth]{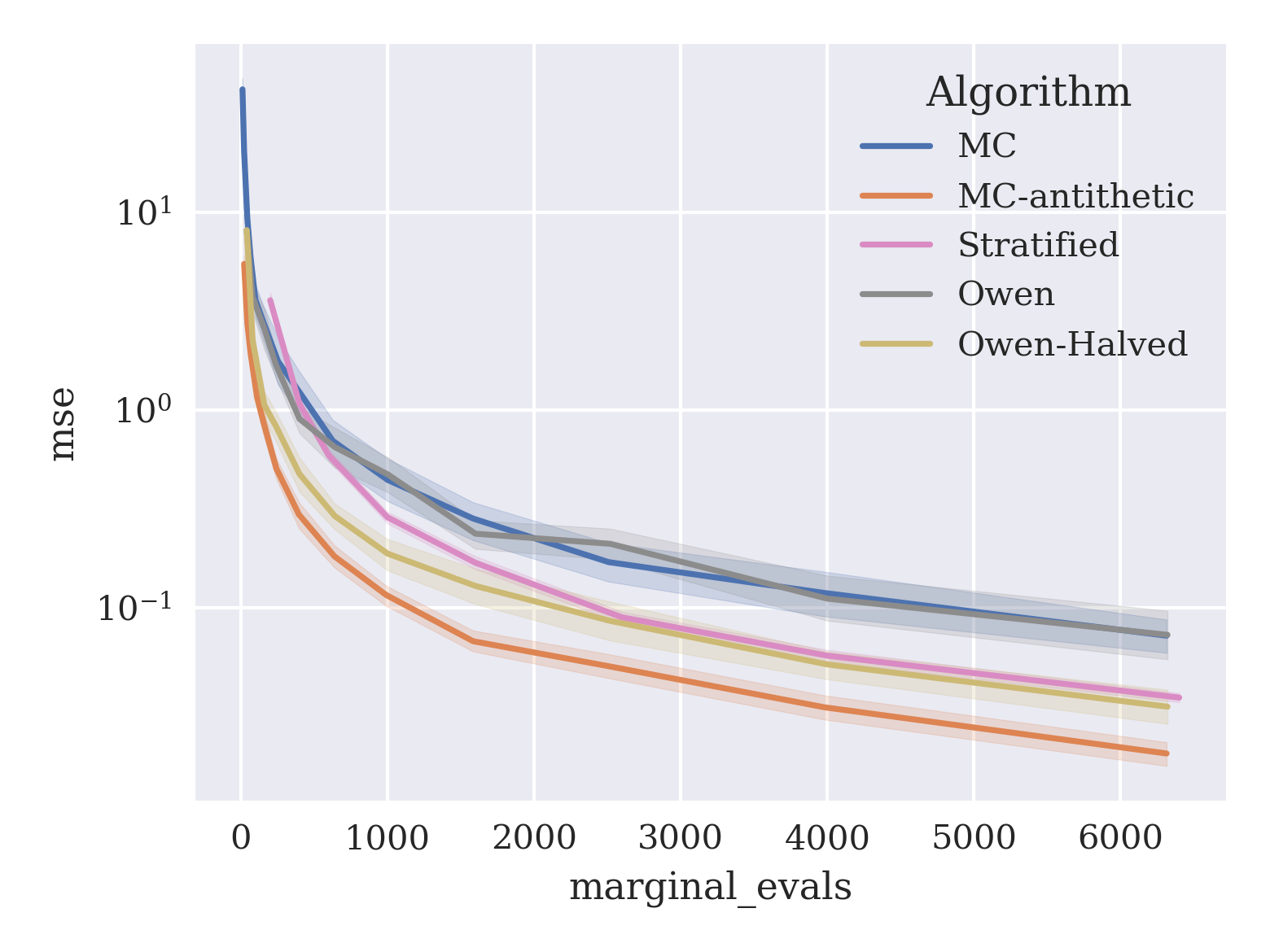}
            \caption[]%
            {{\textit{make\_regression}}}    
        \end{subfigure}
        \hfill
        \begin{subfigure}[b]{0.495\textwidth}   
            \centering 
            \includegraphics[width=\textwidth]{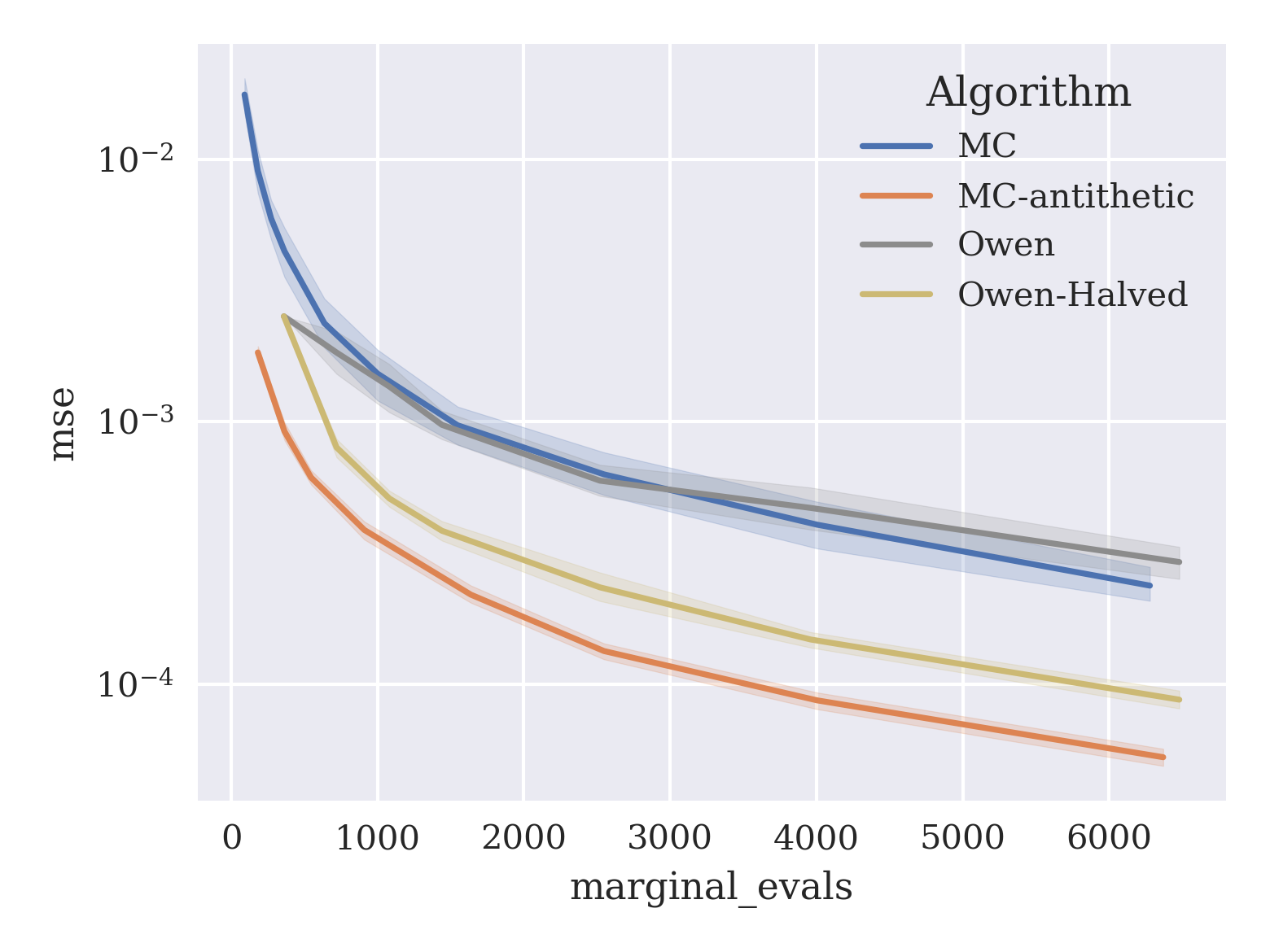}
            \caption[]%
            {{\textit{year}}}    
        \end{subfigure}
        \caption{Existing algorithms - Tabular data, GBDT models} 
        \label{fig:incumbent_eval}
\end{figure*}

\subsection{Proposed algorithms - Tabular data and GBDT models}
\label{sec:proposed_tabular}
Here, we perform experiments using the same methodology in the previous section, examining the mean squared error of the proposed algorithms kernel herding, SBQ, orthogonal and Sobol, against MC-antithetic as the baseline. Figure \ref{fig:new_eval} plots the results. For the lower-dimensional \textit{cal\_housing} and \textit{make\_regression} datasets, we see good performance for the herding and SBQ methods. This good performance does not translate to the higher-dimensional datasets \textit{adult} and \textit{year}, where herding and SBQ are outperformed by the baseline MC-antithetic method. On the problems where herding and SBQ are effective, SBQ outperforms herding in terms of mean squared error, presumably due to its more aggressive optimisation of the discrepancy. The Sobol method is outperformed by the baseline MC-antithetic method in four of six cases. The orthogonal method shows similar performance to MC-antithetic for a small number of samples, but improves over MC-antithetic as the number of samples increases in all six problems. This is because the orthogonal method can be considered an extension of the antithetic sampling scheme --- increasing the number of correlated samples from 2 to $2(d-1)$. The orthogonal method also appears preferable to the Sobol method on this collection of datasets: it loses on two of them (\textit{cal\_housing} and \textit{make\_regression}) but the difference in error is very small on these two datasets.

\begin{figure*}[ht]
        \centering
        \begin{subfigure}[b]{0.495\textwidth}
            \centering
            \includegraphics[width=\textwidth]{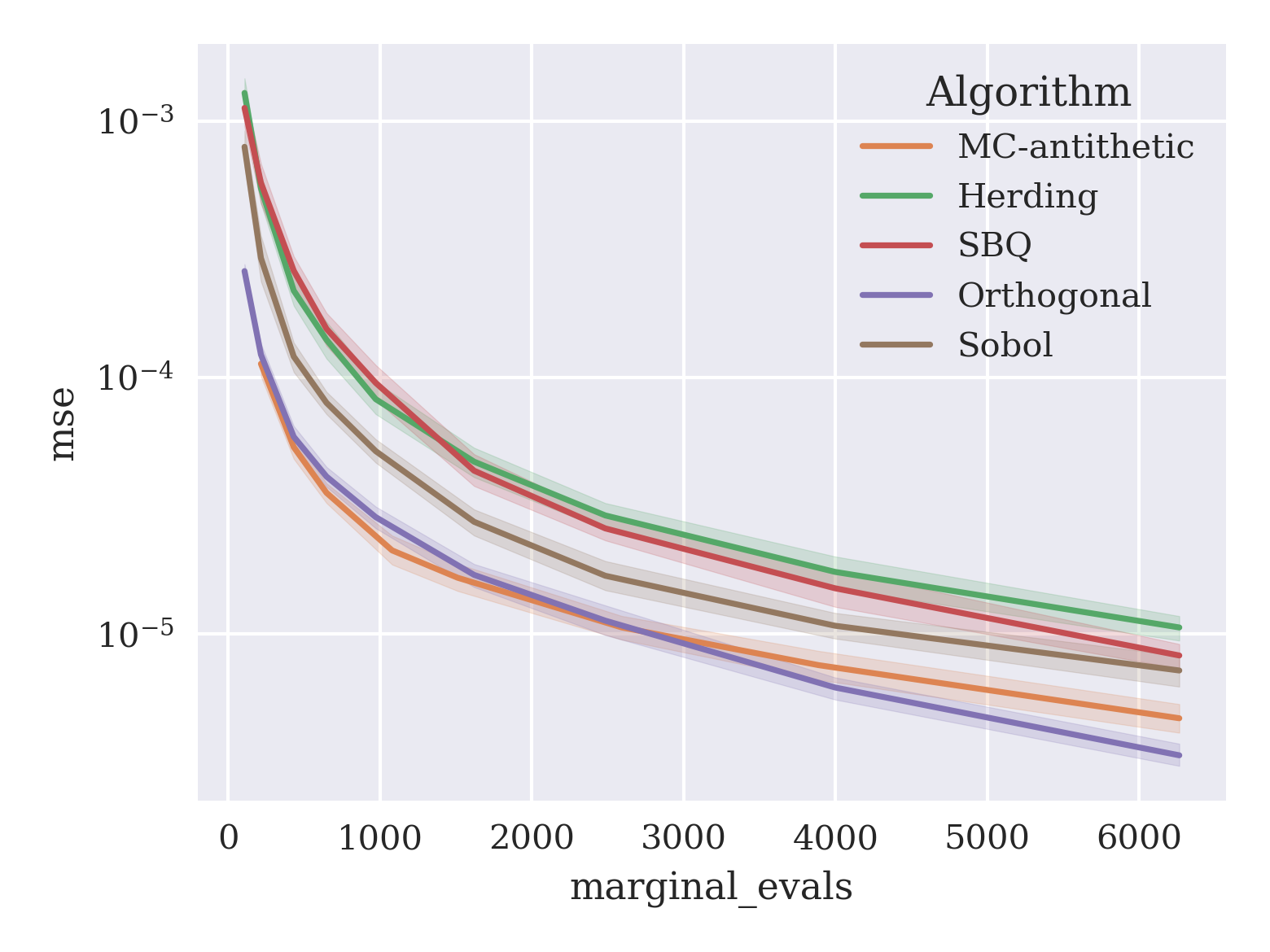}
            \caption[]%
            {{\textit{adult}}}    
        \end{subfigure}
        \hfill
        \begin{subfigure}[b]{0.495\textwidth}  
            \centering 
            \includegraphics[width=\textwidth]{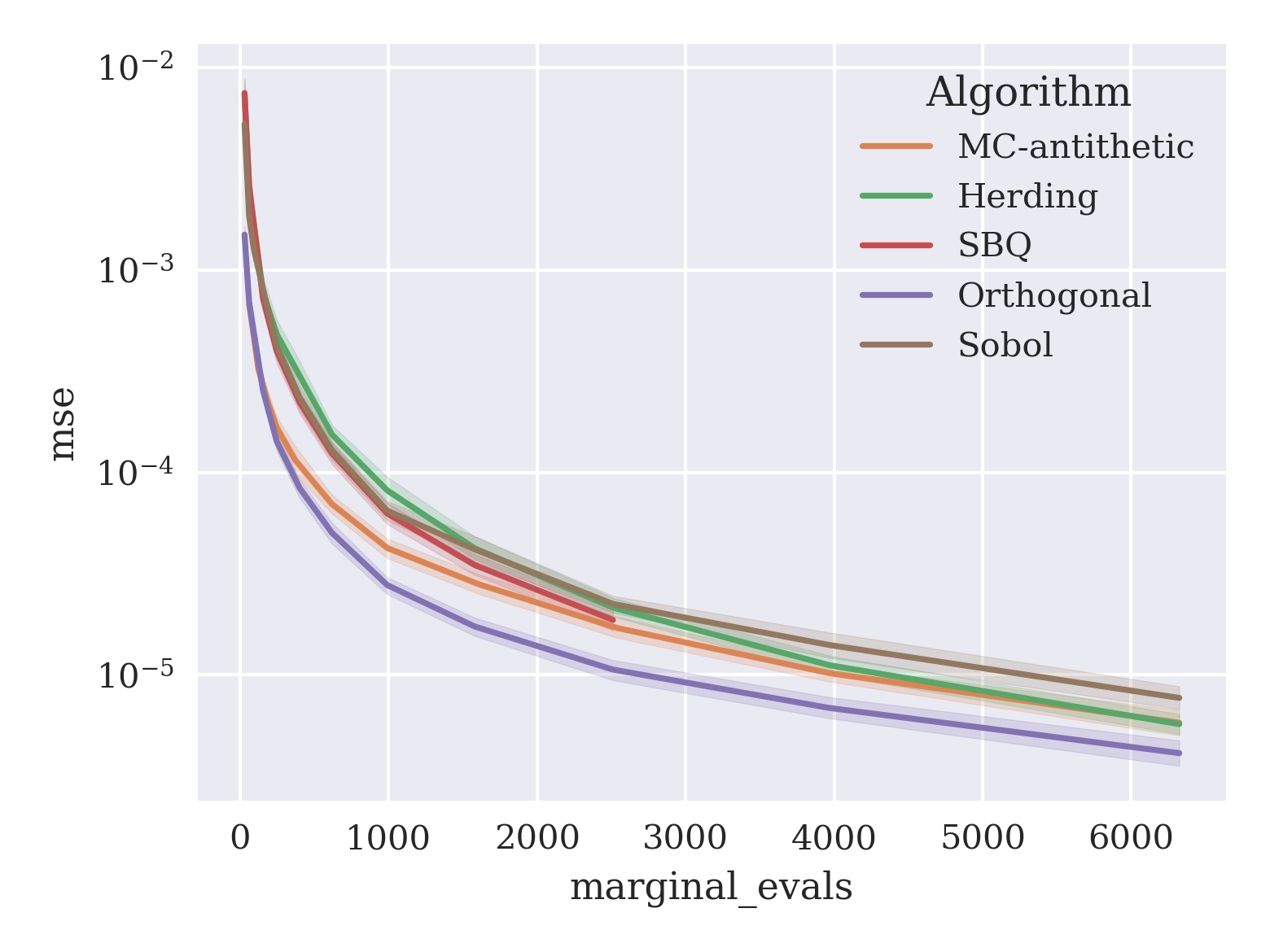}
            \caption[]%
            {{\textit{breast\_cancer}}}    
        \end{subfigure}
%        \vskip\baselineskip
        \begin{subfigure}[b]{0.495\textwidth}   
            \centering 
            \includegraphics[width=\textwidth]{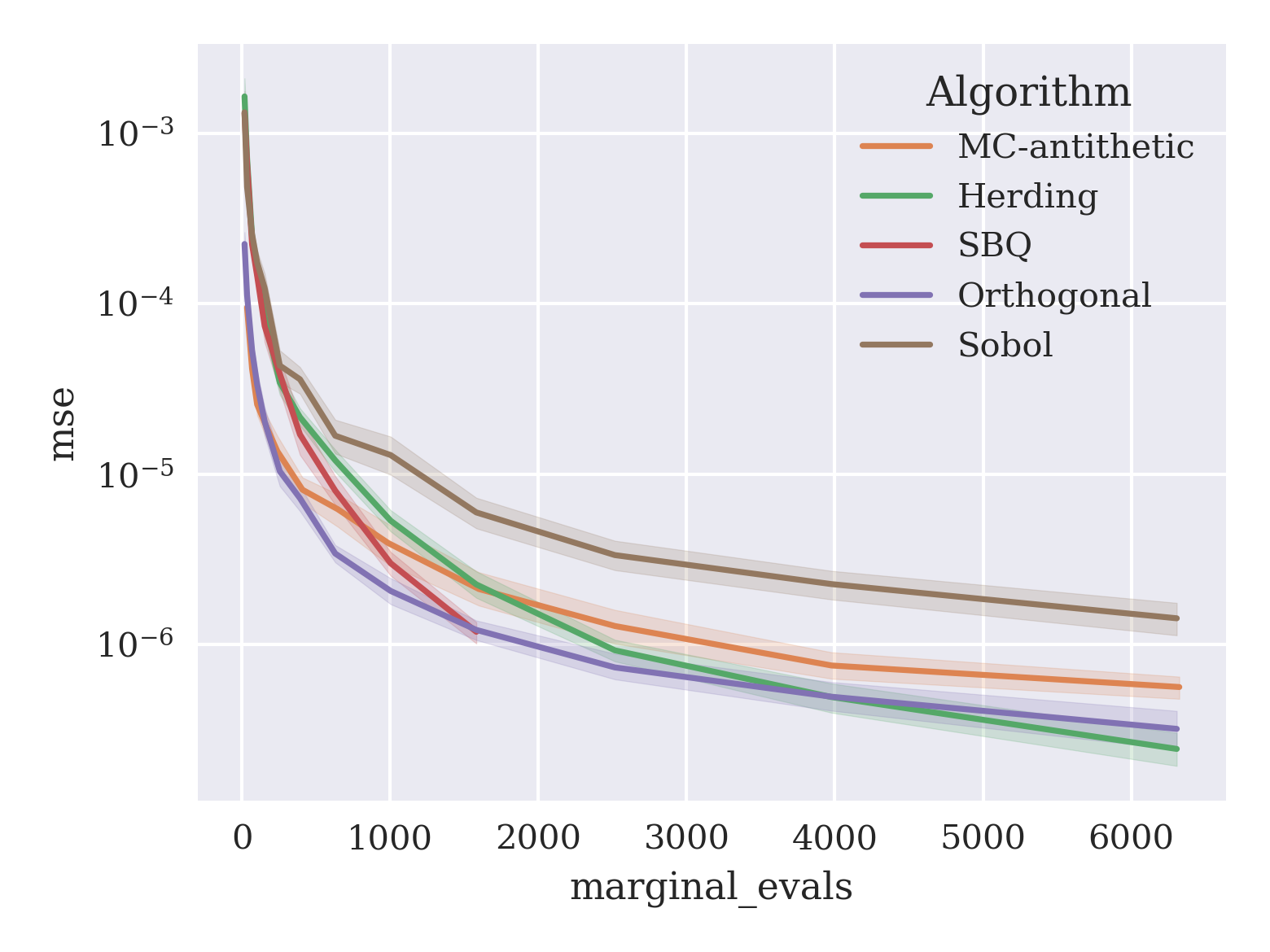}
            \caption[]%
            {{\textit{bank}}}    
        \end{subfigure}
        \hfill
        \begin{subfigure}[b]{0.495\textwidth}   
            \centering 
            \includegraphics[width=\textwidth]{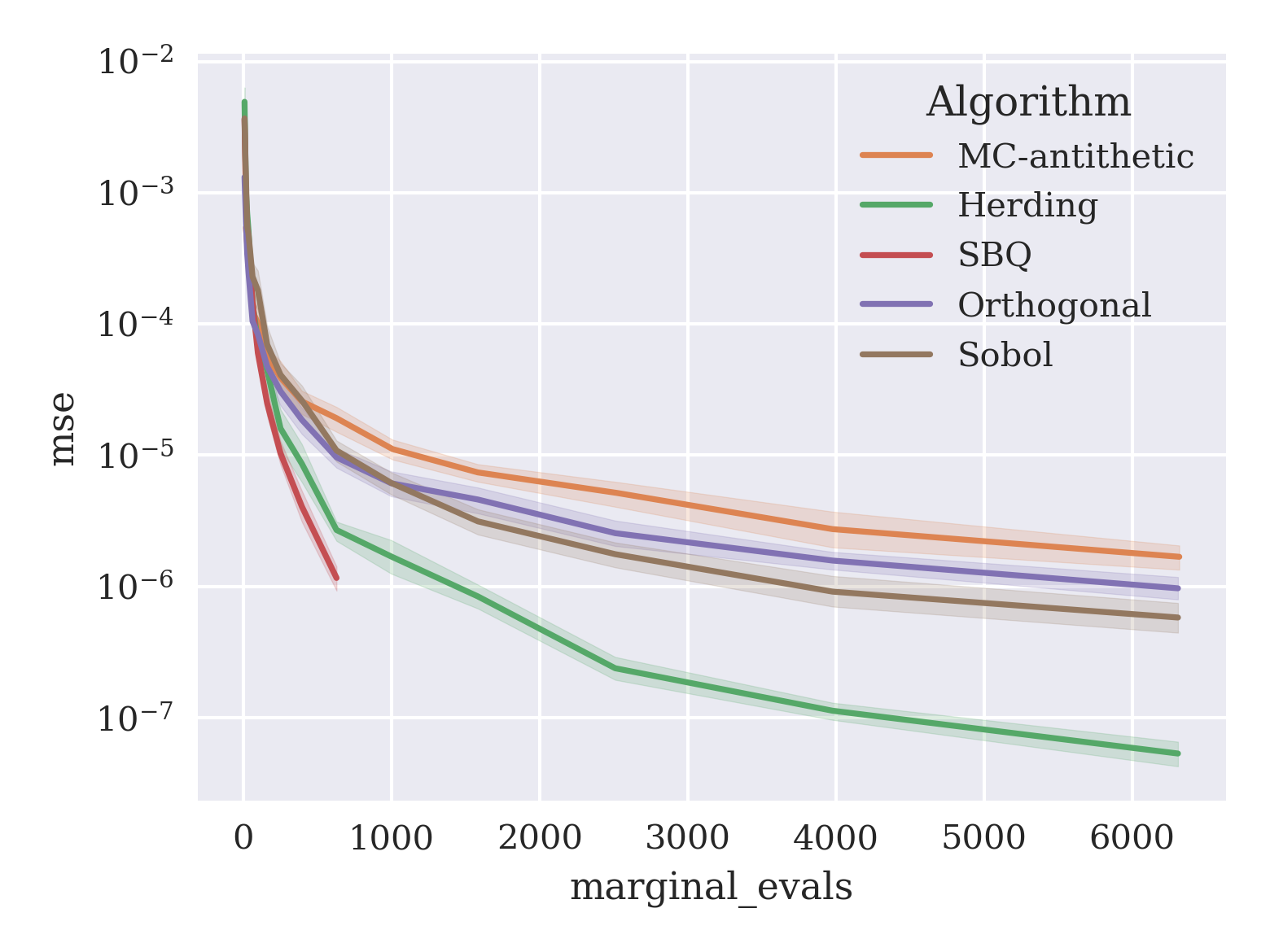}
            \caption[]%
            {{\textit{cal\_housing}}}    
        \end{subfigure}

         \begin{subfigure}[b]{0.495\textwidth}   
            \centering 
            \includegraphics[width=\textwidth]{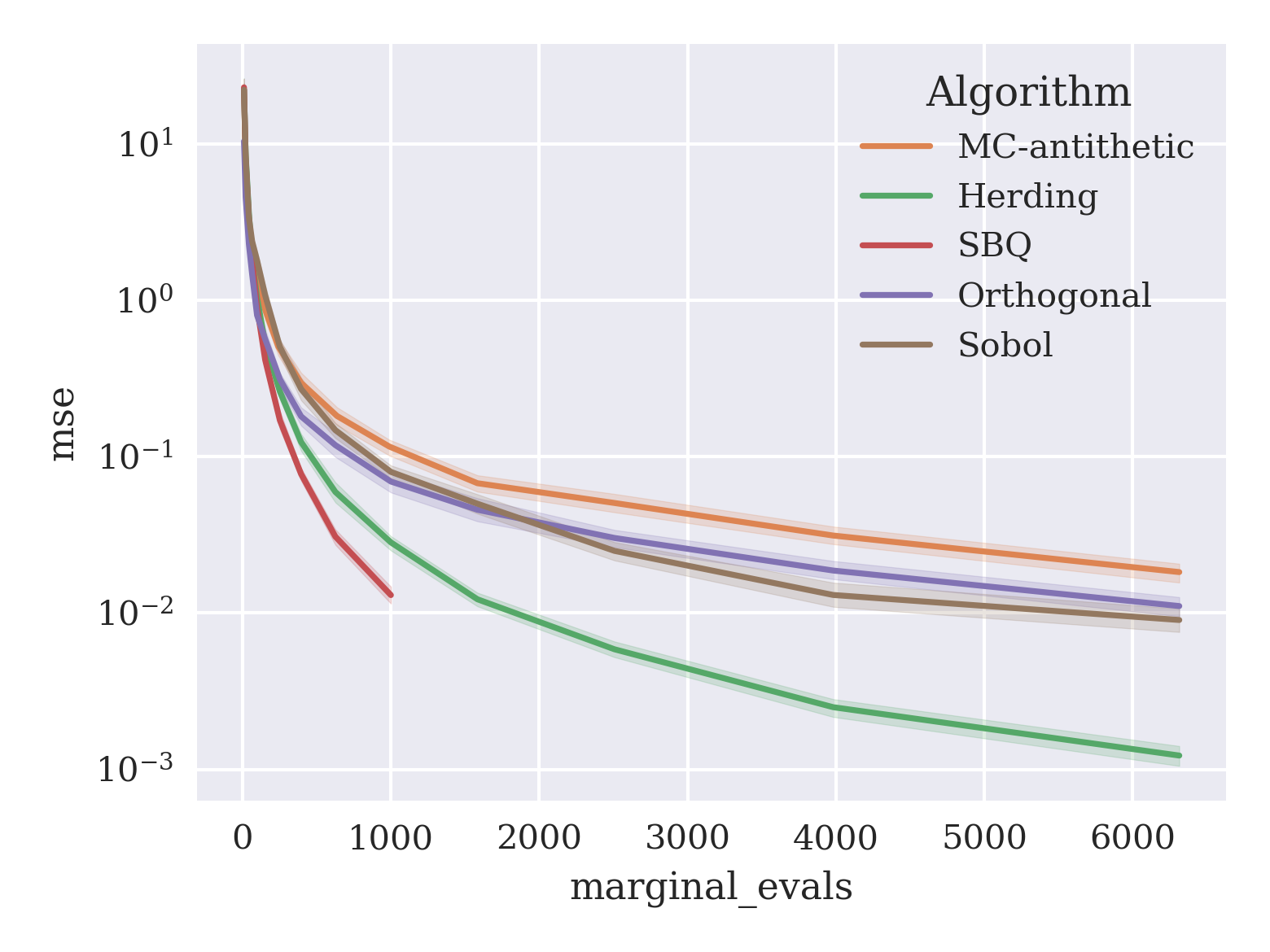}
            \caption[]%
            {{\textit{make\_regression}}}    
        \end{subfigure}
        \hfill
        \begin{subfigure}[b]{0.495\textwidth}   
            \centering 
            \includegraphics[width=\textwidth]{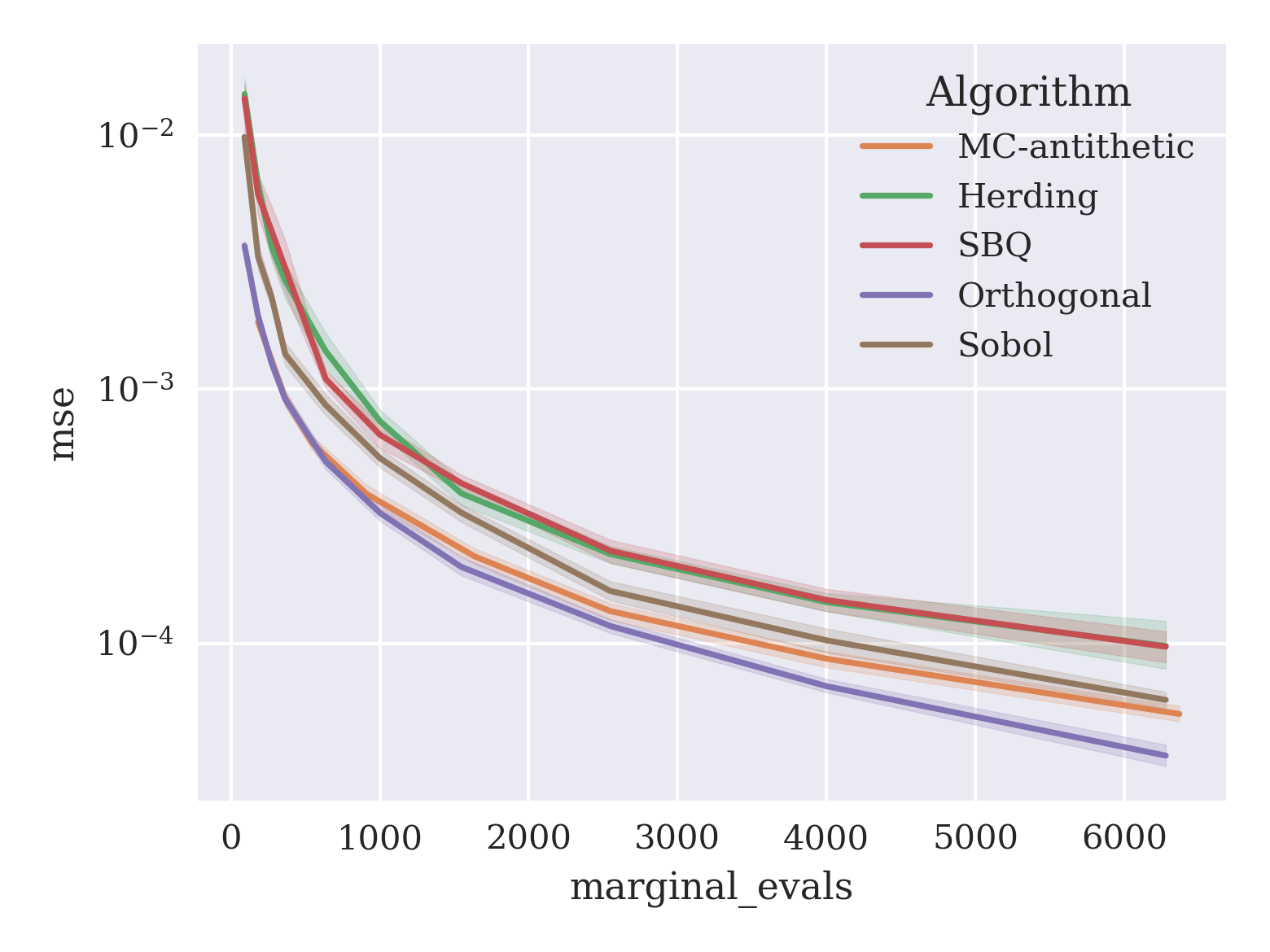}
            \caption[]%
            {{\textit{year}}}    
        \end{subfigure}
        \caption{Proposed algorithms - Tabular data, GBDT models} 
        \label{fig:new_eval}
\end{figure*}

\subsection{Proposed algorithms - Tabular data and MLP models}
\label{sec:proposed_mlp}
Now, we examine error estimates for the proposed algorithms on tabular data using a multi-layer perceptron (MLP) model, presenting the results in Figure \ref{fig:mlp_tabular}. As for the GBDT models, we use the entire dataset for training. The model is trained using the scikit-learn library \citep{scikit-learn} with default parameters: a single hidden layer of 100 neurons, a relu activation function, and trained with the adam optimiser \citep{kingma2014adam} for 200 iterations with an initial learning rate of 0.001. MSE is optimised for regression data, and log-loss for classification data.

For Shapley value computation, features are marginalised out using background features in exactly the same way as for GBDT models. As we do not have access to exact Shapley values, and all sampling algorithms are randomised, we use standard Monte Carlo error estimates based on an unbiased sample estimate. The exact Shapley values $Z$ are substituted with the elementwise mean of the estimates over 25 trials.

For the MLP models, we generally see similar results to the GBDT models: herding and SBQ converging quickly for the lower dimensional \textit{cal\_housing} and \textit{make\_regression} datasets, and the orthogonal method consistently outperforming MC-antithetic across datasets. The orthogonal method also again appears preferable overall to Sobol sampling. For some datasets, such as \textit{adult}, results are more tightly clustered than for the GBDT model. This could indicate fewer higher-order feature interactions in the single layer MLP model, leading to lower variance in the Shapley value characteristic function with respect to the input subsets. In other words, the choice of permutation samples may matter less when strong features interactions are absent.
\begin{figure*}[ht]
        \centering
        \begin{subfigure}[b]{0.495\textwidth}
            \centering
            \includegraphics[width=\textwidth]{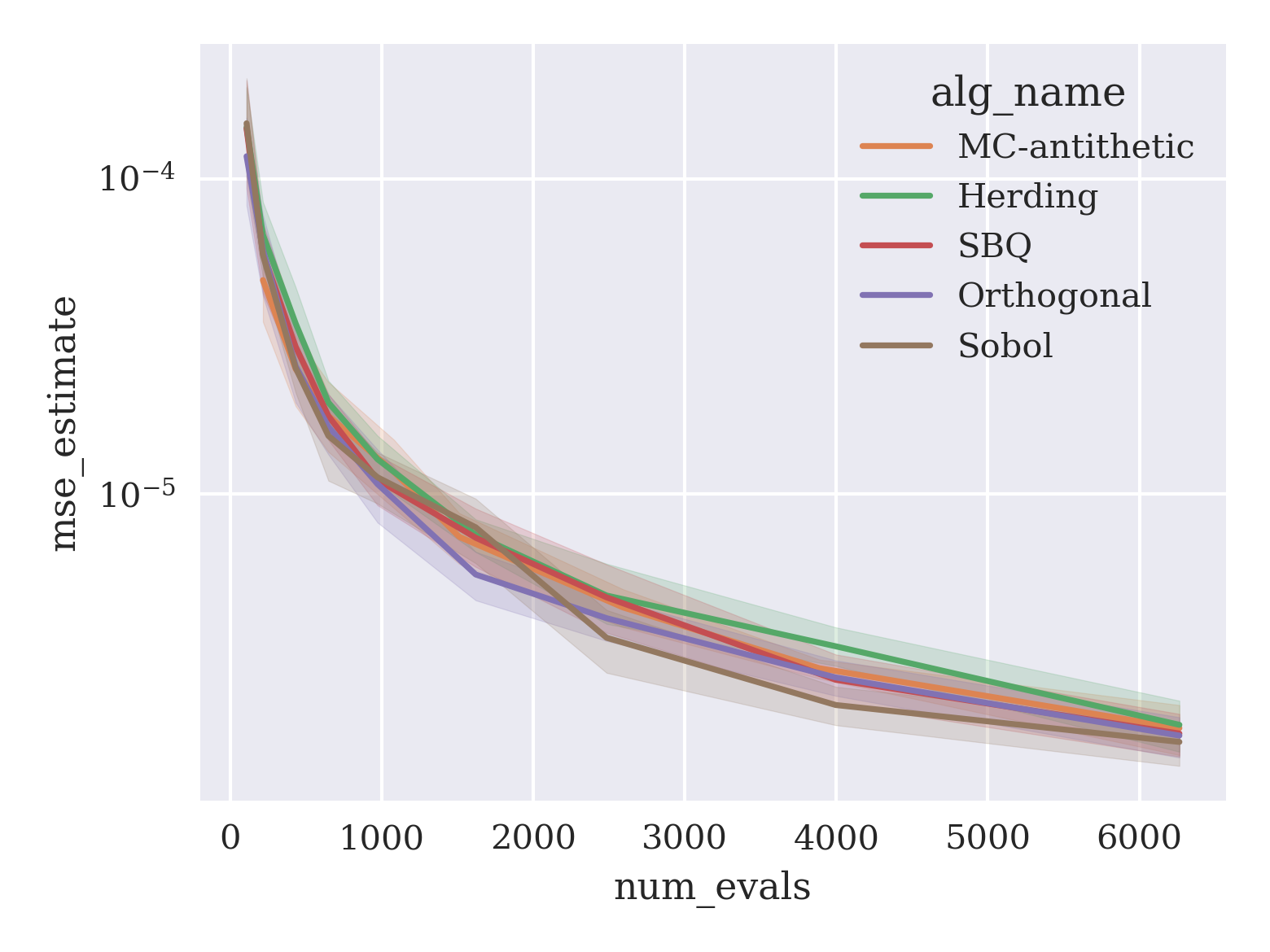}
            \caption[]%
            {{\textit{adult}}}    
        \end{subfigure}
        \hfill
        \begin{subfigure}[b]{0.495\textwidth}  
            \centering 
            \includegraphics[width=\textwidth]{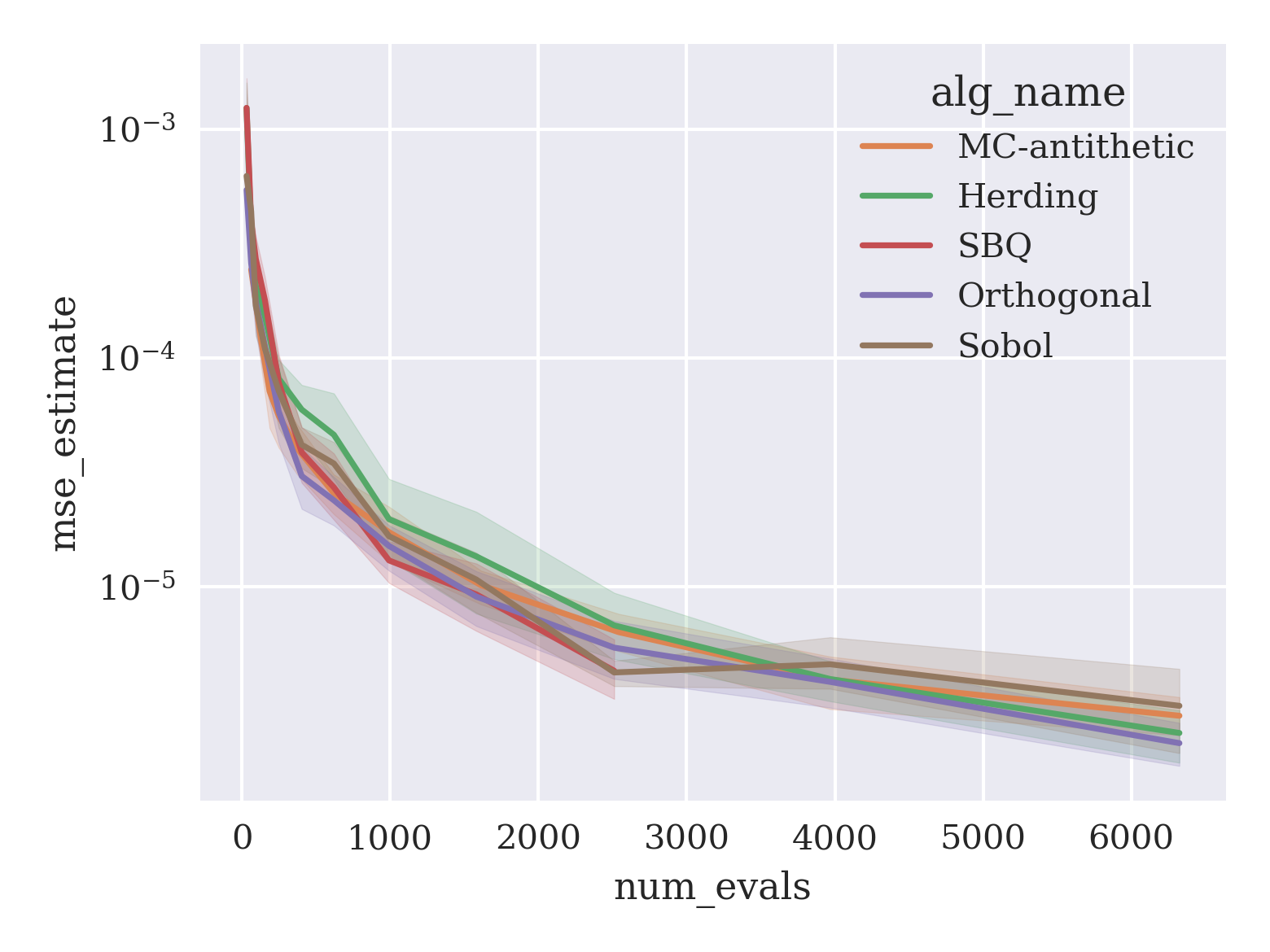}
            \caption[]%
            {{\textit{breast\_cancer}}}    
        \end{subfigure}
%        \vskip\baselineskip
        \begin{subfigure}[b]{0.495\textwidth}   
            \centering 
            \includegraphics[width=\textwidth]{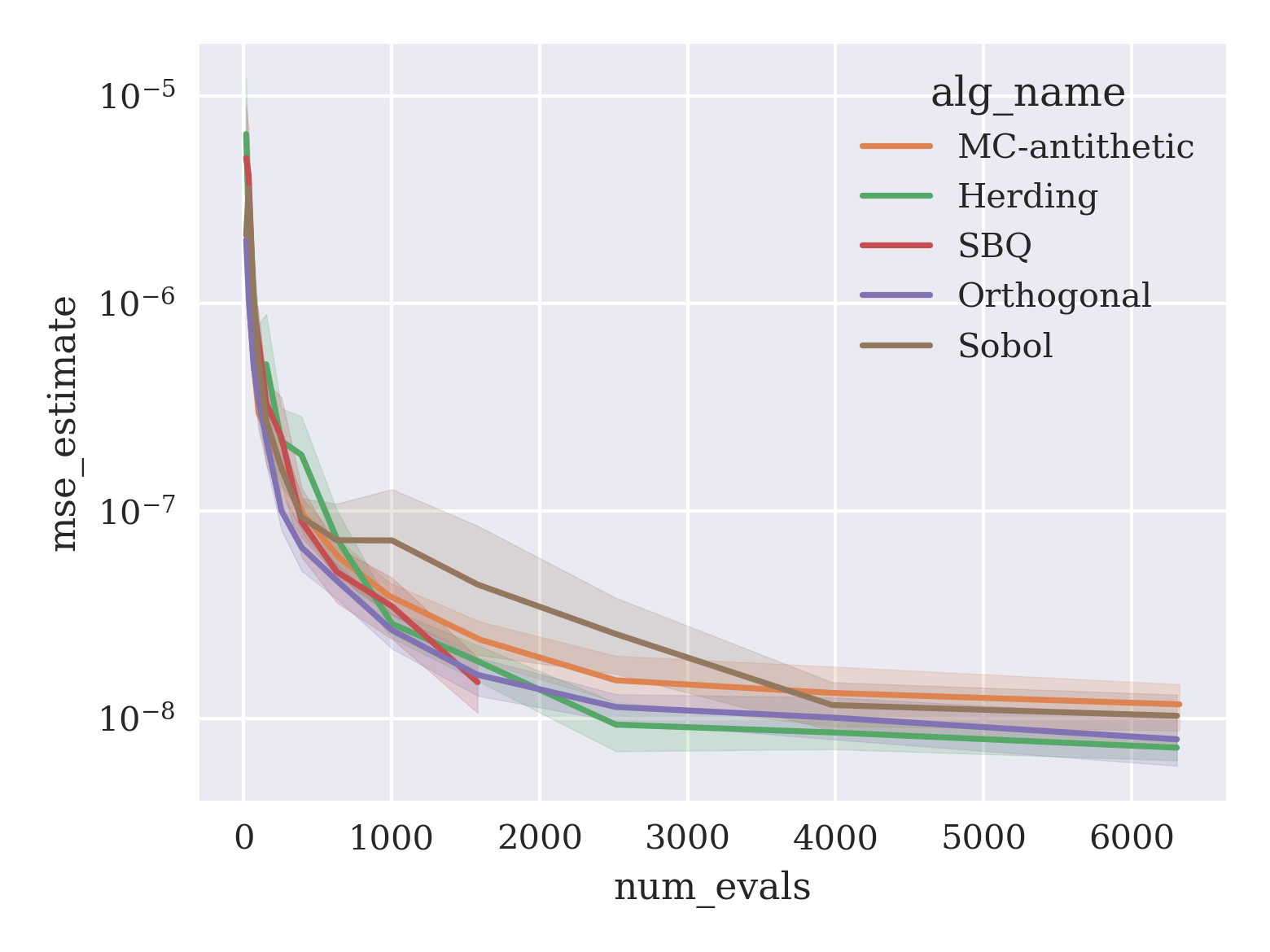}    \caption[]%
            {{\textit{bank}}}    
        \end{subfigure}
        \hfill
        \begin{subfigure}[b]{0.495\textwidth}   
            \centering 
            \includegraphics[width=\textwidth]{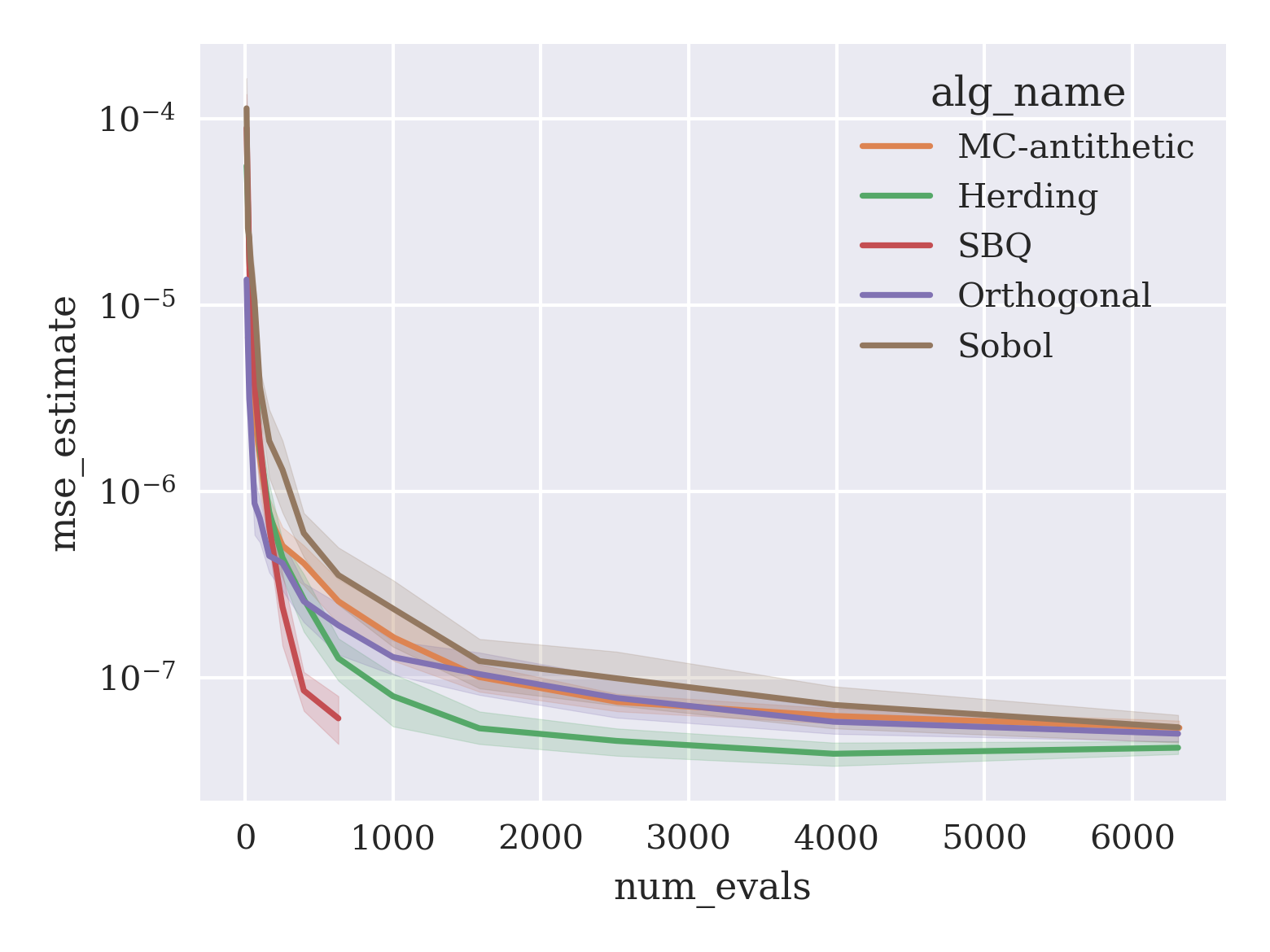}
            \caption[]%
            {{\textit{cal\_housing}}}    
        \end{subfigure}

         \begin{subfigure}[b]{0.495\textwidth}   
            \centering 
            \includegraphics[width=\textwidth]{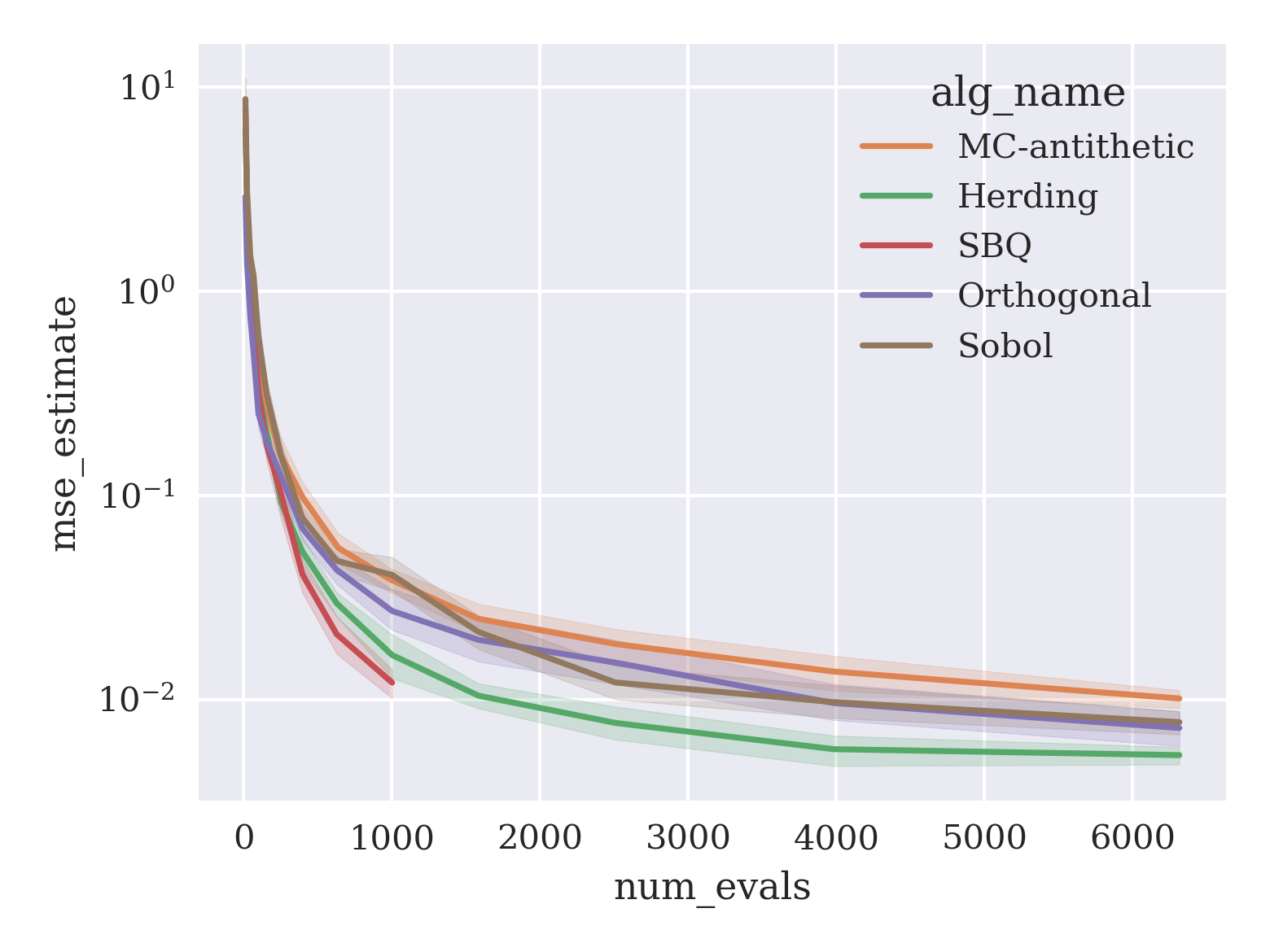}
            \caption[]%
            {{\textit{make\_regression}}}    
        \end{subfigure}
        \hfill
        \begin{subfigure}[b]{0.495\textwidth}   
            \centering 
            \includegraphics[width=\textwidth]{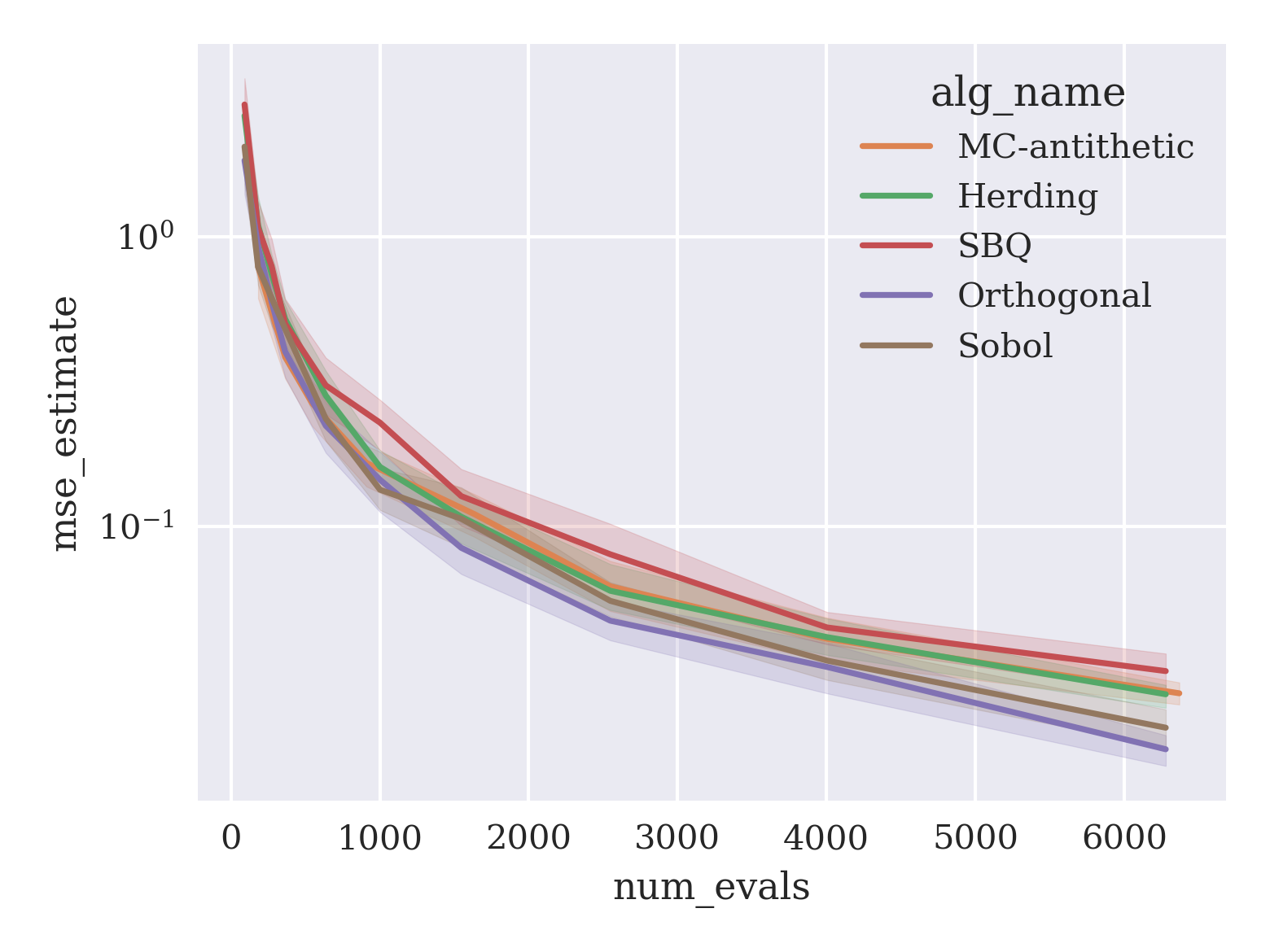}
            \caption[]%
            {{\textit{year}}}    
        \end{subfigure}
        \caption{Proposed algorithms - Tabular data, MLP models} 
        \label{fig:mlp_tabular}
\end{figure*}

\subsection{Proposed algorithms - Discrepancy scores}
\label{sec:proposed_discrepancy}
\label{sec:discrepancy_eval}
Table \ref{tab:discrepancy} shows mean discrepancies over 25 trials for the various permutation sampling algorithms, calculated as per Equation \ref{eq:discrepancy} using the Mallows kernel with $\lambda=4$. Runtime (in seconds) is also reported, where permutation sets are generated using a single thread of a Xeon E5-2698 CPU. We omit results for SBQ at $n=1000$ due to large runtime. At low dimension, the methods directly optimising discrepancy (herding and SBQ) achieve significantly lower discrepancies than the other methods. For $d=10$, $n=1000$, herding achieves almost a twofold reduction in discrepancy over antithetic sampling, directly corresponding to an almost twofold lower error bound under the Koksma-Hlawka inequality.
Antithetic sampling has a higher discrepancy than all other methods here, except in one case ($d=200$, $n=10$) where it achieves lower discrepancy than herding and SBQ. In general, we see the orthogonal and Sobol methods are the most effective at higher dimensions, collectively accounting for the lowest discrepancies at $d=200$. When $n$ is large, the runtime of the herding and SBQ methods becomes impractical. Herding takes as long as 242s to generate $n=1000$ permutations at $d=200$. The Sobol and Orthogonal methods have more reasonable runtimes, the longest of which occurs with Sobol at $n=1000, d=200$, taking 2s. These results show that no single approach is best for all problems but significant improvements can be made over the baseline MC-antithetic method.

The discrepancies computed above are applicable beyond the particular machine learning problems discussed in this paper. Table \ref{tab:discrepancy} provides a reference for how to select samples of permutations at a given computational budget and dimension, not just for Shapley value approximation, but for any bounded function $f : \mathfrak{S}_d \rightarrow \mathbb{R}$. 
\begin{table}[ht!]
\footnotesize
\caption{Discrepancy (lower is better) of permutation samples using Mallows kernel $\lambda=4$}.
\label{tab:discrepancy}
%\vskip 0.150in
\centering
%\scriptsize
\begin{sc}
\begin{tabular}{lllrrrr}
\toprule
    &      &       & \multicolumn{2}{l}{Discrepancy} & \multicolumn{2}{l}{Time} \\
    &      &       &        mean &       std &        mean &        std \\
d & n & Algorithm &             &           &             &            \\
\midrule
\multirow{15}{*}{10} & \multirow{5}{*}{10} & Herding &    0.241 &  0.002 &    0.008 &   0.001 \\
    &      & MC-antithetic &    0.264 &  0.010 &    0.000 &   0.000 \\
    &      & Orthogonal &    0.244 &  0.003 &    0.001 &   0.000 \\
    &      & SBQ &    0.240 &  0.002 &    0.112 &   0.397 \\
    &      & Sobol &    0.258 &  0.007 &    0.003 &   0.006 \\
\cline{2-7}
    & \multirow{5}{*}{100} & Herding &    0.059 &  0.001 &    0.980 &   0.603 \\
    &      & MC-antithetic &    0.084 &  0.004 &    0.001 &   0.001 \\
    &      & Orthogonal &    0.070 &  0.002 &    0.012 &   0.029 \\
    &      & SBQ &    0.056 &  0.000 &   41.546 &   9.239 \\
    &      & Sobol &    0.069 &  0.002 &    0.048 &   0.168 \\
\cline{2-7}
    & \multirow{5}{*}{1000} & Herding &    0.013 &  0.000 &   52.961 &   4.024 \\
    &      & MC-antithetic &    0.027 &  0.002 &    0.019 &   0.040 \\
    &      & Orthogonal &    0.022 &  0.001 &    0.110 &   0.239 \\
    &      & SBQ &           - &         - &           - &          - \\
    &      & Sobol &    0.018 &  0.000 &    0.049 &   0.139 \\
\cline{1-7}
\cline{2-7}
\multirow{15}{*}{50} & \multirow{5}{*}{10} & Herding &    0.270 &  0.001 &    0.023 &   0.047 \\
    &      & MC-antithetic &    0.272 &  0.002 &    0.001 &   0.003 \\
    &      & Orthogonal &    0.269 &  0.000 &    0.024 &   0.045 \\
    &      & SBQ &    0.270 &  0.001 &    0.344 &   0.879 \\
    &      & Sobol &    0.271 &  0.001 &    0.009 &   0.007 \\
\cline{2-7}
    & \multirow{5}{*}{100} & Herding &    0.080 &  0.000 &    1.129 &   0.483 \\
    &      & MC-antithetic &    0.086 &  0.001 &    0.001 &   0.000 \\
    &      & Orthogonal &    0.072 &  0.000 &    0.054 &   0.170 \\
    &      & SBQ &    0.079 &  0.000 &   27.135 &   7.967 \\
    &      & Sobol &    0.079 &  0.000 &    0.009 &   0.006 \\
\cline{2-7}
    & \multirow{5}{*}{1000} & Herding &    0.023 &  0.000 &   85.039 &   3.604 \\
    &      & MC-antithetic &    0.027 &  0.000 &    0.049 &   0.201 \\
    &      & Orthogonal &    0.023 &  0.000 &    0.352 &   1.165 \\
    &      & SBQ &           - &         - &           - &          - \\
    &      & Sobol &    0.022 &  0.000 &    0.960 &   0.713 \\
\cline{1-7}
\cline{2-7}
\multirow{15}{*}{200} & \multirow{5}{*}{10} & Herding &    0.280 &  0.001 &    0.112 &   0.401 \\
    &      & MC-antithetic &    0.273 &  0.000 &    0.000 &   0.000 \\
    &      & Orthogonal &    0.272 &  0.000 &    0.196 &   0.051 \\
    &      & SBQ &    0.280 &  0.001 &    0.098 &   0.185 \\
    &      & Sobol &    0.272 &  0.000 &    0.795 &   1.436 \\
\cline{2-7}
    & \multirow{5}{*}{100} & Herding &    0.084 &  0.000 &    3.429 &   1.765 \\
    &      & MC-antithetic &    0.086 &  0.000 &    0.043 &   0.121 \\
    &      & Orthogonal &    0.083 &  0.000 &    0.464 &   1.134 \\
    &      & SBQ &    0.084 &  0.000 &   39.163 &  10.230 \\
    &      & Sobol &    0.084 &  0.000 &    0.692 &   0.778 \\
\cline{2-7}
    & \multirow{5}{*}{1000} & Herding &    0.026 &  0.000 &  242.516 &   6.934 \\
    &      & MC-antithetic &    0.027 &  0.000 &    0.007 &   0.002 \\
    &      & Orthogonal &    0.023 &  0.000 &    0.561 &   0.212 \\
    &      & SBQ &           - &         - &           - &          - \\
    &      & Sobol &    0.023 &  0.000 &    1.996 &   0.782 \\
\bottomrule
\end{tabular}
\end{sc}
\end{table}

\subsection{Proposed algorithms - Image data and deep CNN models}
\label{sec:proposed_images}
We continue by evaluating the effectiveness of the proposed sampling algorithms for an image classification interpretability problem. Figure \ref{fig:resnet} depicts eight images randomly selected from the ImageNet 2012 dataset of \cite{russakovsky2015imagenet}. We use approximate Shapley values to examine the contribution of the different image tiles towards the output label predicted by a ResNet50 \citep{he2016deep} convolutional neural network. Images are preprocessed as per \cite{he2016deep}, by cropping to a 1:1 aspect ratio, centering along the larger axis, resizing to 224x224, and subtracting the mean RGB values of the ImageNet training set. We examine the highest probability class output for each image. The predicted labels are displayed above each image in Figure \ref{fig:resnet}. Note that labels may be incorrect (e.g. ``vacuum"). To examine the Shapley values for each image, we group pixels into 14x14x3 tiles, considering each tile to be a single feature. This reduces the dimensionality of the interpretability problem from $224 \cdot 224 \cdot 3=150,528$ to a more tractable 256 dimensions. When a tile is not part of the active feature set, its pixel values are set to (0,0,0) (black). For the purpose of computing Shapley values, we examine the log-odds output of the ResNet50 model, as the additivity of these outputs is consistent with the efficiency property of Shapley values. Sampling algorithms are  applied to the Shapley value problem 25 times, each with a different seed. As computing an exact baseline is intractable, we estimate the mean squared error in the same manner as Section \ref{sec:proposed_mlp}. Error estimates are presented as a bar chart in the third column of Figure \ref{fig:resnet}. The second column displays a heat map of the estimated Shapley values for the first trial of the sampling algorithm with the lowest error estimate for the corresponding image. Yellow areas show image tiles that contribute positively to the predicted label, darker purple areas correspond to areas contributing negatively to the predicted label. From this analysis, we see that the Sobol method has the lowest error estimate in all cases. While the herding, orthogonal and SBQ methods generally show lower sample variance than plain Monte Carlo, they do not appear to generate significantly better solutions than the much simpler MC-antithetic method for this problem. This raises the question of whether the herding and SBQ methods could do better with a better choice of $\lambda$ parameter. However, Figure \ref{fig:resnet_lambda} in Appendix \ref{app:lambda} shows that alternative parameter values do not significantly improve the performance of herding and SBQ for this problem.

Table \ref{tab:image_time} shows the execution time of permutation generation compared compared to other computation needed to generate the Shapley values for a single image. This other computation consists of evaluating ResNet50 and performing weighted averages. Generating Shapley values for an image using 100 permutation samples and 256 features requires $100 \cdot (256+1)=25700$ model evaluations, taking around 40s on an Nvidia V100 GPU. Permutations are generated using a single thread of a Xeon E5-2698 CPU. Of the permutation sampling algorithms, we see that the linear-time algorithms (MC, MC-antithetic, Orthogonal, Sobol) do not significantly affect total runtime, however the runtime of the Herding and SBQ algorithms is significant relative to the time required for obtaining predictions from the model.

\begin{figure}   
            \centering 
            \includegraphics[width=0.95\textwidth]{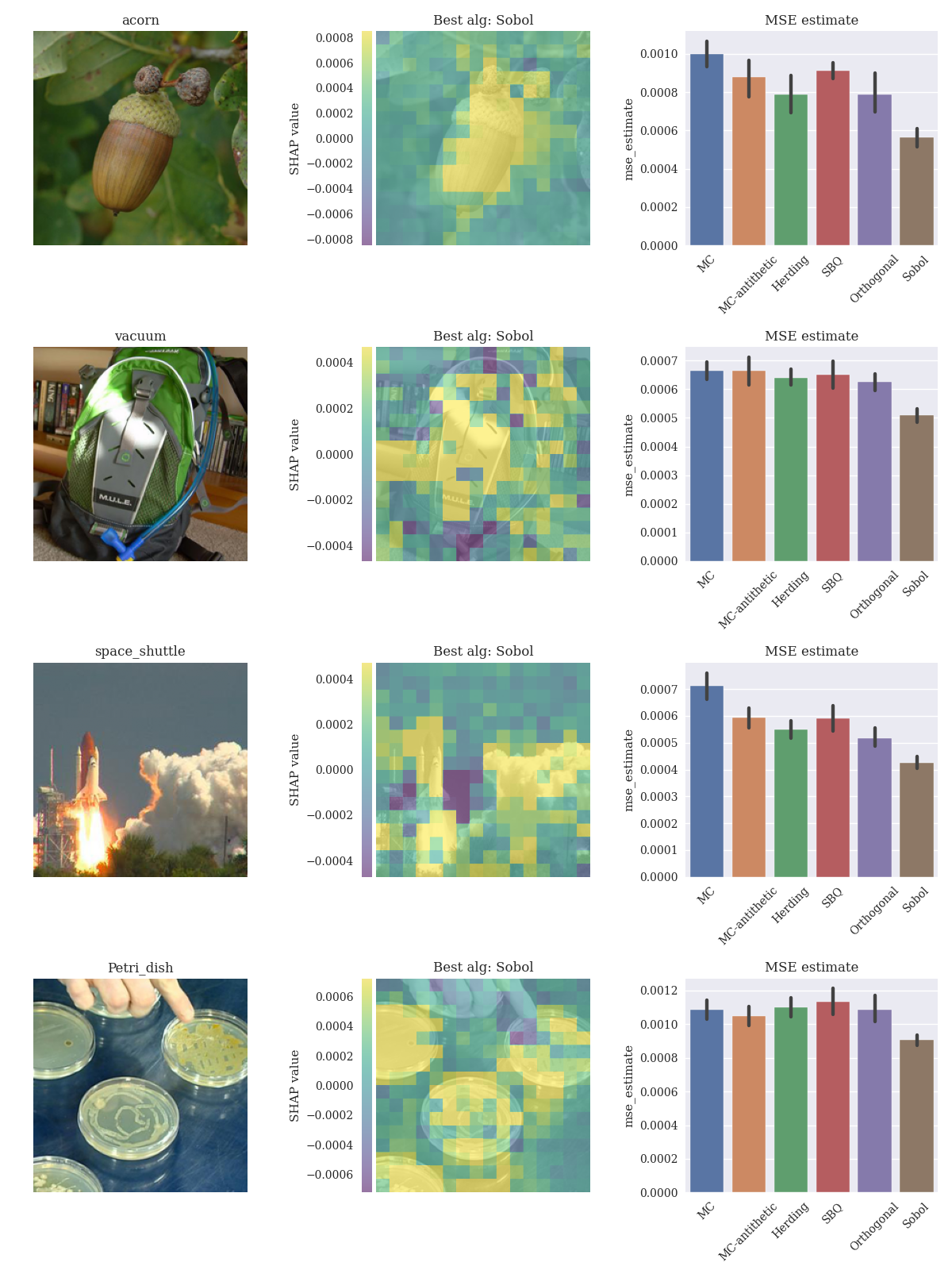}
            \caption{MSE estimates for 100 permutation samples applied to image classifications made by ResNet50}    
            \label{fig:resnet}
\end{figure}
\renewcommand{\thefigure}{\arabic{figure} (Cont.)}
\addtocounter{figure}{-1}
\begin{figure}   
            \centering 
            \includegraphics[width=0.95\textwidth]{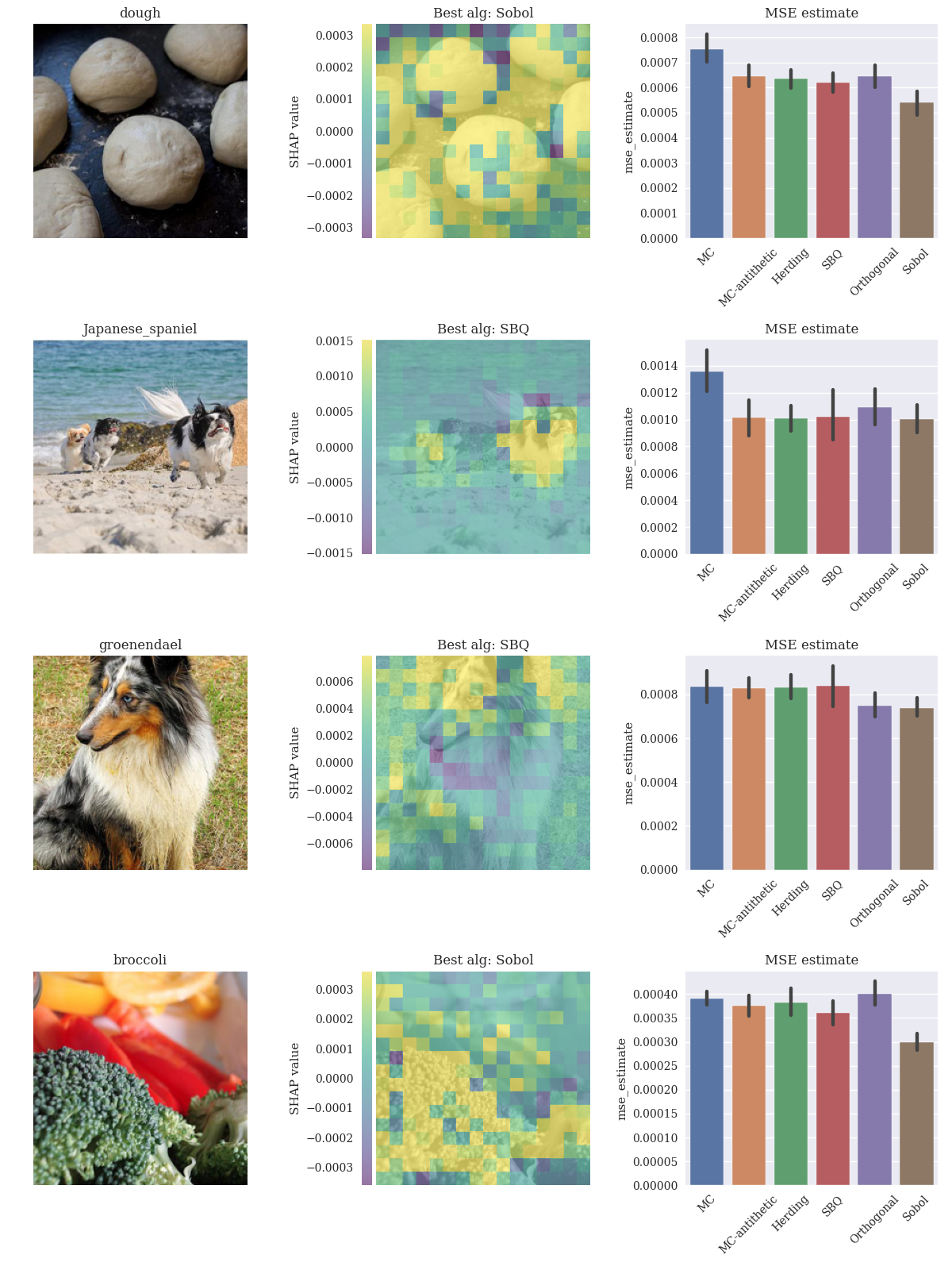}
            \caption{MSE estimates for 100 permutation samples applied to image classifications made by ResNet50}    
\end{figure}
\renewcommand{\thefigure}{\arabic{figure}}
\begin{table}
\centering
\begin{tabular}{lrrrr}
\toprule
{} & \multicolumn{2}{l}{Permutation time (s)} & \multicolumn{2}{l}{Other time (s)} \\
{} &                 mean &       std &           mean &       std \\
Algorithms    &                      &           &                &           \\
\midrule
Herding       &             3.050 &  0.431 &      40.791 &  0.491 \\
MC            &             0.001 &  0.000 &      40.586 &  0.538 \\
MC-antithetic &             0.001 &  0.000 &      40.898 &  0.553 \\
Orthogonal    &             0.231 &  0.012 &      40.666 &  0.460 \\
SBQ           &             6.253 &  1.126 &      40.480 &  0.437 \\
Sobol         &             0.050 &  0.019 &      40.622 &  0.546 \\
\bottomrule
\end{tabular}
\caption{Time to generate Shapley values for a single image, separated into time to generate 100 permutations, and other (model evaluation and averaging of model evaluations). Linear-time algorithms all account for $<$ 0.125\% of Shapley value run-time. Run-time of the non-linear-time algorithms (Herding, SBQ) is much more significant.}
\label{tab:image_time}
\end{table}

\section{Conclusion}
In this work, we propose new techniques for the approximation of Shapley values in machine learning applications based on careful selection of samples from the symmetric group $\mathfrak{S}_d$. One set of techniques draws on theory of reproducing kernel Hilbert spaces and the optimisation of discrepancies for functions of permutations, and another exploits connections between permutations and the hypersphere $\mathbb{S}^{d-2}$. We perform empirical analysis of approximation error for GBDT and neural network models trained on tabular data and image data. We also evaluate data-independent discrepancy scores for various sampling algorithms at different dimensionality and sample sizes. The introduced sampling methods show improved convergence over existing state-of-the-art methods in many cases. Our results show that kernel-based methods may be more effective for lower-dimensional problems, and methods sampling from $\mathbb{S}^{d-2}$ are more effective for higher-dimensional problems. Further work may be useful to identify the precise conditions under which optimising discrepancies based on a Mallows kernel is effective, and to clarify the impact of dimensionality on choice of sampling algorithm for Shapley value approximation.

\bibliography{main}
\appendix
\section{Proof of Theorem~\ref{thm:k_tau} (See page~\pageref{thm:k_tau})}
\ktau*
\label{app:proof}
\begin{proof}
For $1 \leq a \leq d-1$, write $t_a \in \mathfrak{S}_d$ for the adjacent transposition of $a$ and $a+1$, i.e., the permutation so that $t_a(j) = j$ for $j \neq a,a+1$, $t_a(a) = a+1$ and $t_a(a+1) = a$.  We interpret a product of permutations to be their composition as functions.  For a permutation $\pi \in \mathfrak{S}_d$, write $\nu(\pi)$ for the quantity $\sum_{j=1}^d j \pi(j)$, and note that $\nu(I) = \sum_{j=1}^d j^2 = d(d+1)(2d+1)/6$.

It is well-known that the number of inversions $n_{\textrm{dis}}(I,\pi) = |\{(i,j) : i < j \textrm{ and } \pi(i) > \pi(j)\}|$ in a permutation $\pi$ equals the least $k$ so that there exist $a_1,\ldots,a_k$ with
\begin{equation} \label{eq1}
\pi = \prod_{i=1}^k t_{a_i}.
\end{equation}
This quantity $k$ is known as the ``length'' of $\pi$ and is exactly the distance in the $1$-skeleton of the permutohedron representation of $\mathfrak{S}_d$.  Furthermore, the $a_i$ can be obtained via bubble sort, i.e., the product (\ref{eq1}) begins with 
$$
t_{\pi(1)-1} t_{\pi(1)-2} \cdots t_{1}
$$
and proceeds recursively on $\pi|_{\{2,\ldots,d\}}$.  Write $\pi_j$ for the product of the first $j$ terms in (\ref{eq1}) for $1 \leq j \leq k$, i.e., $\pi_j = \prod_{i=1}^j t_{a_i}$, with $\pi_0 = I$.  Then the pairs $e_j = \{\pi_j(a_j),\pi_j(a_{j}+1)\}$ are all distinct, because entries of $\pi$ in one-line notation switch places at most once when applying the adjacent transpositions, i.e., a larger value $a$, once it switches places with a smaller value $b$ immediately to its left, never switches place with $b$ again.  Furthermore, note that
\begin{align*}
\nu(\pi_{j+1})-\nu(\pi_j) &= (j \pi_{j+1}(a_j) + (j+1) \pi_{j+1}(a_{j}+1)) - (j \pi_j(a_j) + (j+1) \pi_j(a_{j}+1)) \\
&= (j \pi_{j}(a_{j}+1) + (j+1) \pi_j(a_j)) - (j \pi_j(a_j) + (j+1) \pi_j(a_{j}+1)) \\
&= \pi_{j}(a_{j}+1) - \pi_j(a_j),
\end{align*}
a quantity which is always negative because the sequence of transpositions obtained above only ever increases the number of inversions.  Therefore, the collection $\{e_j\}_{j=1}^k$ consists of $k$ distinct edges of a complete graph on $\{1,\ldots,d\}$ and 
\begin{align*}
\nu(\pi) &= \nu(\pi_k) = \nu(\pi_k) - \nu(I) + \frac{d(d+1)(2d+1)}{6} \\
&= \frac{d(d+1)(2d+1)}{6} + \sum_{j=1}^k \pi_{j}(a_{j}+1) - \pi_j(a_j) \\
&= \frac{d(d+1)(2d+1)}{6} - \sum_{j=1}^k \wt(e_j)
\end{align*}
where $\wt(\{a,b\}) = |b-a|$.  By greedily selecting the highest-weight or lowest-weight edges of the complete graph $K_d$ weighted by $\wt(\cdot)$, the quantity $\sum_{j=1}^k \wt(e_j)$ is always at least
$$
1 \cdot (d-1) + 2 \cdot (d-2) + \cdots + (d-m) \cdot m = \frac{(d + 2m - 1)(d - m + 1)(d - m)}{6}
$$
where $m$ is the smallest integer so that $\sum_{j=1}^{d-m} (d-j) = (d + m - 1)(d - m)/2 \leq k$, because the summands correspond to $d-1$ edges of weight $1$, $d-2$ edges of weight $2$, and so on up to $m$ edges of weight $d-m$.  Similarly, $\sum_{j=1}^k \wt(e_j)$ is at most
$$
(d-1) \cdot 1 + (d-2) \cdot 2 + \cdots + M \cdot (d-M) = \frac{(d + 2M - 1)(d - M + 1)(d - M)}{6}
$$
where $M$ is the largest integer so that $\sum_{j=1}^{d-M} j = (d-M)(d-M+1)/2 \geq k$, since in this case we bound the total edge weight via $1$ edge of weight $d-1$, $2$ edges of weight $d-2$, and so on up to $d-M$ edges of weight $M$. Then, letting $\alpha = k/\binom{d}{2}$ (so that $\alpha \in [0,1]$),
\begin{align*}
m &= \left \lfloor \frac{\sqrt{4d^2 - 4d - 8k + 1} + 1}{2} \right \rfloor = d \sqrt{1- \alpha} \pm 1 \\
M &= \left \lceil \frac{2d - \sqrt{8k + 1} + 1}{2} \right \rceil = d(1-\sqrt{\alpha}) \pm 1
\end{align*}
It is straightforward to verify that, if $f(s) = (d + 2s - 1)(d - s + 1)(d - s)/6$, then $s = O(d)$ implies $f(s \pm 1) = f(s) + O(d^2)$. So, letting $\alpha = k/
\binom{d}{2}$ (so that $\alpha \in [0,1]$)
\begin{align*}
\nu(\pi)& \leq \frac{d(d+1)(2d+1)}{6} - f(M) \\
& = \frac{d(d+1)(2d+1)}{6} - f(d\sqrt{1-\alpha}) + O(d^2) \\
& = \frac{d^3}{3} - \frac{d^3(1 + 2 \sqrt{1-\alpha})(1-\sqrt{1-\alpha})^2}{6} + O(d^2) \\
& = d^3 \left ( \frac{2}{3} - \frac{\alpha}{2} - \frac{(1-\alpha)^{3/2}}{3} \right ) + O(d^2)
\end{align*}
and
\begin{align*}
\nu(\pi)& \geq \frac{d(d+1)(2d+1)}{6} - f(m) \\
& = \frac{d^3}{3} - f(d(1-\sqrt{\alpha})) + O(d^2) \\
& = \frac{d^3}{3} - \frac{d^3 (1 + 2(1-\sqrt{\alpha})) (1 - (1-\sqrt{\alpha}))^2}{6}  + O(d^2) \\
& = d^3 \left ( \frac{1}{3} - \frac{\alpha}{2} + \frac{\alpha^{3/2}}{3} \right ) + O(d^2).
\end{align*}
(Note that the functions in parentheses meet for $\alpha = 0,1$.) Thus, applying the fact that $\nu(\sigma' \circ \sigma^{-1}) = I^T (\sigma' \circ \sigma^{-1}) = \sigma^{T} \sigma'$, where we regard permutations both as functions $\pi$ of $\{1,\ldots,d\}$ and as vectors $(\pi(1),\ldots,\pi(d))$,
$$
 2 + 2 \alpha^{3/2} \leq \frac{6 \sigma^{T} \sigma'}{d^3} + O(d^{-1}) + 3 \alpha \leq 4 - 2(1-\alpha)^{3/2} 
$$
Then, since 
$$
K_\tau(\sigma,\sigma') = 1 - \frac{2 n_{\textrm{dis}}( I,\sigma' \sigma^{-1})}{\binom{d}{2}} = 1 - 2 \alpha
$$
we have
\begin{align*}
\frac{1}{4} + \left (\frac{1-K_\tau(\sigma,\sigma')}{2} \right )^{3/2} \leq \frac{3 \sigma^{T} \sigma'}{d^3} + O(d^{-1}) - \frac{ 3K_\tau(\sigma,\sigma')}{4} \leq \frac{5}{4}-\left ( \frac{1+K_\tau(\sigma,\sigma')}{2} \right )^{3/2}.
\end{align*}
Writing $\sigma = \rho x + \mu$ and $\sigma' = \rho x' + \mu$ yields the first claim of the result, since then
$$
\sigma^T \sigma' = \frac{d(d^2-1)}{12} A(\sigma)^T A(\sigma') + \frac{d(d+1)^2}{4}.
$$
For the second claim, note that, if $\sigma^T \sigma' = d^3 (1/4 + o(1))$ (the expected value for random permutations, corresponding to $A(\sigma)^T A(\sigma') \approx 0$),
$$
-2 + 4 \left (\frac{1-K_\tau(\sigma,\sigma')}{2} \right )^{3/2} \leq  -3 K_\tau(\sigma,\sigma') + O(d^{-1}) \leq 2-4 \left ( \frac{1+K_\tau(\sigma,\sigma')}{2} \right )^{3/2},
$$
i.e.,
$$
|K_\tau(\sigma,\sigma')| \leq 1/2 + o(1).
$$
\end{proof}

\section{Selection of parameters for the Mallows kernel}
\label{app:lambda}
The experimental analysis of Section \ref{sec:evaluation} requires the selection of a Mallows kernel $\lambda$ parameter for the kernel herding and SBQ algorithms, and for the calculation of discrepancies reported in Table \ref{tab:discrepancy}. As a matter of practicality, we limit the comparisons to a single version of the Mallows kernel due to space constraints. In theory, this parameter could be tuned and the optimal performance reported for each dataset, however, we consider this an unfair reflection of the algorithms performance, as the total number of samples, including the tuning phase, would be considerably higher than for the other algorithms. For kernel-based methods to be effective in practice they should not require extensive parameter tuning. Therefore, we fix $\lambda=4$, choosing this as an acceptable value based on experiments on different data sources presented below.

Figures \ref{fig:lambda_tabular_gbdt}, \ref{fig:lambda_tabular_mlp}, and \ref{fig:resnet_lambda} show the error of the kernel herding algorithm using 100 permutation samples and various $\lambda$ values. As usual, the shaded areas represent 95\% confidence intervals. We perform these experiments for tabular datasets with GBDT models, tabular datasets with MLP models, and image data with a ResNet50 model, corresponding to the experiments of Section \ref{sec:evaluation}. For some dataset/model combinations a smaller $\lambda$ value appears to be preferable, for others a larger value is preferable. In the case of image data, the impact of the parameter is small in terms of total MSE, and for tabular data, it is difficult to assign any particular trend due to the volatility of the results. In summary, we compromise with a selection of $\lambda=4$, which appears to perform acceptably in a wide range of cases. 

\begin{figure*}[ht]
        \centering
        \begin{subfigure}[b]{0.495\textwidth}
            \centering
            \includegraphics[width=\textwidth]{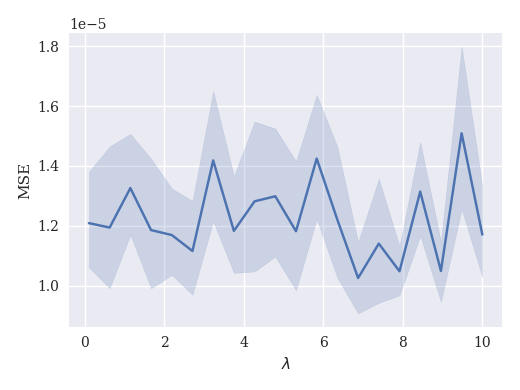}
            \caption[]%
            {{\textit{adult}}}    
        \end{subfigure}
        \hfill
        \begin{subfigure}[b]{0.495\textwidth}  
            \centering 
            \includegraphics[width=\textwidth]{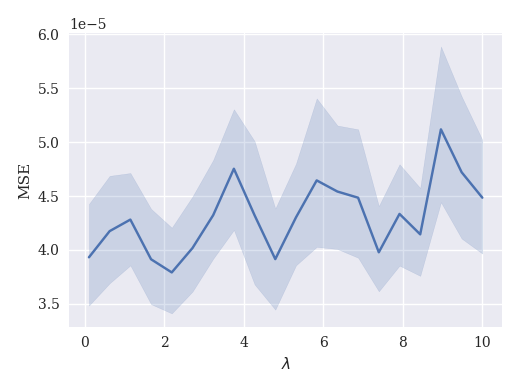}
            \caption[]%
            {{\textit{breast\_cancer}}}    
        \end{subfigure}
%        \vskip\baselineskip
        \begin{subfigure}[b]{0.495\textwidth}   
            \centering 
            \includegraphics[width=\textwidth]{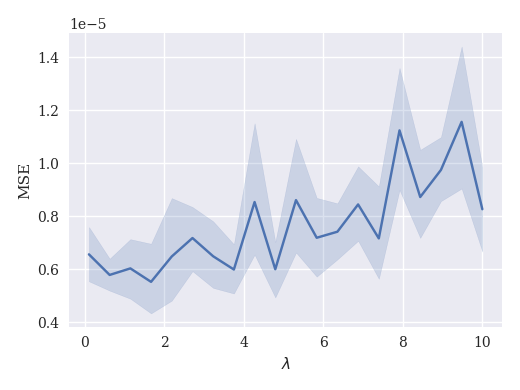}
            \caption[]%
            {{\textit{bank}}}    
        \end{subfigure}
        \hfill
        \begin{subfigure}[b]{0.495\textwidth}   
            \centering 
            \includegraphics[width=\textwidth]{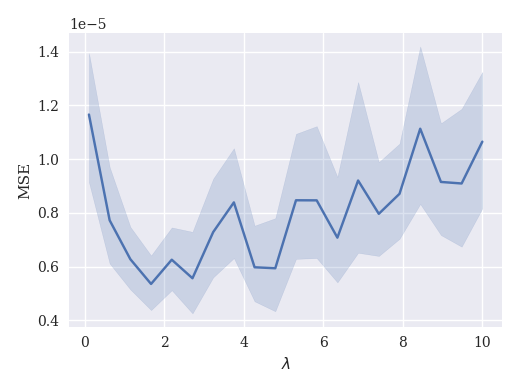}
            \caption[]%
            {{\textit{cal\_housing}}}    
        \end{subfigure}

         \begin{subfigure}[b]{0.495\textwidth}   
            \centering 
            \includegraphics[width=\textwidth]{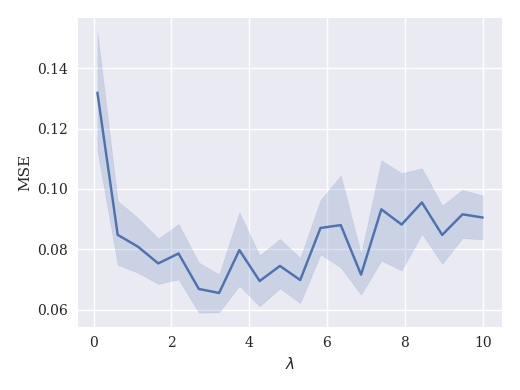}
            \caption[]%
            {{\textit{make\_regression}}}    
        \end{subfigure}
        \hfill
        \begin{subfigure}[b]{0.495\textwidth}   
            \centering 
            \includegraphics[width=\textwidth]{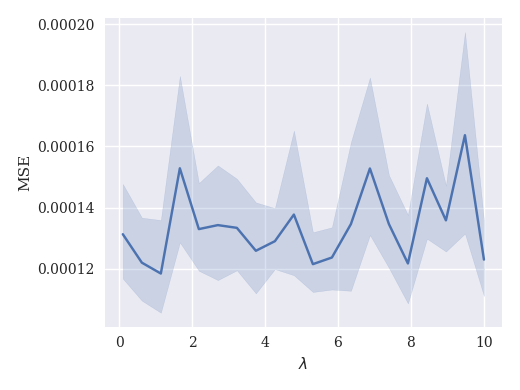}
            \caption[]%
            {{\textit{year}}}    
        \end{subfigure}
        \caption{Varying $\lambda$ for 100 herding samples --- Tabular data and GBDT models. Selection of a consistently effective $\lambda$ value is unclear.} 
        \label{fig:lambda_tabular_gbdt}
\end{figure*}

\begin{figure*}[ht]
        \centering
        \begin{subfigure}[b]{0.495\textwidth}
            \centering
            \includegraphics[width=\textwidth]{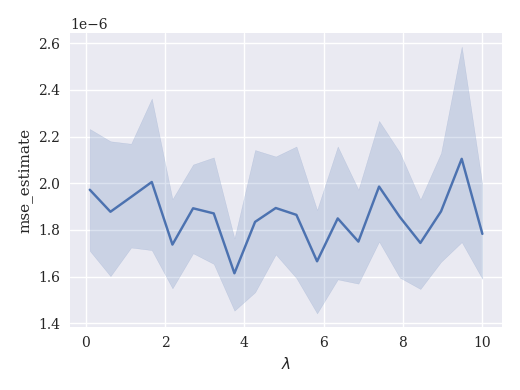}
            \caption[]%
            {{\textit{adult}}}    
        \end{subfigure}
        \hfill
        \begin{subfigure}[b]{0.495\textwidth}  
            \centering 
            \includegraphics[width=\textwidth]{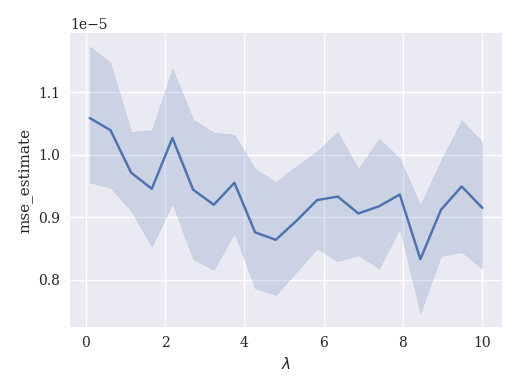}
            \caption[]%
            {{\textit{breast\_cancer}}}    
        \end{subfigure}
%        \vskip\baselineskip
        \begin{subfigure}[b]{0.495\textwidth}   
            \centering 
            \includegraphics[width=\textwidth]{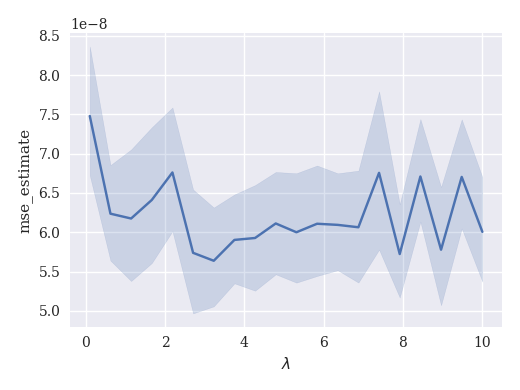}
            \caption[]%
            {{\textit{bank}}}    
        \end{subfigure}
        \hfill
        \begin{subfigure}[b]{0.495\textwidth}   
            \centering 
            \includegraphics[width=\textwidth]{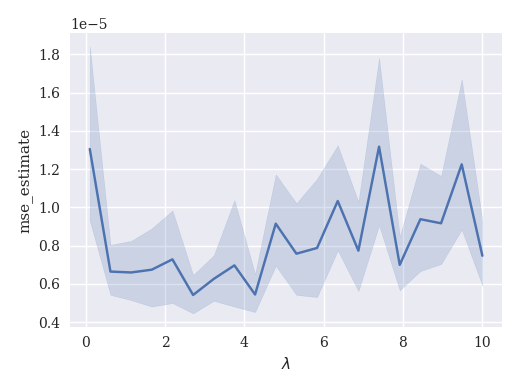}
            \caption[]%
            {{\textit{cal\_housing}}}    
        \end{subfigure}

         \begin{subfigure}[b]{0.495\textwidth}   
            \centering 
            \includegraphics[width=\textwidth]{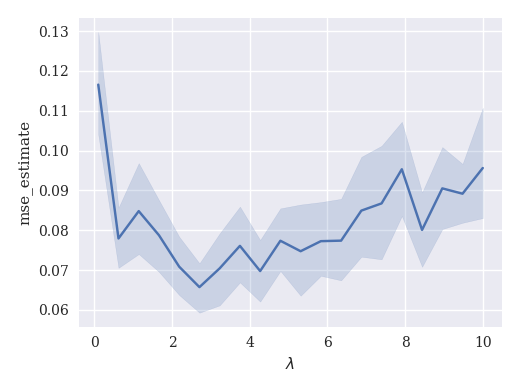}
            \caption[]%
            {{\textit{make\_regression}}}    
        \end{subfigure}
        \hfill
        \begin{subfigure}[b]{0.495\textwidth}   
            \centering 
            \includegraphics[width=\textwidth]{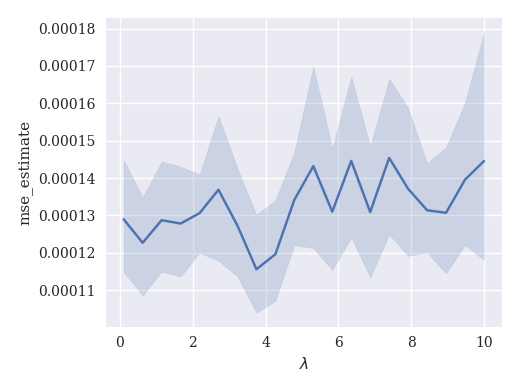}
            \caption[]%
            {{\textit{year}}}    
        \end{subfigure}
        \caption{Varying $\lambda$ for 100 herding samples --- Tabular data and MLP models. Selection of a consistently effective $\lambda$ value is unclear.} 
        \label{fig:lambda_tabular_mlp}
\end{figure*}

\begin{figure}   
            \centering 
            \includegraphics[height=0.95\textheight]{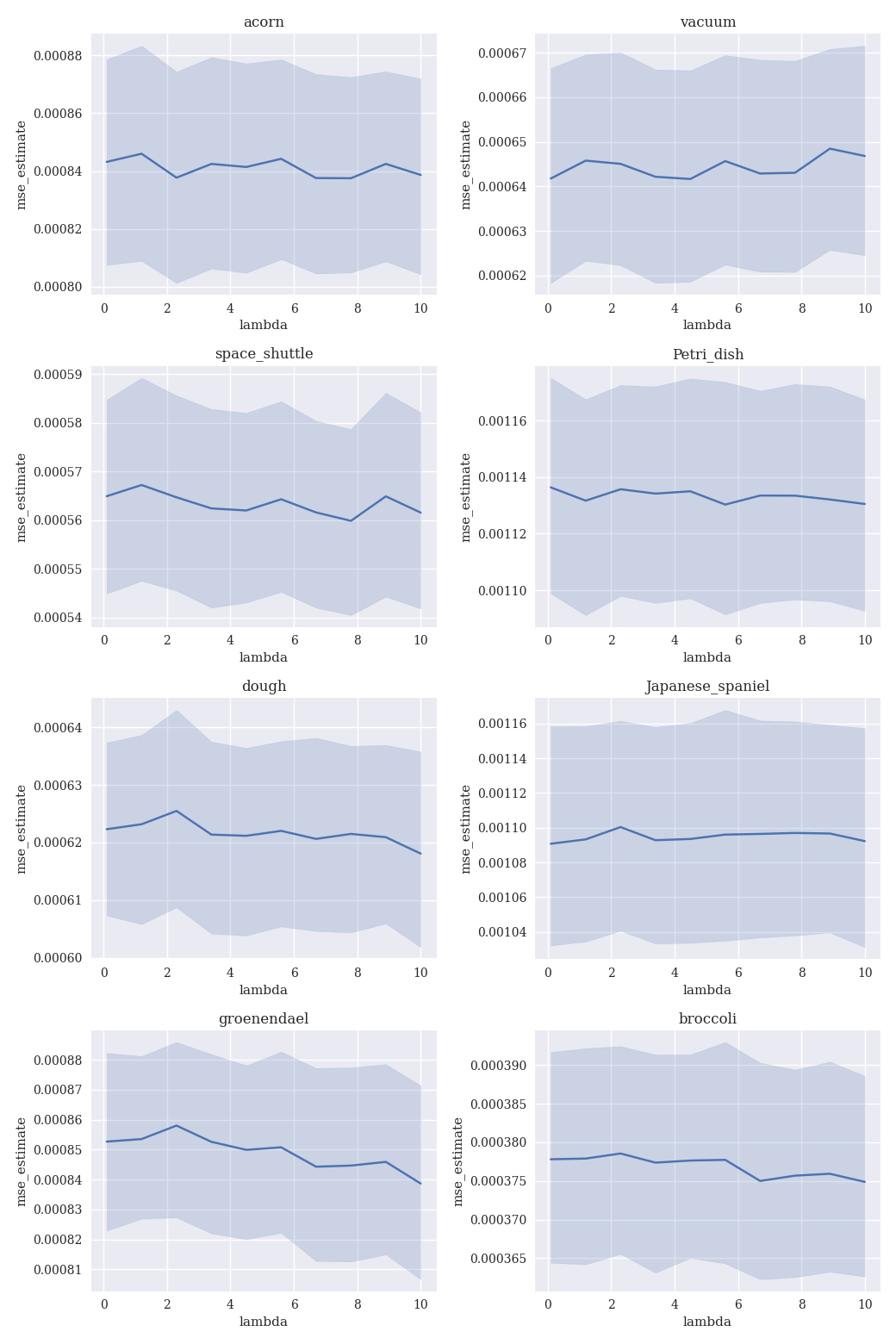}
            \caption{Varying $\lambda$ for 100 herding samples --- Image data and ResNet50 model. Varying the $\lambda$ parameter for our 256 dimensional image data has little impact on average.}    
            \label{fig:resnet_lambda}
\end{figure}

It is also necessary to choose the number of argmax samples for the herding and SBQ algorithms. Recall from Section \ref{sec:kernel_herding} that we approximate the argmax in herding and SBQ, choosing a new permutation sample by selecting a set of uniform random permutations and selecting one to minimise the discrepancy. Figure \ref{fig:argmax_samples} shows the effect of varying the number of argmax samples on mean squared error for tabular datasets and GBDT models. We find that 5 to 10 samples is too low for optimal performance, but there is little difference between 25 and 50 samples, so choose 25 samples as a compromise for good accuracy and reasonable runtime.
\begin{figure*}[ht]
        \centering
        \begin{subfigure}[b]{0.495\textwidth}
            \centering
            \includegraphics[width=\textwidth]{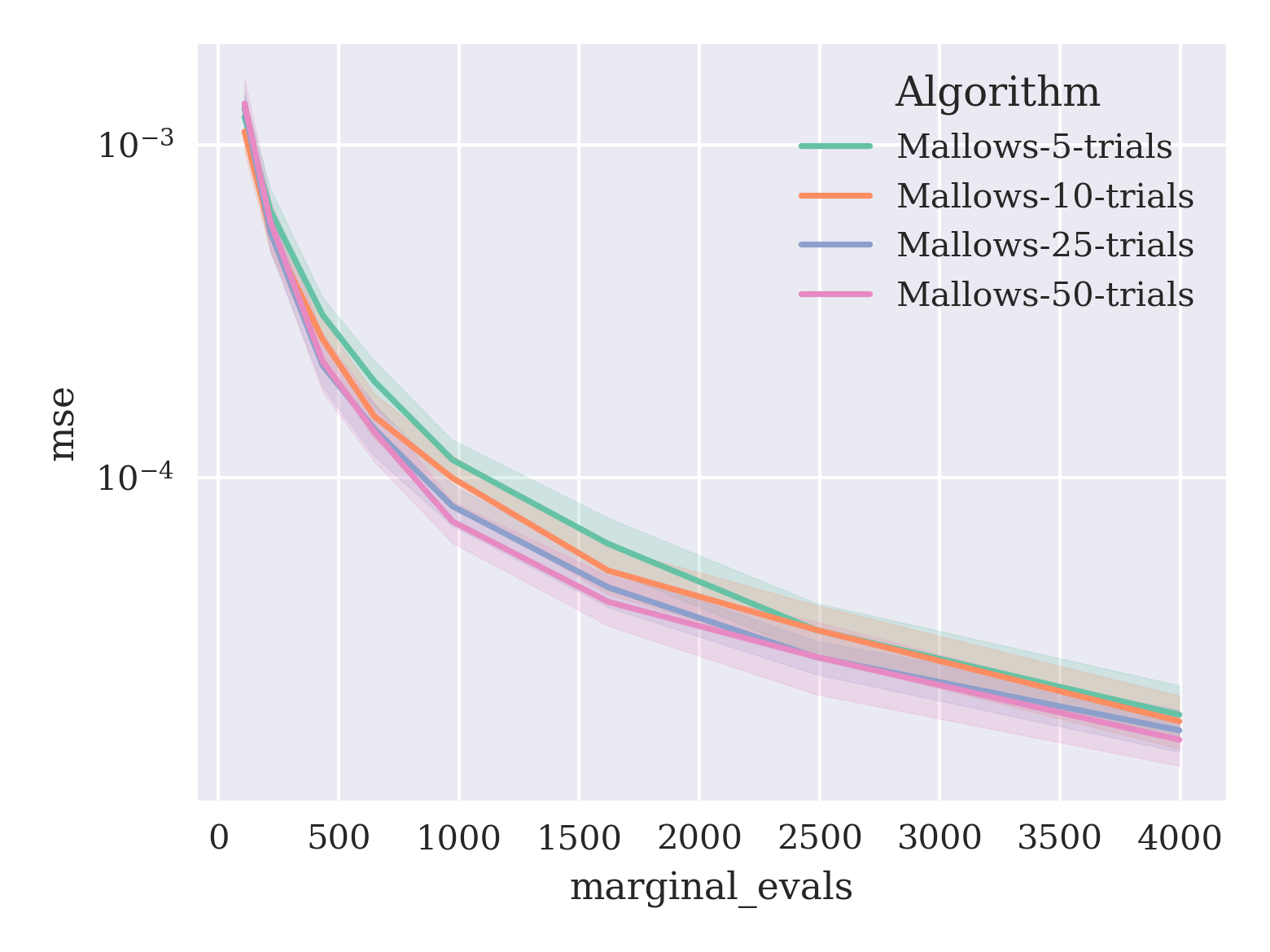}
            \caption[]%
            {{\textit{adult}}}    
        \end{subfigure}
        \hfill
        \begin{subfigure}[b]{0.495\textwidth}  
            \centering 
            \includegraphics[width=\textwidth]{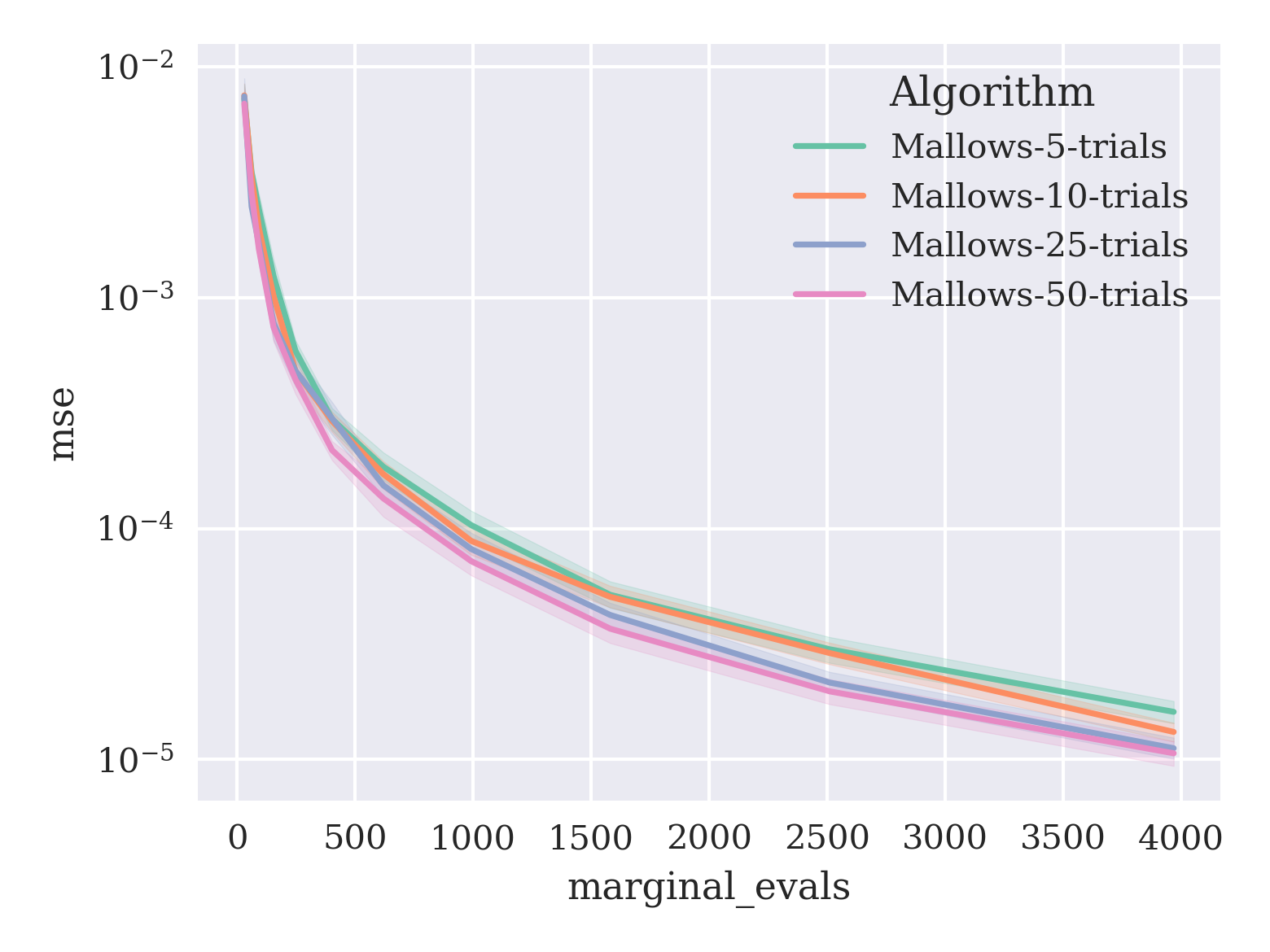}
            \caption[]%
            {{\textit{breast\_cancer}}}    
        \end{subfigure}
%        \vskip\baselineskip
        \begin{subfigure}[b]{0.495\textwidth}   
            \centering 
            \includegraphics[width=\textwidth]{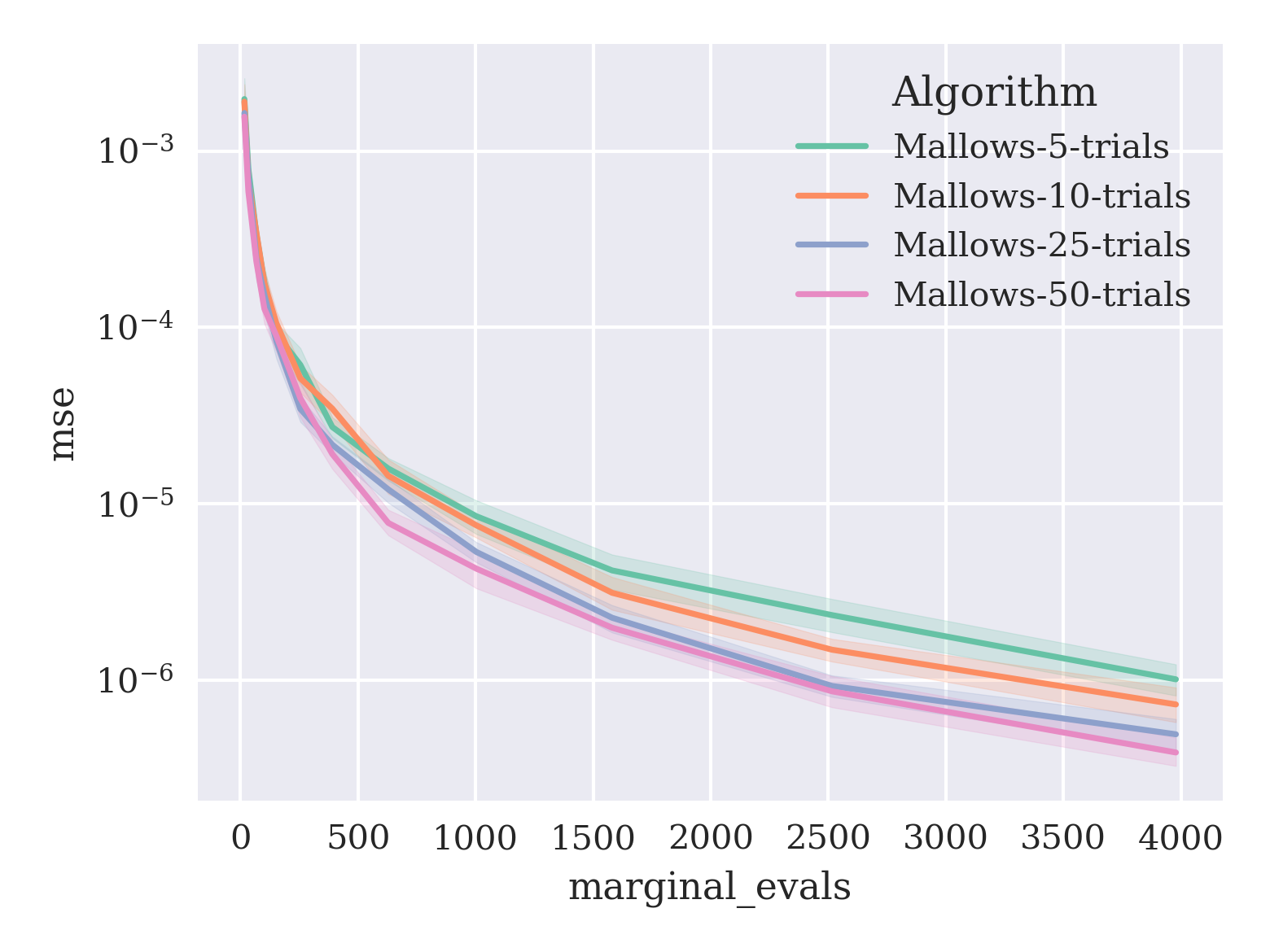}
            \caption[]%
            {{\textit{bank}}}    
        \end{subfigure}
        \hfill
        \begin{subfigure}[b]{0.495\textwidth}   
            \centering 
            \includegraphics[width=\textwidth]{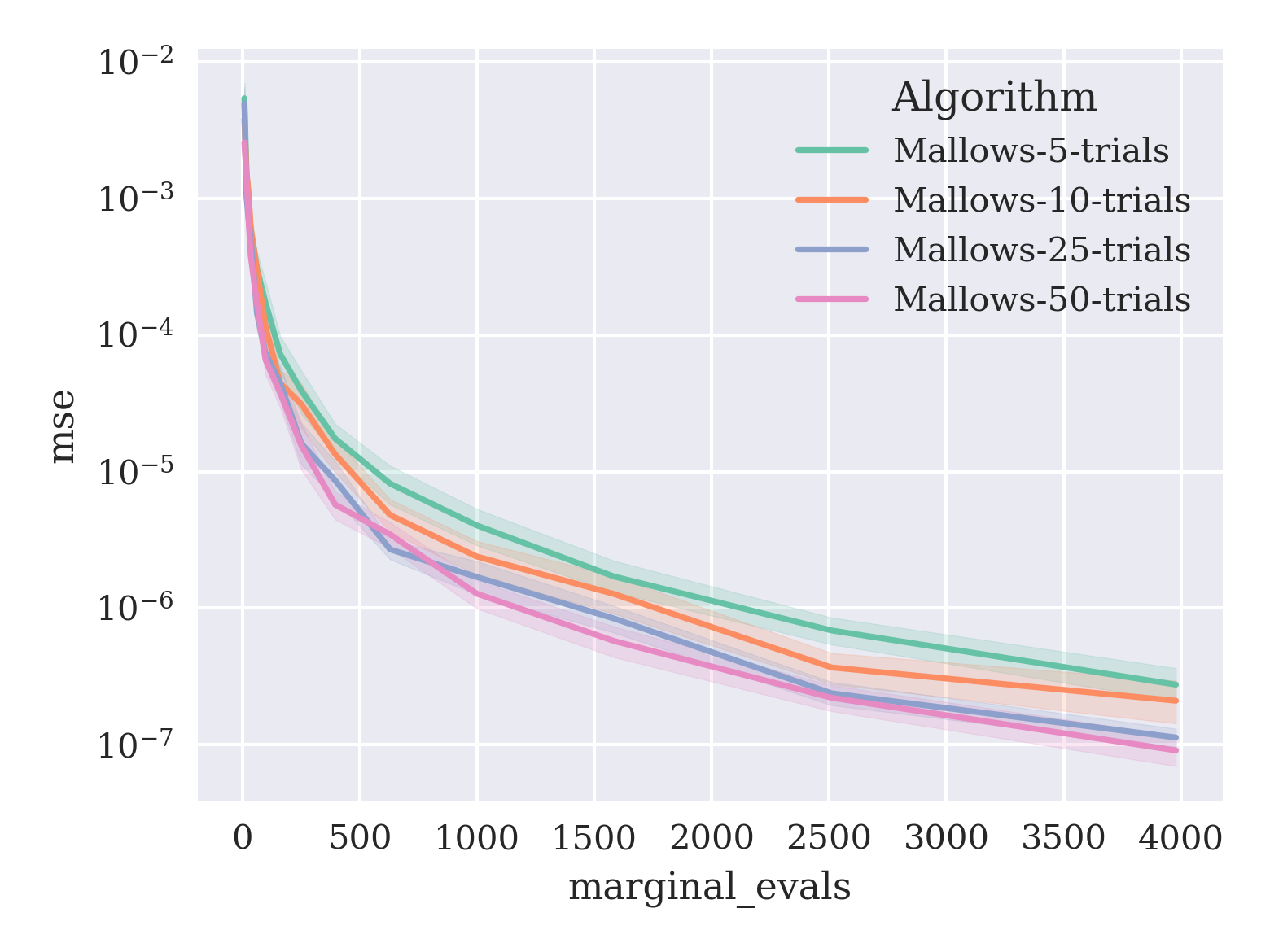}
            \caption[]%
            {{\textit{cal\_housing}}}    
        \end{subfigure}

         \begin{subfigure}[b]{0.495\textwidth}   
            \centering 
            \includegraphics[width=\textwidth]{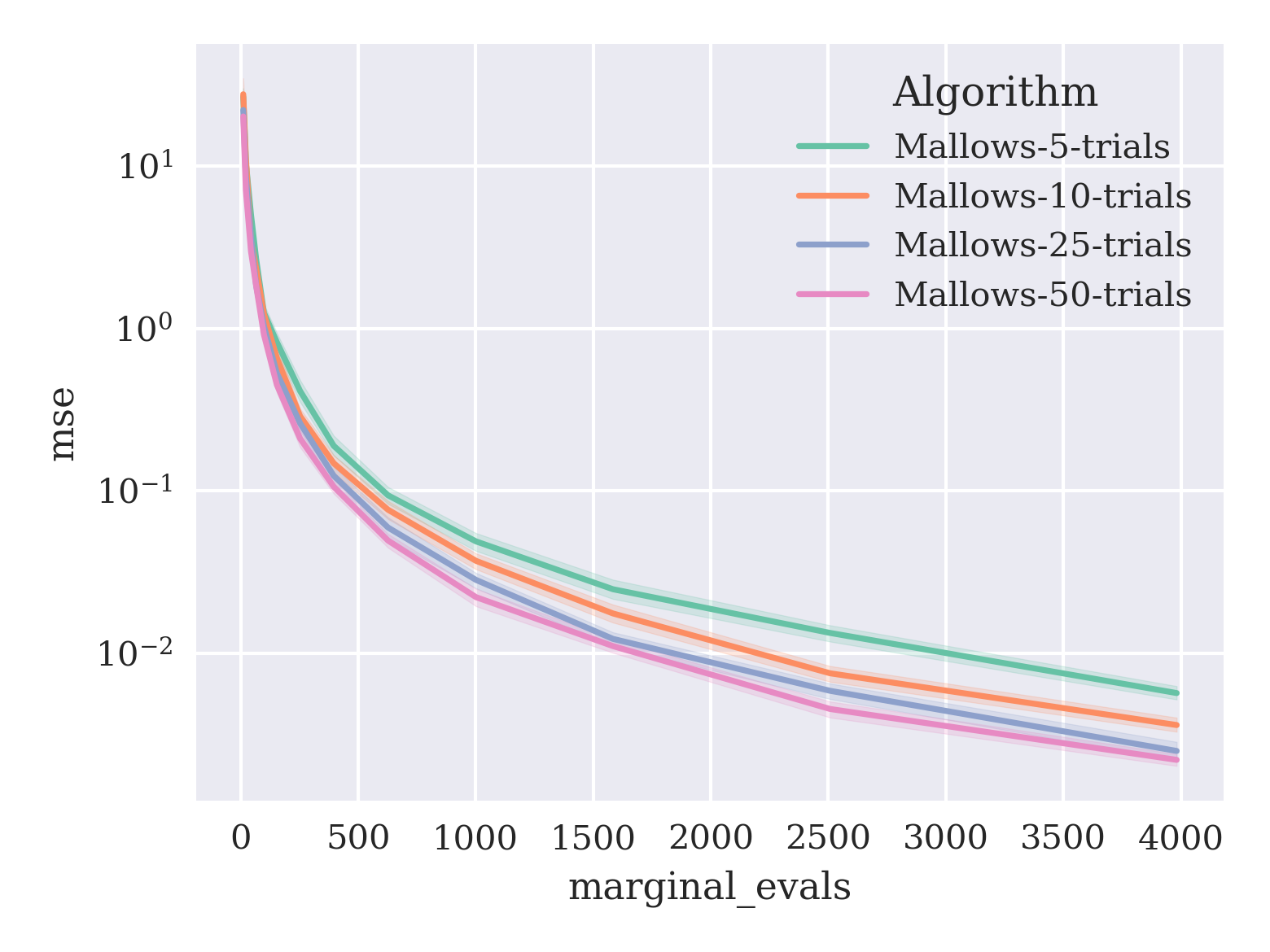}
            \caption[]%
            {{\textit{make\_regression}}}    
        \end{subfigure}
        \hfill
        \begin{subfigure}[b]{0.495\textwidth}   
            \centering 
            \includegraphics[width=\textwidth]{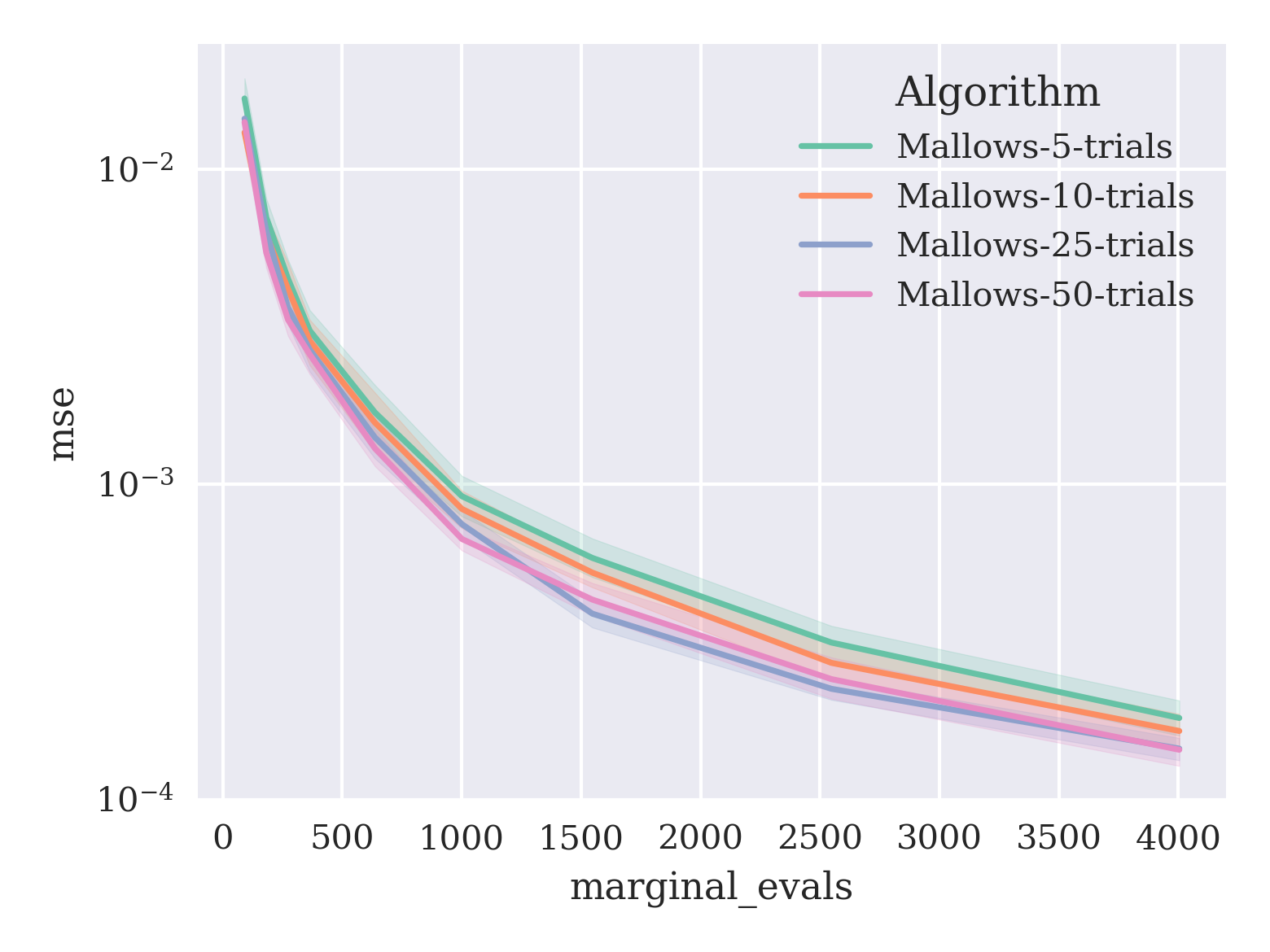}      \caption[]%
            {{\textit{year}}}    
        \end{subfigure}
        \caption{Varying argmax samples for herding algorithm ($\lambda=4)$ --- Tabular datasets and GBDT models. Increasing the number of trials improves accuracy with diminishing returns. We choose 25 trials, compromising between accuracy and runtime.} 
        \label{fig:argmax_samples}
\end{figure*}

Given the parameters for the Mallows kernel above, we can also compare it to the Spearman and Kendall tau kernels introduced in Section \ref{sec:kernel_methods} using the herding algorithm. Figure \ref{fig:kernels_comparison} compares the performance of these kernels on tabular data with GBDT models. The Mallows kernel is applied with $\lambda=4$, and all kernels are using 25 argmax samples. The Spearman kernel is clearly outperformed by both other kernels. The Kendall Tau kernel is effective for 4 out of 6 datasets, but lags behind for \textit{make\_regression} and \textit{cal\_housing}. The Mallows kernel is either the most effective, or within a 95\% confidence interval of the most effective kernel for all datasets. For this reason, as well as its universal property, we use the Mallows kernel exclusively in the experiments of Section \ref{sec:evaluation}.

\begin{figure*}[ht]
        \centering
        \begin{subfigure}[b]{0.495\textwidth}
            \centering
            \includegraphics[width=\textwidth]{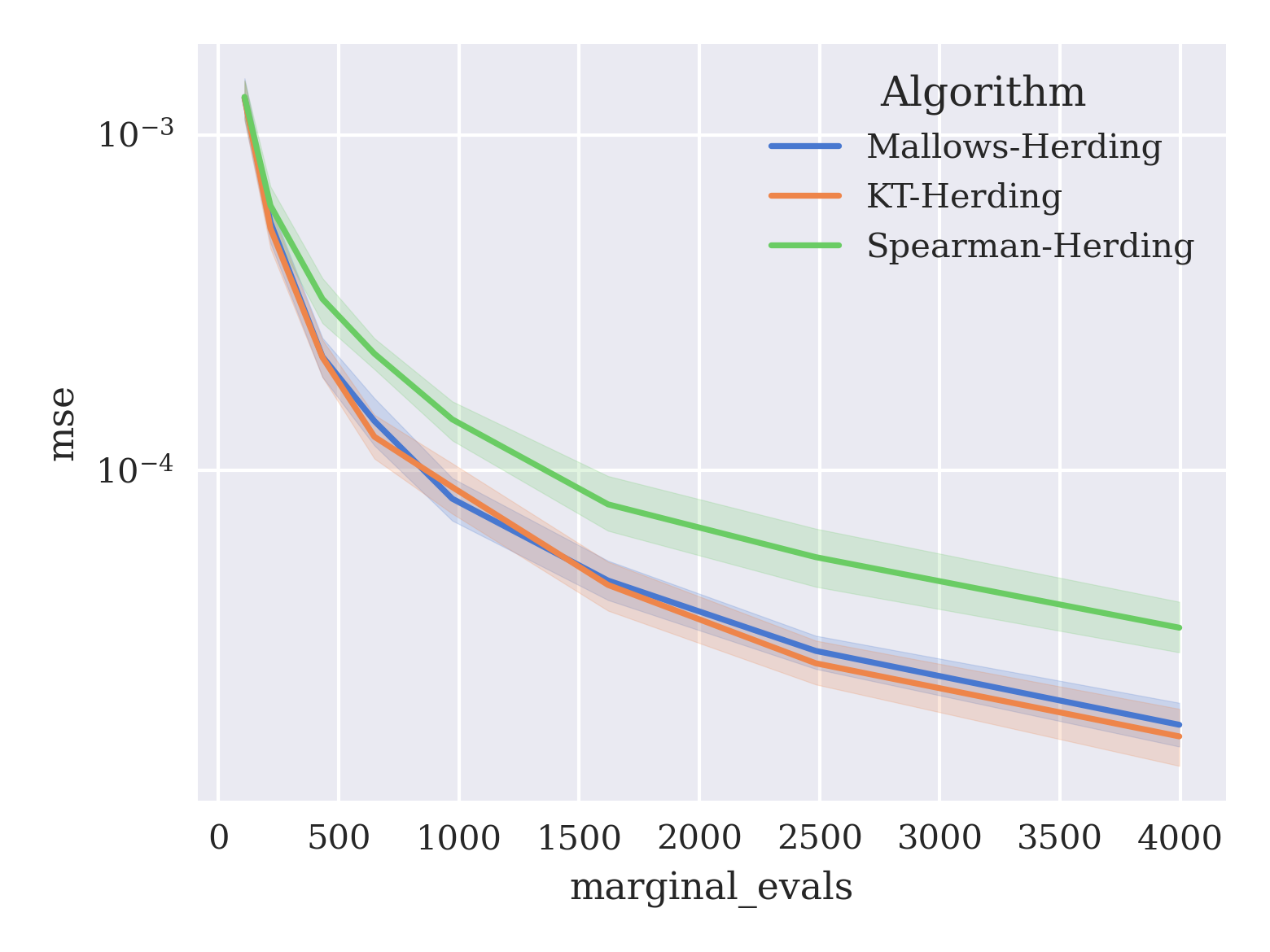}
            \caption[]%
            {{\textit{adult}}}    
        \end{subfigure}
        \hfill
        \begin{subfigure}[b]{0.495\textwidth}  
            \centering 
            \includegraphics[width=\textwidth]{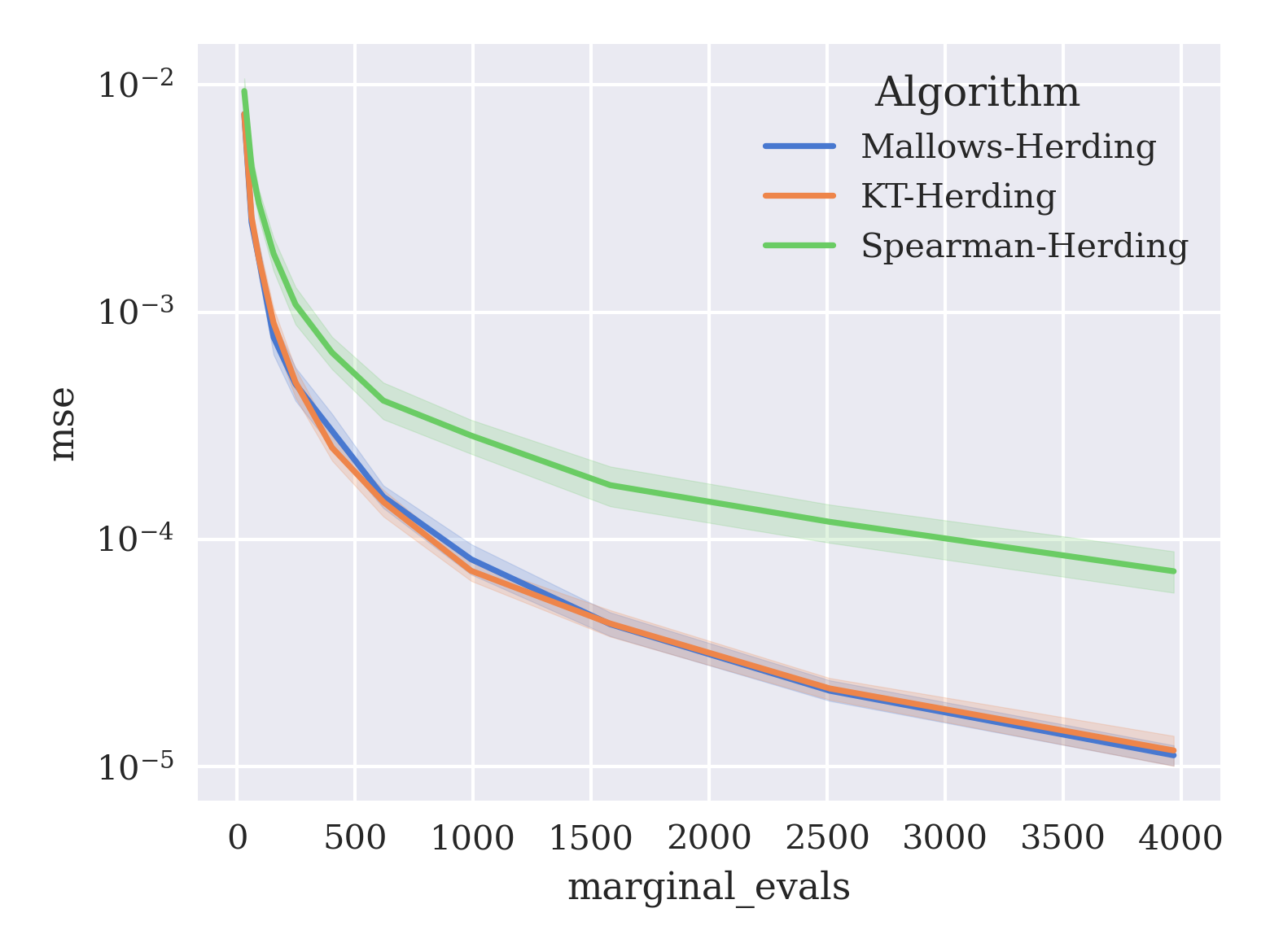}
            \caption[]%
            {{\textit{breast\_cancer}}}    
        \end{subfigure}
%        \vskip\baselineskip
        \begin{subfigure}[b]{0.495\textwidth}   
            \centering 
            \includegraphics[width=\textwidth]{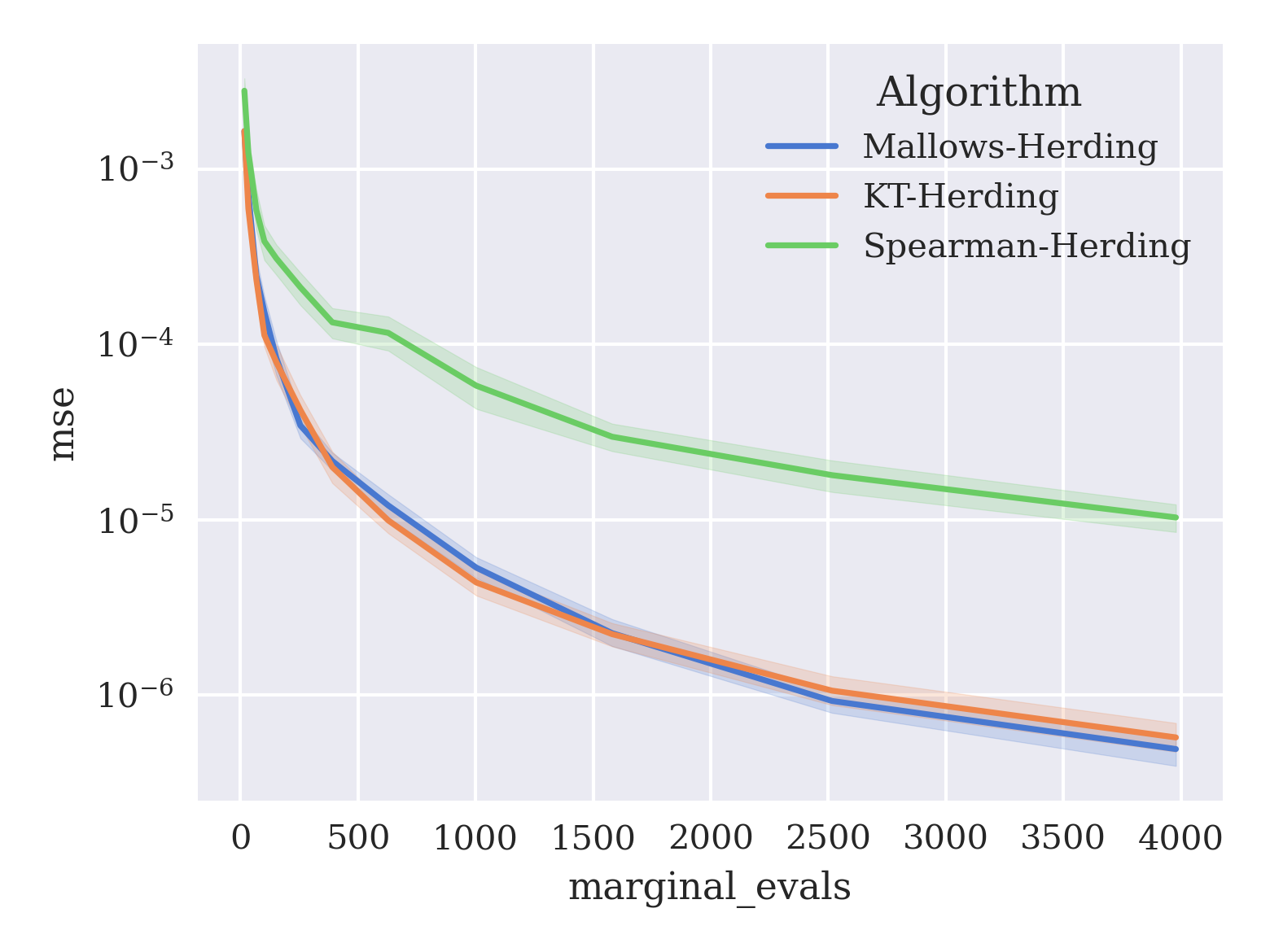}
            \caption[]%
            {{\textit{bank}}}    
        \end{subfigure}
        \hfill
        \begin{subfigure}[b]{0.495\textwidth}   
            \centering 
            \includegraphics[width=\textwidth]{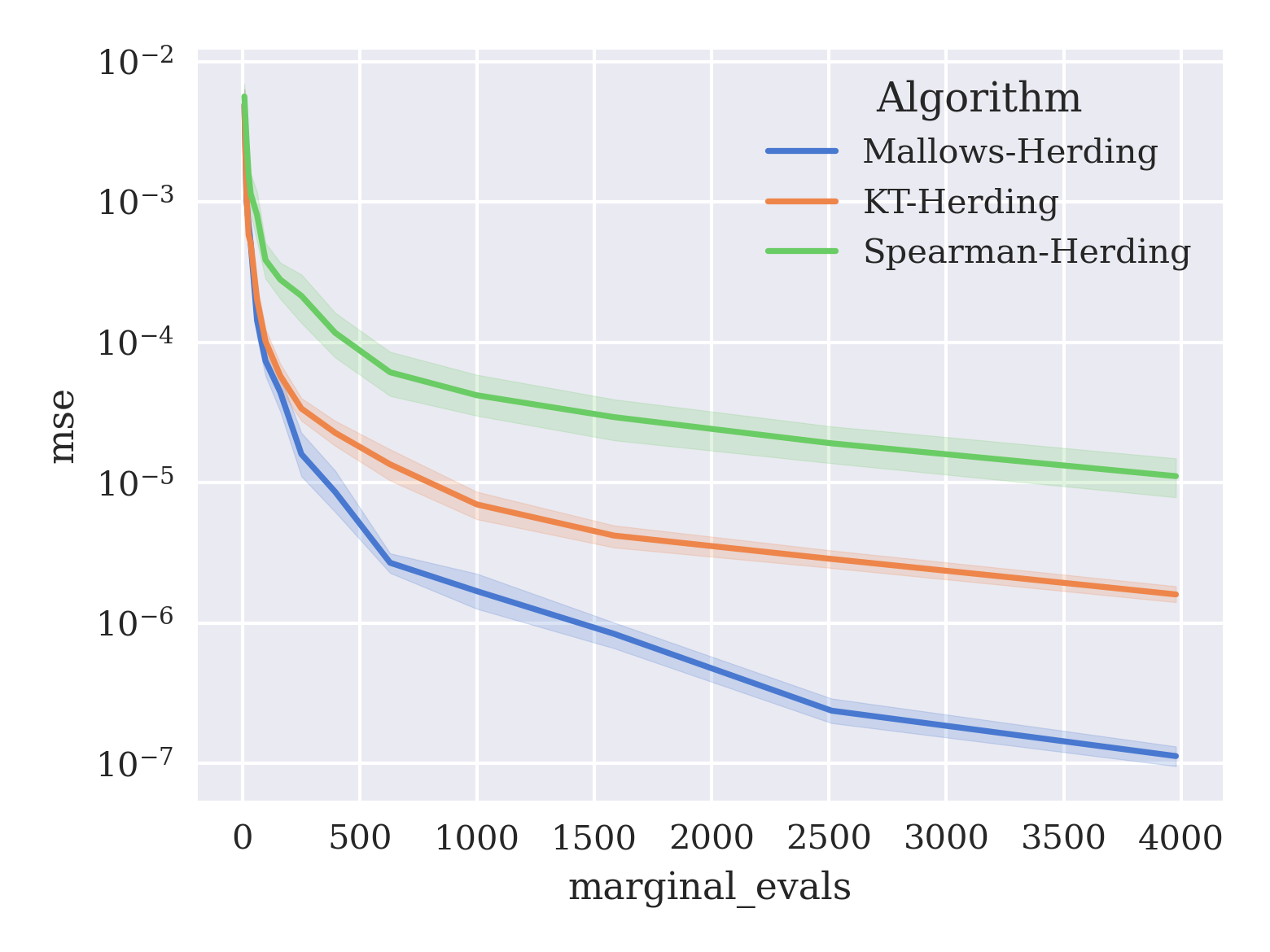}
            \caption[]%
            {{\textit{cal\_housing}}}    
        \end{subfigure}

         \begin{subfigure}[b]{0.495\textwidth}   
            \centering 
            \includegraphics[width=\textwidth]{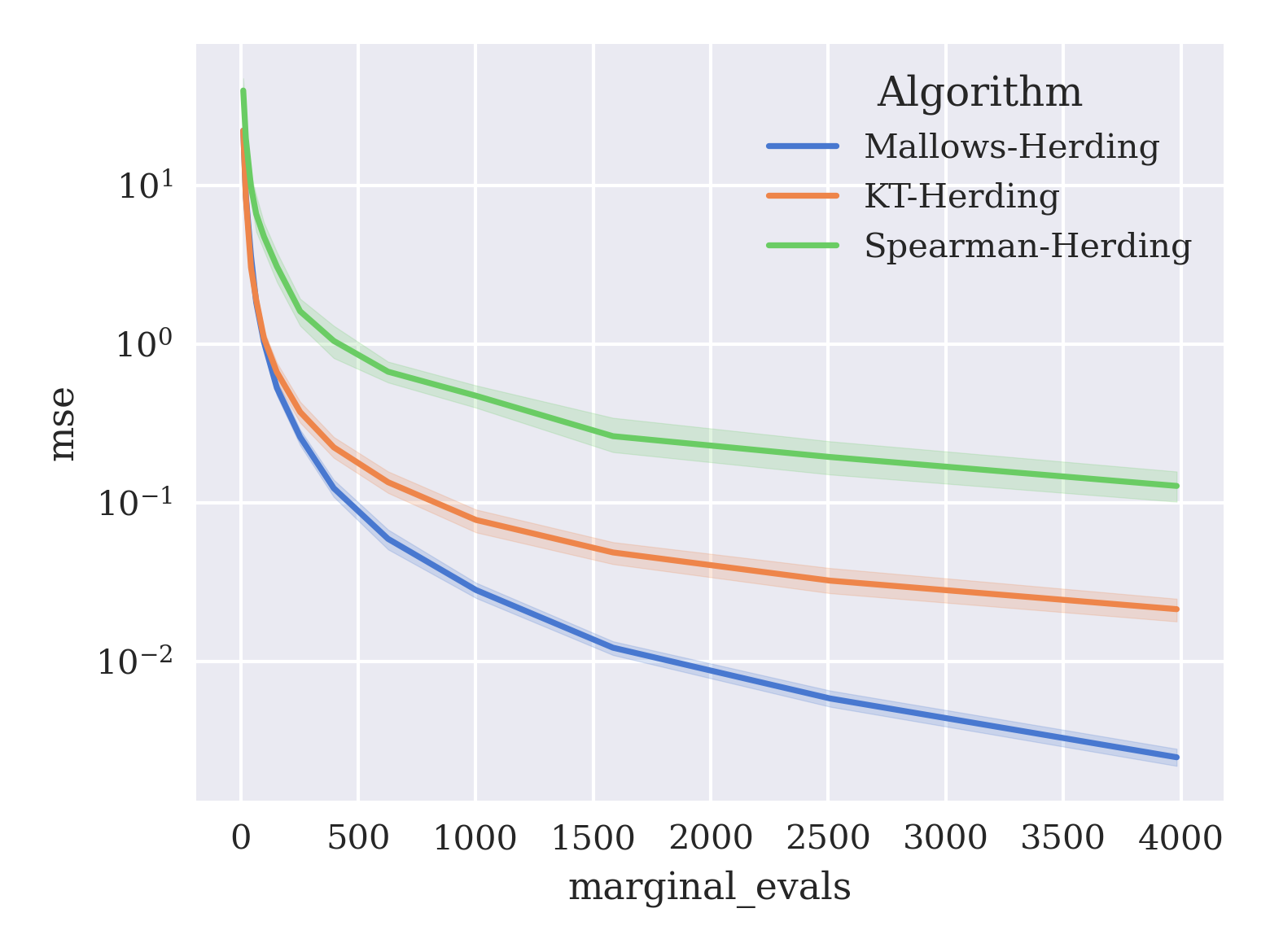}
            \caption[]%
            {{\textit{make\_regression}}}    
        \end{subfigure}
        \hfill
        \begin{subfigure}[b]{0.495\textwidth}   
            \centering 
            \includegraphics[width=\textwidth]{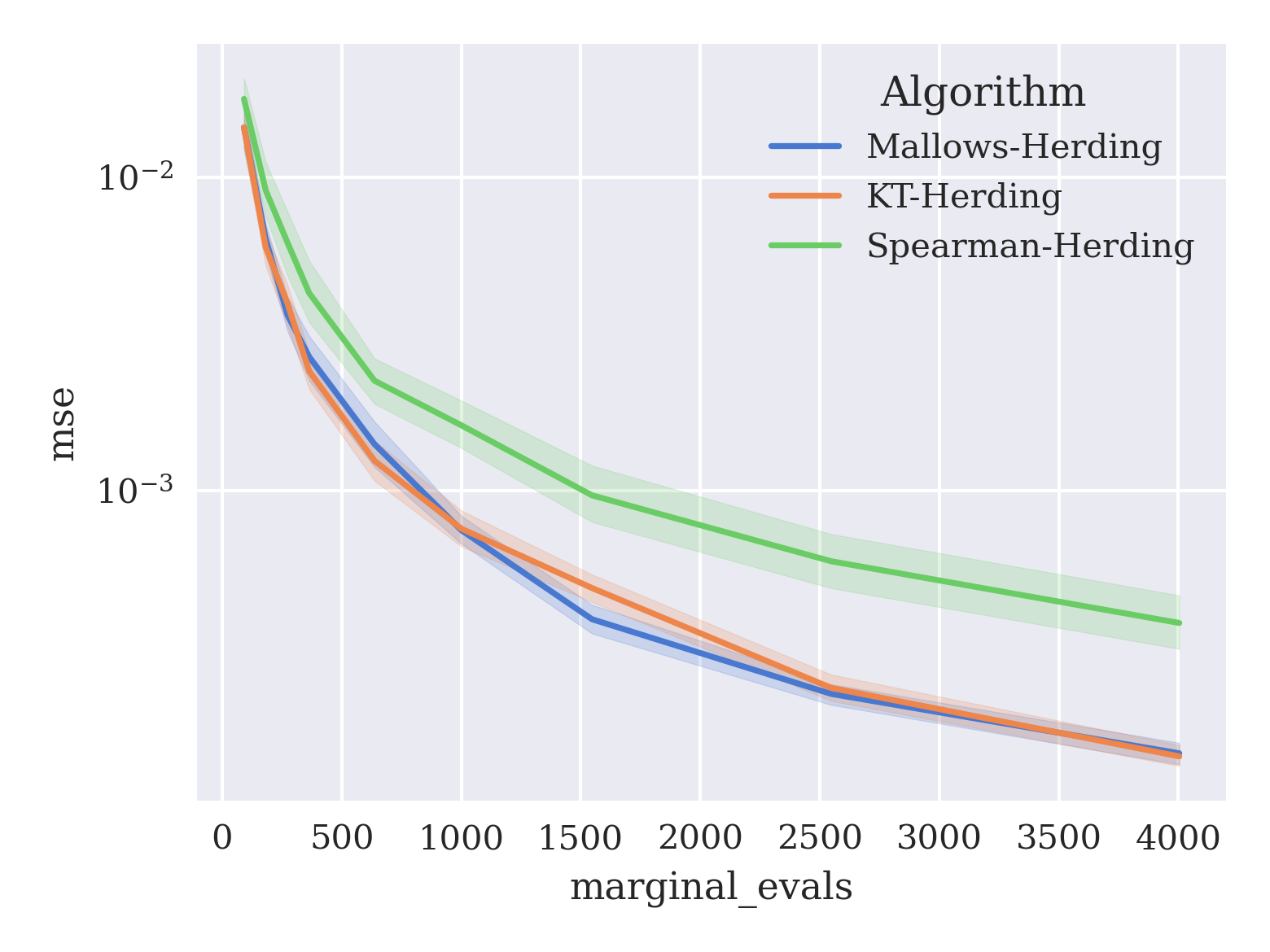}      \caption[]%
            {{\textit{year}}}    
        \end{subfigure}
        \caption{Comparing permutation kernels for kernel herding using tabular data and GBDT models. The Mallows kernel performs at least as well as the other (non-universal) kernels, and often better.} 
        \label{fig:kernels_comparison}
\end{figure*}

\end{document}